\documentclass[journal]{IEEEtran}
\usepackage[T1]{fontenc}

\usepackage{calrsfs}
\usepackage{amsmath, amsthm, accents, amssymb, amsfonts, bm,
  mathtools, cases}
\usepackage[cal = boondox,%
  calscaled = 1.1,%
  bb = ams,%
  bbscaled = .94,%
  frakscaled = .97,%
  scr = esstix]{mathalpha}

\usepackage{algorithm}
\usepackage{algpseudocode}
\usepackage[inline]{enumitem}
\usepackage{graphicx}

\usepackage{caption}
\usepackage{subcaption}
\usepackage{xr-hyper}
\usepackage{hyperref}
\hypersetup{hidelinks}

\usepackage[capitalize]{cleveref}
\usepackage{url}
\usepackage{booktabs}

\usepackage{textcomp}

\usepackage{etoolbox}

\usepackage{siunitx}
\usepackage[mathlines, switch]{lineno}
\crefname{line}{Line}{lines}

\usepackage[textsize=footnotesize]{todonotes}

\newcommand\given{{\mathbin{}\mid\mathbin{}}}
\newcommand{\IntegerP}{\mathbb{N}}
\newcommand{\IntegerPP}{\mathbb{N}_*}
\newcommand{\Real}{\mathbb{R}}
\newcommand{\RealP}{\mathbb{R}_{+}}
\newcommand{\RealPP}{\mathbb{R}_{++}}
\newcommand{\PS}{\mathbb{S}_{++}}
\newcommand\SetSymbol[1][]{
  \nonscript\,#1\vert \allowbreak \nonscript\,\mathopen{}}
\DeclarePairedDelimiterX\Set[1]{\lbrace}{\rbrace}%
{ \renewcommand\given{\SetSymbol[\delimsize]} #1 }
\DeclarePairedDelimiterX\norm[1]\lVert\rVert{\ifblank{#1}{\:\cdot\:}{#1}}
\DeclarePairedDelimiterX\innerp[2]{\langle}{\rangle}{#1
  \mathop{}\delimsize\vert\mathop{} #2}

\newcommand\vect[1]{\mathbf{#1}}
\newcommand\vectgr[1]{\bm{#1}}
\newcommand{\Hilbert}{\mathscr{H}}
\newcommand{\Banach}{\mathcal{B}}
\newcommand{\Bellman}{\mathbfcal{T}}

\DeclareMathOperator{\Fix}{Fix}

\DeclareMathOperator{\Df}{D}
\DeclareMathOperator{\trace}{tr}
\DeclareMathOperator{\Exp}{Exp}

\DeclareMathOperator{\clo}{clos}

\DeclareMathOperator{\inv}{inv}

\newtheoremstyle{mystyle}
  {}{\topsep}
  {\normalfont}{}
  {\normalfont\bfseries}{.}
  {1ex}
  {\thmname{#1}\thmnumber{ #2}\thmnote{ (#3)}}

\theoremstyle{mystyle}
\newtheorem{theorem}{Theorem}
\newtheorem{lemma}[theorem]{Lemma}
\newtheorem{assumption}[theorem]{Assumption}
\newtheorem{assumptions}[theorem]{Assumptions}
\newtheorem{proposition}[theorem]{Proposition}
\newtheorem{example}[theorem]{Example}

\newtheorem{fact}[theorem]{Fact}
\newtheorem{remark}[theorem]{Remark}

\newlist{thmlist}{enumerate}{1}
\setlist[thmlist]{label=\textbf{(\roman{*})}, ref=\thetheorem(\roman{*}), noitemsep}

\newlist{lemlist}{enumerate}{1}
\setlist[lemlist]{label=\textbf{(\roman{*})}, ref=\thelemma(\roman{*}), noitemsep}

\newlist{exlist}{enumerate}{1}
\setlist[exlist]{label=\textbf{(\roman{*})}, ref=\theexample(\roman{*}), noitemsep}

\newlist{factlist}{enumerate}{1}
\setlist[factlist]{label=\textbf{(\roman{*})}, ref=\thefact(\roman{*}), noitemsep}

\newlist{proplist}{enumerate}{1}
\setlist[proplist]{label=\textbf{(\roman{*})}, ref=\theproposition(\roman{*}), noitemsep}

\newlist{asslist}{enumerate}{1}
\setlist[asslist]{label=\textbf{(\roman{*})}, ref=\theassumption(\roman{*}), noitemsep}

\newlist{assslist}{enumerate}{1}
\setlist[assslist]{label=\textbf{(\roman{*})}, ref=\theassumptions(\roman{*}), noitemsep}

\newlist{deflist}{enumerate}{1}
\setlist[deflist]{label=\textbf{(\roman{*})}, ref=\thedefintion(\roman{*}), noitemsep}

\newlist{algolist}{enumerate}{1}
\setlist[algolist]{label=\textbf{(\roman{*})}, ref=\thealgo(\roman{*}), noitemsep}

\newlist{claimlist}{enumerate}{1}
\setlist[claimlist]{label=\textbf{(\roman{*})},
  ref=\theclaim(\roman{*}), noitemsep}

\crefname{theorem}{Theorem}{Theorems}
\crefname{proposition}{Proposition}{Propositions}
\crefname{assumption}{Assumption}{Assumptions}
\crefname{assumptions}{Assumptions}{Assumptions}
\crefname{lemma}{Lemma}{Lemmata}
\crefname{definition}{Definition}{Definitions}
\crefname{example}{Example}{Examples}
\crefname{algo}{Algorithm}{Algorithms}
\crefname{fact}{Fact}{Facts}
\crefname{claim}{Claim}{Claims}
\crefname{corollary}{Corollary}{Corollaries}

\crefname{thmlisti}{Theorem}{Theorems}
\crefname{lemlisti}{Lemma}{Lemmata}
\crefname{proplisti}{Proposition}{Propositions}
\crefname{asslisti}{Assumption}{Assumptions}
\crefname{assslisti}{Assumption}{Assumptions}
\crefname{deflisti}{Definition}{Definitions}
\crefname{exlisti}{Example}{Examples}
\crefname{algolisti}{Algorithm}{Algorithms}
\crefname{factlisti}{Fact}{Facts}
\crefname{claimlisti}{Claim}{Claims}

\crefname{figure}{Figure}{Figures}
\crefname{section}{Section}{Sections}

\allowdisplaybreaks

\usepackage[%
  backend = biber,%
  maxbibnames = 10,%
  minbibnames = 2,%
  style = ieee,%
  citestyle = numeric-comp,%
  bibencoding = utf8,%
  datamodel = software,%
  defernumbers = true,%
  sortcites=true,%
  doi=true,%
  isbn=false,%
  url=false,%
  eprint=true%
]{biblatex}

\makeatletter
\newlength{\negph@wd}
\DeclareRobustCommand{\negphantom}[1]{%
  \ifmmode
  \mathpalette\negph@math{#1}%
  \else
  \negph@do{#1}%
  \fi
}
\newcommand{\negph@math}[2]{\negph@do{$\m@th#1#2$}}
\newcommand{\negph@do}[1]{%
  \settowidth{\negph@wd}{#1}%
  \hspace*{-\negph@wd}%
}
\makeatother

\usepackage{tikz}
\usepackage{xcolor}
\usepackage{pgfplots}
\pgfplotsset{compat=newest}
\usepgflibrary{fillbetween}
\usetikzlibrary{angles, arrows, arrows.meta, shapes,
tikzmark, patterns}

\colorlet{dqn}{green!80!black}
\newcommand{\dqn}{
    \tikz[baseline=-0.5ex]{
        \draw[line width=1pt, color=dqn] (-0.3, 0) -- (0.3, 0);
        \draw[color=dqn] plot[mark size=4pt, mark=square*,
        mark options={line width=1pt}]
        coordinates {(0,0)};
    }
}

\colorlet{ppo}{violet!80!black}
\newcommand{\ppo}{
    \tikz[baseline=-0.5ex]{
        \draw[line width=1pt, color=ppo] (-0.3, 0) -- (0.3, 0);
        \draw[color=ppo] plot[mark size=4pt, mark=triangle*,
        mark options={line width=1pt}]
        coordinates {(0,0)};
    }
}

\colorlet{proposed}{blue!80!black}
\newcommand{\proposed}{
    \tikz[baseline=-0.5ex]{
        \draw[line width=1pt, color=proposed] (-0.3, 0) -- (0.3, 0);
        \draw[color=proposed] plot[mark size=4pt, mark=*,
        mark options={line width=1pt}]
        coordinates {(0,0)};
    }
}

\colorlet{klspi}{orange!80!black}
\newcommand{\klspi}{
    \tikz[baseline=-0.5ex]{
        \draw[line width=1pt, color=klspi] (-0.3, 0) -- (0.3, 0);
        \draw[color=klspi] plot[mark size=4pt,
        mark=diamond*,
        mark options={line width=1pt}]
        coordinates {(0,0)};
    }
}

\colorlet{obr}{magenta!80!black}
\newcommand{\obr}{
    \tikz[baseline=-0.5ex]{
        \draw[line width=1pt, color=obr] (-0.3, 0) -- (0.3, 0);
        \draw[color=obr] plot[mark size=4pt,
        mark=diamond*,
        mark options={line width=1pt, fill=white}]
        coordinates {(0,0)};
    }
}

\colorlet{emgmm}{red!80!black}
\newcommand{\emgmm}{
    \tikz[baseline=-0.5ex]{
        \draw[line width=1pt, color=emgmm] (-0.3, 0) -- (0.3, 0);
        \draw[color=emgmm] plot[mark size=4pt,
        mark=pentagon*,
        mark options={line width=1pt}]
        coordinates {(0,0)};
    }
}

\newcommand{\quicksymbol}[2]{
    \tikz[baseline=-0.5ex]{
        \draw[line width=1pt, color=#1] (-0.3, 0) -- (0.3, 0);
        \draw[color=#1] plot[mark size=4pt, #2]
        coordinates {(0,0)};
    }
}

\begin{document}


\title{ Gaussian-Mixture-Model Q-Functions for\\ Policy Iteration in Reinforcement Learning }

\author{Minh Vu and Konstantinos Slavakis,~\IEEEmembership{Senior Member,~IEEE}
  \thanks{The authors are with the Department of Information and Communications, Institute of Science
    Tokyo, 4259-G2-4 Nagatsuta-cho, Midori-ku, Yokohama, Kanagawa, Japan. Emails:
    \texttt{vu.d.a5c3@m.isct.ac.jp, slavakis@ict.eng.isct.ac.jp} } }

\maketitle

\begin{abstract}
  Unlike their conventional use as estimators of probability density functions in reinforcement learning
  (RL), this paper introduces a novel function-approximation role for Gaussian mixture models (GMMs) as
  direct surrogates for Q-function losses. The proposed estimators, termed GMM-QFs, possess substantial
  representational capacity, as they are shown to be universal approximators over a broad class of
  functions. They are further embedded within Bellman residuals, where their learnable parameters---a
  fixed number of mixing weights, together with Gaussian mean vectors and covariance matrices---are
  inferred from data via optimization on a product Riemannian manifold. This geometric perspective on the
  parameter space naturally introduces Riemannian optimization into the policy-evaluation step of
  standard policy-iteration (PI) frameworks. {\color{black} Moreover, rigorous theoretical analysis
    establishes performance bounds on Q-function estimation error under the proposed PI scheme.}
  Supporting numerical tests show that GMM-QFs deliver competitive performance and, in some cases,
  outperform state-of-the-art approaches across a range of benchmark RL tasks, all while maintaining a
  significantly smaller computational footprint than deep-learning methods.
\end{abstract}


\section{Introduction}\label{sec:intro}

Reinforcement learning (RL)~\cite{Bertsekas:RLandOC:19, Sutton:IntroRL:18} is a machine-learning paradigm
in which an intelligent ``agent'' interacts with its environment to determine an optimal policy that
minimizes the cumulative cost/loss associated with its ``actions.''  By modeling the environment as a
Markov decision process (MDP)~\cite{Bertsekas:RLandOC:19}, RL furnishes a rigorous mathematical framework
for sequential-making problems in complex domains such as robotics, wireless communications, data mining,
and bioinformatics~\cite{Bertsekas:RLandOC:19}.

To determine an optimal policy, RL strategies generally aim to estimate the value or loss (Q-function)
corresponding to the selection of a particular action in a given state, using feedback obtained from the
environment (\cref{fig:dp}). Classical methods such as Q-learning~\cite{watkins92Qlearning} and
state-action-reward-state-action (SARSA)~\cite{singh00sarsa} estimate Q-functions via lookup tables that
store Q-values for every possible state-action pair. Although effective in environments with discrete and
small-scale state-action spaces, these tabular approaches become computationally prohibitive in practical
settings characterized by continuous or large-scale state-action spaces. To mitigate this limitation,
significant attention has been devoted to the development of RL algorithms that leverage (typically
non-linear) models to approximate Q-functions~\cite{Bertsekas:RLandOC:19}.

Approximation models for Q-functions have a long-standing history in RL. Classical kernel-based (KB)RL
methods~\cite{ormoneit02kernel, ormoneit:autom:02, bae:mlsp:11} represent Q-functions as elements of
Banach spaces, typically those comprising all essentially bounded functions. Methods such as temporal
difference (TD)~\cite{sutton88td}, least-squares (LS)TD~\cite{lagoudakis03lspi, regularizedpi:16,
  xu07klspi}, Bellman-residual (BR) approaches~\cite{onlineBRloss:16}, and more recent nonparametric
designs~\cite{vu23rl, akiyama24proximal, akiyama24nonparametric} model Q-functions within user-defined
reproducing kernel Hilbert spaces (RKHSs)~\cite{aronszajn50kernels, scholkopf2002learning}, thereby
exploiting the underlying geometric structure and computational efficiency afforded by the reproducing
inner product. A notable drawback of these kernel-based models, however, is that their number of design
parameters typically scales with the quantity of observed data, which can result in significant memory
and computational overhead---particularly in dynamic settings with nonstationary data distributions. To
alleviate this limitation, dimensionality-reduction techniques have been introduced~\cite{xu07klspi,
  vu23rl}. Nevertheless, constraining often the number of basis elements in the approximation subspace to
effect dimensionality reduction may lead to degraded accuracy in the resulting Q-function estimates. A
comprehensive review of KBRL, (LS)TD, and BR methods, along with their connections to RKHSs, is provided
in~\cite{akiyama24nonparametric}.

\begin{figure}[t]
  \centering
  \includegraphics[width=.95\columnwidth]{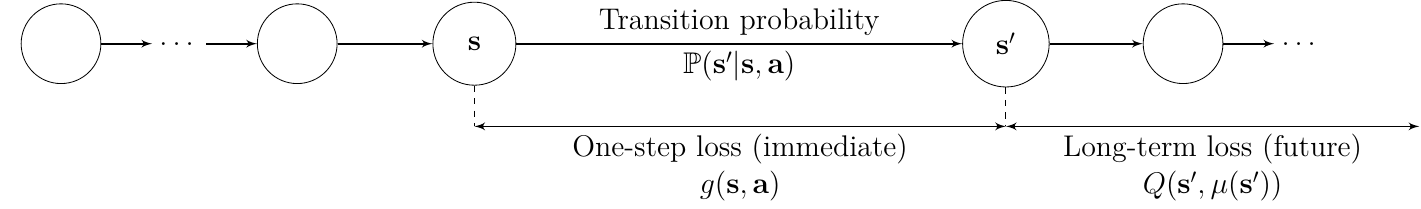}
  \caption{RL as a sequential decision-making process: at state $\vect{s}$, the RL agent takes
    decision/action $\vect{a} \coloneqq \mu( \vect{s} )$, suffers the one-step loss $g(\vect{s},
    \vect{a})$, and moves to the next state $\vect{s}^{\prime}$ according to some transition
    probability. Function $\mu(\cdot)$ denotes the policy or decision-making mechanism. The agent seeks
    to identify a policy that minimizes the cumulative (long-term) loss---quantified by the Q-function
    $Q(\cdot)$---incurred over its sequence of actions.}\label{fig:dp}
\end{figure}

Deep neural networks have been widely adopted as nonlinear Q-function approximators, most notably in the
form of deep Q-networks (DQNs), e.g.,~\cite{mnih13dqn, hasselt16ddqn}. These models are typically trained
using \textit{experience data}~\cite{lin93experience} gathered from previously employed policies and
stored in a \textit{replay buffer.} The replay buffer serves as a repository of prior knowledge and
contributes significantly to stabilizing the training of DQN-based algorithms. While DQNs offer a
practical and powerful means of training RL agents---leveraging the computational efficiency and
expressiveness of neural networks---they exhibit limitations in several key scenarios.  In particular,
online deployment requires rapid adaptation to new data, which is challenging for DQNs due to their
massive parameter count. While frequent re-training is theoretically possible in simulation, it is too
computationally and hardware-intensive for lightweight, real-time applications. Furthermore, large-scale
neural networks in DQNs are notoriously more prone to instability and overfitting, reducing their
robustness against non-stationary or unpredictable real-world conditions (out-of-distribution problems).

In distributional RL~\cite{bellemare17distRL, QRDQN, sato99em, mannor05rl, 
  agostini17gmmrl, choi19distRL}---a prominent and increasingly influential direction in RL---Q-functions
are usually treated as (statistics of) random variables (RVs). For instance, \cite{mannor05rl} assumes
that samples of value or Q-functions are jointly Gaussian, an assumption that, when combined with
classical least squares (LS) and Gauss-Markov theory, leads to Kalman-type algorithmic solutions. To
enable more expressive probabilistic modeling, Gaussian mixture models (GMMs)~\cite{Reynolds:GMM:19,
  McLachlan:FiniteMixtures:00} have been widely adopted to approximate either the joint PDF $p(Q,
\vect{s}, \vect{a})$---in which Q-function $Q$, state $\vect{s}$, and action $\vect{a}$ are all modeled
as realizations (observations) of RVs---or the conditional $p(Q \given \vect{s}, \vect{a})$. This
conventional use of GMMs, closely linked to maximum likelihood estimation~\cite{Demp:77,
  Figueiredo:mixtures:02}, naturally motivates the adoption of expectation-maximization (EM)
procedures~\cite{sato99em, mannor05rl,  agostini17gmmrl}. As such, distributional RL
estimates a desirable Q-function $Q$ \textit{indirectly}\/ as a statistical byproduct of the modeled
PDF---most commonly as the mean of the conditional distribution $p(Q \given \vect{s},
\vect{a})$. However, it is well-established that EM algorithms are highly sensitive to
initialization~\cite{Figueiredo:mixtures:02}, which can negatively impact convergence and stability.

{ 
  This manuscript introduces the following contributions.

  \begin{enumerate}[font=\bfseries]

  \item[(C1)] GMMs are employed \textit{directly}\/ as surrogates of Q-functions, rather than as
    approximators of their PDFs. The proposed GMMs serve as universal approximators for a broad class of
    functions---see \cref{thm:universal.approx}. Unlike distributional RL, no statistical assumptions on
    Q-functions themselves or their evaluations are made. Although GMMs have previously been utilized to
    model stochastic policy functions~\cite{RPPO}---that is, conditional PDFs---within the framework of
    policy-gradient and policy-search methods~\cite{sutton99pg}, their direct use as Q-function surrogates
    appears to be novel within the RL literature.

  \item[(C2)] This work establishes a \textit{parametric}\/ model, in contrast to the nonparametric
    strategies~\cite{park93rbf} employed in the popular KBRL~\cite{ormoneit02kernel, ormoneit:autom:02,
      bae:mlsp:11, sutton88td, lagoudakis03lspi, regularizedpi:16, xu07klspi, onlineBRloss:16}, where the
    number of tunable parameters---mixing weights of the kernel functions---scales with the volume of
    observed data. The proposed parametric GMM-QFs, by contrast, have a fixed number of learnable
    parameters---namely, mixing weights, mean vectors, and covariance matrices of the Gaussian
    kernels---regardless of the size of the observed dataset.

  \item[(C3)] The learnable parameters of GMM-QFs constitute a (product) \textit{Riemannian manifold.}
    This structure is not forcibly imposed on the design; rather, it arises naturally from the nature of
    the parameters themselves---covariance matrices inherently form a Riemannian manifold. This work
    harnesses this underlying geometry: GMM-QFs are integrated into a Bellman-residual-inspired loss, and
    parameters are learned on a Riemannian manifold, opening the door, for the first time, to the
    versatile machinery of Riemannian optimization~\cite{Absil:OptimManifolds:08} within the
    policy-evaluation step of standard PI schemes. The proposed parametric model, together with
    Riemannian optimization, fosters a computationally efficient framework---see \cref{sec:tests}, where
    it is shown that GMM-QFs require fewer parameters to achieve performance comparable to, and in many
    cases surpassing, that of standard DQNs.
  \end{enumerate}
}
The manuscript is organized as follows. \cref{sec:rl} reviews the classical Bellman mappings and
establishes the notation. The proposed class of GMM-QFs, together with their
universal approximation properties, is introduced
in~\cref{sec:GMM-QFs}. The Riemannian-optimization-driven policy iteration (PI)
framework incorporating these models, along with its performance
analysis, is detailed in~\cref{sec:rpi}. Numerical experiments on benchmark control tasks are
reported in~\cref{sec:tests}. Finally, \cref{sec:conclusion} provides concluding remarks and outlines
directions for future research. Additional numerical tests are presented in the supplementary
material. This manuscript substantially extends the conference paper~\cite{Vu:eusipco:25}.

\section{Policy Iteration and Bellman Residuals}\label{sec:rl}

\subsection{Notations}\label{subsec:notations}

Let $\mathfrak{S} \subset \Real^D$ denote the \textit{continuous}\/ state space, with state vector
$\vect{s} \in \mathfrak{S}$, for some $D \in \IntegerPP$ ($\IntegerPP$ is the set of all positive
integers). The \textit{discrete}\/ and \textit{finite}\/ action space is denoted by $\mathfrak{A}$, with
action $\vect{a} \in \mathfrak{A}$. An agent, currently at state $\vect{s} \in \mathfrak{S}$, takes
action $\vect{a} \in \mathfrak{A}$ and transits to a new state $\vect{s}^\prime \in\mathfrak{S}$ under
transition probability $p(\vect{s}^\prime \given \vect{s}, \vect{a})$, which is unknown to the user/agent
in most cases of practice. This transition 
incurs a real-valued 
\textit{one-step}\/ penalty (loss), which represents the immediate or short-term cost of the agent's
action $\vect{a}$ at state $\vect{s}$ and is typically provided by the environment; see \cref{fig:dp}.

For convenience, define also the state-action vector $\vect{z} \coloneqq \zeta(\vect{s}, \vect{a})$, with
a user-defined mapping $\zeta \colon \mathfrak{S} \times \mathfrak{A} \to \mathfrak{Z} \subset
\Real^{D_z}$, where $D_z \in \IntegerPP$. Mapping $\zeta (\cdot, \cdot)$ is introduced to accommodate a
broad range of state-action pairs, including cases where the discrete $\mathfrak{A}$ is categorical,
e.g., it contains actions such as ``move to the left'' or ``move to the right.''  Even when
$\mathfrak{A}$ is categorical, the mapping $\zeta (\cdot, \cdot)$ is introduced so that $\mathfrak{Z}$
becomes a subset of the continuous space $\Real^{D_z}$, thereby enabling the subsequent definition of
Gaussian functions on $\mathfrak{Z}$.

As such, $\mathfrak{Z}$ serves as the domain of the \textit{one-step loss}\/ $g \colon \mathfrak{Z} \to
\Real \colon \vect{z} \mapsto g(\vect{z})$, as well as of the \textit{Q-function}\/ $Q \colon
\mathfrak{Z} \to \Real \colon \vect{z} \mapsto Q(\vect{z})$. Consequently, the one-step penalty received
from the environment when the agent takes action $\vect{a}$ in state $\vect{s}$ is given by $g
(\zeta(\vect{s}, \vect{a})) \in \Real$, while $Q (\zeta(\vect{s}, \vect{a}) )$ represents the
\textit{long-term}\/ cost that the agent will incur over an infinite future horizon if it takes action
$\vect{a}$ in state $\vect{s}$; see \cref{fig:dp}. The exact functional forms of $g$ and $Q$ are
typically unknown to the agent, but are commonly modeled and inferred from observed data together with
a-priori knowledge of the surrounding environment.

Following~\cite{Bertsekas:RLandOC:19}, consider the set of all mappings $\mathcal{M} \coloneqq
\Set{\mu(\cdot) \given \mu(\cdot) \colon \mathfrak{S} \to \mathfrak{A} \colon \vect{s} \mapsto
  \mu(\vect{s})}$. In other words, $\mu(\vect{s})$ denotes the action that the agent takes at state
$\vect{s}$ under $\mu$. The set of policies is defined as $\Pi \coloneqq \mathcal{M}^{\IntegerP}
\coloneqq \Set{\mu_0, \mu_1, \dots, \mu_n, \dots \given \mu_n \in \mathcal{M}, n \in \IntegerP}$. A
policy will be denoted by $\pi \in \Pi$. Given $\mu\in\mathcal{M}$, a stationary policy $\pi_\mu$ is
defined as $\pi_\mu \coloneqq (\mu, \mu, \dots, \mu, \dots)$. It is customary for $\mu$ to denote also
$\pi_\mu$. Finally, for compact notations, let $\overline{1,K} \coloneqq \{1, \ldots, K\}$, for any $K\in
\IntegerPP$.

\subsection{Bellman mappings, policy iteration, and approximation spaces}\label{subsec:bellman}

The classical \textit{Bellman mappings}~\cite{Bertsekas:RLandOC:19} compute the total cost---that is, the
sum of the one-step loss and the expected long-term loss---the agent would incur if action $\vect{a}$
were taken at state $\vect{s}$. More specifically, if $\Banach$ stands for a \textit{user-defined}\/
space of Q-functions, typically the Banach space of all essentially bounded
functions~\cite{Bertsekas:RLandOC:19}, then the classical Bellman mappings $\Bellman^{\diamond}_\mu,
\Bellman^{\diamond} \colon \Banach \to \Banach \colon Q \mapsto \Bellman^{\diamond}_\mu Q,
\Bellman^{\diamond} Q$ are defined as~\cite{Bertsekas:RLandOC:19}
\begin{subequations}\label{Bellman.maps.standard}
  \begin{align}
    (\Bellman_{\mu}^{\diamond} Q)( \zeta(\vect{s}, \vect{a}) ) & \coloneqq g( \zeta(\vect{s}, \vect{a}) ) + \alpha
    \mathbb{E}_{\vect{s}^{\prime} \given (\vect{s}, \vect{a})} \{  Q( \zeta(\vect{s}^{\prime}, \mu(\vect{s}^{\prime})) )
    \}\,, \label{Bellman.standard.mu} \\
    (\Bellman^{\diamond} Q)( \zeta(\vect{s}, \vect{a}) ) & \coloneqq g( \zeta(\vect{s}, \vect{a}) ) + \alpha
    \mathbb{E}_{\vect{s}^{\prime} \given ( \vect{s}, \vect{a} ) } \{  \min_{\vect{a}^{\prime}\in \mathfrak{A}} Q(
      \zeta(\vect{s}^{\prime}, \vect{a}^{\prime}) ) \}\,, \label{Bellman.standard}
  \end{align}
\end{subequations}
$\forall (\vect{s}, \vect{a})$, where $\mathbb{E}_{\bm{s}^{\prime} \given (\vect{s}, \vect{a})} \{ \cdot
\}$ stands for the conditional expectation operator with respect to (w.r.t.) the potentially next state
$\vect{s}^{\prime}$ conditioned on $(\vect{s}, \vect{a})$, and $\alpha \in [0,1)$ is the discount
  factor. A more rigorous discussion based on probabilistic arguments and notation is deferred to
  \cref{subsec:performance}. Mapping \eqref{Bellman.standard.mu} refers to the case where the agent takes
  actions according to a stationary policy $\mu$, while \eqref{Bellman.standard} serves as a greedy
  variation of \eqref{Bellman.standard.mu}.

For a mapping $\Bellman\colon \Banach \to \Banach$, its fixed-point set is defined as $\Fix \Bellman
\coloneqq \{Q\in\Banach \given \Bellman Q=Q \}$. It is well-known that $\Fix \Bellman_{\mu}^\diamond$ and
$\Fix \Bellman^{\diamond}$ play central roles in identifying \textit{optimal}\/ policies which minimize
the total loss~\cite{Bertsekas:RLandOC:19}. Usually, the discount factor $\alpha \in [0,1)$ to render
  $\Bellman_{\mu}^\diamond, \Bellman^{\diamond}$ strict contractions and to ensure that $\Fix
  \Bellman_{\mu}^{\diamond}$ and $\Fix \Bellman^{\diamond}$ are nonempty and
  singletons~\cite{Bertsekas:RLandOC:19, hb.plc.book}. These arguments will be revisited in
  \cref{subsec:performance}.

Policy iteration (PI) is a popular RL framework~\cite{Bertsekas:RLandOC:19, Sutton:IntroRL:18}. It
comprises two stages per iteration $n \in \IntegerP$: \textit{policy evaluation}\/ and \textit{policy
  improvement.} For example, given a stationary policy $\mu_n$ at iteration $n$,
policy evaluation computes a Q-function $Q_n$ that ``closely'' approximates, in an appropriate
sense---when possible---$Q_{\mu_n}^{\diamond} \in \Fix \Bellman^{\diamond}_{\mu_n}$, while policy
improvement updates the policy according to the following greedy rule:
\begin{equation}\label{eq:policy-impr}
  \mu_{n+1} (\vect{s}) \coloneqq \arg \min\nolimits_{ \vect{a} \in \mathfrak{A}} Q_n(\, \zeta(\vect{s}, \vect{a})\, )
  \,, \quad \forall \vect{s}\in \mathfrak{S} \,.
\end{equation}
PI repeats this process to generate a sequence of (stationary) policies and Q-functions as in $\mu_0 \to
Q_{0} \to \mu_1 \to Q_{1} \to \ldots$, with the goal that the resulting sequence of the
Q-functions converges, in some sense, to $Q^{\diamond} \in \Fix \Bellman^{\diamond}$. The \textit{optimal
  policy}\/ then is defined as $\mu^{\diamond}(\vect{s}) \coloneqq \arg \min\nolimits_{ \vect{a} \in
  \mathfrak{A}} Q^{\diamond}( \zeta (\vect{s}, \vect{a}) )$, $\forall \vect{s}$.

Real-world problems involving high-dimensional or even continuous state spaces $\mathfrak{S}$ pose
significant computational challenges for tabular methods, where Q-functions must be evaluated at
\textit{every}\/ $\vect{s} \in \mathfrak{S}$ to identify policies---see, for example,
\eqref{eq:policy-impr}. To address this limitation, function-approximation models for Q-functions have
been developed as viable alternatives to tabular methods, enabling \textit{approximate PI}\/ schemes
through their incorporation into policy-evaluation steps.

To this end, a reproducing kernel Hilbert space (RKHS) $\mathcal{H}$ is a popular choice for the
underlying approximation space of Q-functions~\cite{sutton88td, ormoneit02kernel, ormoneit:autom:02,
  bae:mlsp:11, lagoudakis03lspi, regularizedpi:16, xu07klspi, onlineBRloss:16, vu23rl, akiyama24proximal,
  akiyama24nonparametric}, owing to the presence of an inner product and the reproducing property induced
by its associated kernel~\cite{aronszajn50kernels, scholkopf2002learning}---features that are absent in a
Banach space, which is the typical choice for $\mathcal{B}$ in the earlier discussion. For an RKHS
$\mathcal{H}$ and its inner product $\innerp{\cdot}{\cdot}_{ \mathcal{H} }$, induced by a reproducing
kernel $\kappa(\cdot, \cdot) \colon \Real^{D_z} \times \Real^{D_z} \to \Real \colon ( \vect{z}^{\prime},
\vect{z}) \mapsto \kappa ( \vect{z}^{\prime}, \vect{z})$~\cite{aronszajn50kernels,
  scholkopf2002learning}, with the feature mapping $\varphi(\cdot) \colon \Real^{D_z} \to \mathcal{H}
\colon \vect{z} \mapsto \varphi( \vect{z} ) \coloneqq \kappa(\cdot, \vect{z})$. A popular RKHS example is
the space $\Hilbert_{ \vect{C} }$, induced by the following Gaussian kernel
\begin{align}
  \kappa ( \vect{z}^{\prime}, \vect{z} ) & \coloneqq \mathscr{G}(\vect{z}^{\prime} \mid \vect{z}, \vect{C} ) \coloneqq
  e^{ - (\vect{z}^{\prime} - \vect{z})^{\intercal} \vect{C}^{-1} (\vect{z}^{\prime} - \vect{z}) } \,, \label{eq:gaussian}
\end{align}
for an arbitrarily fixed matrix $\vect{C}$ in the set of all $D_z \times D_z$ positive definite matrices
$\PS^{D_z}$. In this case, the feature mapping becomes $\varphi( \vect{z} ) = \mathscr{G}( \cdot \mid
\vect{z}, \vect{C} )$, and
\begin{subequations}\label{single.gaussian.rkhs}
  \begin{alignat}{2}
    \Hilbert_{ \vect{C} }^{\text{pre}}
    & \coloneqq \Bigl \{ \sum\nolimits_{ k = 1}^K && \xi_k \mathscr{G}
    (\cdot \mid \vect{z}_k, \vect{C}) \mathop{} \big | \mathop{} \Set{ \xi_k }_{k=1}^K \subset \Real, \notag \\
    &&& \Set{\vect{z}_k}_{k=1}^K \subset \Real^{D_z}, K \in \IntegerPP \Bigr \}\,, \label{pre.gaussian.rkhs} \\
    \Hilbert_{ \vect{C} } & \coloneqq \clo \Hilbert_{ \vect{C} }^{\text{pre}} && \,, \label{gaussian.rkhs}
  \end{alignat}
\end{subequations}
where \eqref{pre.gaussian.rkhs} is a pre-Hilbert space, while \eqref{gaussian.rkhs} completes that space
by the closure $\clo$, taken w.r.t.\ the topology arising from the inner product induced by
\eqref{eq:gaussian}.

\subsection{Bellman residuals}\label{subsec:approx-pe}

Since transition probabilities are typically unknown in
practice, computing conditional expectation in
\eqref{Bellman.standard.mu} is generally infeasible. To surmount this
limitation, RL studies usually consider a sampling dataset $\mathcal{D}_{\mu}[T] \coloneqq \{ (\,
\vect{s}_t, \vect{a}_t \coloneqq \mu( \vect{s}_t ), g_t \coloneqq g( \zeta(\vect{s}_t, \vect{a}_t)),
\vect{s}^{\prime}_t \, ) \}_{t=1}^{T} \subset \mathfrak{S} \times \mathfrak{A} \times \mathfrak{S} \times
\Real$, for a stationary $\mu(\cdot)$ and a usually large $T \in \IntegerPP$. Here, $\vect{s}^\prime_t$
denotes the successor state of $\vect{s}_t$, if action $\vect{a}_t$ is taken under policy
$\mu(\cdot)$. For convenience in notations, let $\vect{z}_t \coloneqq \zeta(\vect{s}_t, \vect{a}_t)$ and
$\vect{z}^\prime_t \coloneqq \zeta(\vect{s}^{\prime}_t, \mu(\vect{s}^{\prime}_t))$.

Lacking access to the conditional expectation, rather than relying on the ensemble-based formulation in
\eqref{Bellman.standard.mu}, the Bellman residual (BR) approach~\cite{williams94br, onlineBRloss:16,
  akiyama24nonparametric} computes the following \textit{empirical}\/ Bellman mapping:
\begin{align}
  \hat{\Bellman}_{\mu} Q \coloneqq \arg \min\nolimits_{ Q^{\prime} \in \mathcal{H} }
  & \overbrace{ \tfrac{1}{T}\sum\nolimits_{t=1}^{T} \big[ g_t + \alpha Q^{\prime} (\vect{z}_t') - Q^{\prime}
      (\vect{z}_t) \big]^2 }^{ \hat{\mathcal{L}}_{\mu}[T]( Q^{\prime} ) } \notag\\
  & + \rho \norm{ Q' - Q }^2_{\mathcal{H}} \,, \label{eq:BR.loss}
\end{align}
where $\rho$ is a regularization coefficient, introduced to stabilize solutions along the lines of the
celebrated Tikhonov regularization. The \textit{empirical}\/ loss $\hat{\mathcal{L}}_{\mu}[T](\cdot)$
serves as an estimate of the \textit{ensemble}\/ one:
\begin{align}
  \mathcal{L}_{\mu} (Q) \coloneqq \mathbb{E}\{\, [ g(\vect{z}) + \alpha Q(\vect{z}^{\prime}) - Q(\vect{z}) ]^2\, \}
  \,, \label{ensemble.loss}
\end{align}
where $\mathbb{E}\{ \cdot \}$ stands for expectation---a rigorous probabilistic discussion is deferred to
\cref{subsec:performance}.

RKHSs have been already widely employed not only in BR schemes but also in temporal-difference (TD)
learning, via gradient-descent-based solutions~\cite{sutton88td, onlineBRloss:16}, as well as in
least-squares (LS)TD~\cite{lagoudakis03lspi, xu07klspi, regularizedpi:16}. Owing to the representer
theorem~\cite{KimeldorfWahba1970, scholkopf2002learning}, all admissible Q-functions in an RKHS are
expressed as elements of the linear subspace of $\mathcal{H}$ spanned by $\Set{ \varphi(\vect{z}_t),
  \varphi(\vect{z}^{\prime}_t) }_{t=1}^T$. However, this property is precisely the source of the ``curse
of dimensionality'' in RKHS-based nonparametric methods~\cite{sutton88td, ormoneit02kernel,
  ormoneit:autom:02, bae:mlsp:11, lagoudakis03lspi, regularizedpi:16, onlineBRloss:16, xu07klspi, vu23rl,
  akiyama24proximal, akiyama24nonparametric}: increasing the representational capacity of $\Set{
  \varphi(\vect{z}_t), \varphi(\vect{z}^{\prime}_t) }_{t=1}^T$ requires a large number $T$ of data
samples, which enhances expressiveness but simultaneously impairs computational efficiency due to the
resulting increase in complexity. A comprehensive variational framework for nonparametric solutions via
RKHSs, encompassing (LS)TD and BR solutions, is presented in~\cite{akiyama24nonparametric}.

\section{The Class of GMM Q-Functions (GMM-QFs)}\label{sec:GMM-QFs}

\subsection{Definition of GMM-QFs}\label{sec:GMM-QFs.def}

Selecting an ``optimal'' kernel function in \eqref{single.gaussian.rkhs}---through the choice of the
matrix $\vect{C}$---to fit the data requires substantial manual effort and poses a significant challenge
in RKHS-based nonparametric methods~\cite{sutton88td, ormoneit02kernel, ormoneit:autom:02, bae:mlsp:11,
  lagoudakis03lspi, regularizedpi:16, onlineBRloss:16, xu07klspi, vu23rl, akiyama24proximal,
  akiyama24nonparametric}. This difficulty persists even when multiple kernels are
employed. Unfortunately, the need for careful kernel selection further exacerbates the notorious ``curse
of dimensionality'' in nonparametric designs, discussed at the end of \cref{subsec:approx-pe}.

To mitigate these shortcomings of nonparametric designs, this paper introduces a novel
\textit{parametric}\/ functional approximation model for Q-functions. Motivated by
GMMs~\cite{McLachlan:FiniteMixtures:00}, and for a user-defined $K\in \IntegerPP$, GMM-QFs are defined as
the following class of functionals:
\begin{align}
  \mathcal{Q}_K \coloneqq \Bigl\{ Q(\cdot) \coloneqq
  & \sum\nolimits_{k=1}^K \xi_k \mathscr{G}( \cdot \mid \vect{m}_k, \vect{C}_k) \mathop{} \big | \mathop{}
  \Set{\xi_k}_{k=1}^K \subset \Real\,, \notag \\
  & \Set{\vect{m}_k}_{k=1}^K \subset \Real^{D_z}\,, \Set{\vect{C}_k}_{k=1}^K \subset \PS^{D_z} \Bigr\}
  \,, \label{eq:gmm-q}
\end{align}
where the function $\mathscr{G}_k(\cdot) \coloneqq \mathscr{G}( \cdot \mid \vect{m}_k, \vect{C}_k)$ is
defined by \eqref{eq:gaussian}, with the symbol $\vect{m}_k$ used instead of $\vect{z}$ to underline the
fact that $\vect{m}_k$ serves as the ``mean'' of the Gaussian $\mathscr{G}_k$. Notice also that
$\mathcal{Q}_{K} \subset \mathcal{Q}_{K+1}$, $\forall K\in \IntegerPP$, so that $\cup_{ k = 1}^K
\mathcal{Q}_k = \mathcal{Q}_{K}$.

Without any loss of generality, no two $\mathscr{G}_k$s in the Q-function representation of
\eqref{eq:gmm-q} are considered to be identical, that is, $( \vect{m}_k, \vect{C}_k ) \neq (
\vect{m}_{k'}, \vect{C}_{k'} )$, $\forall (k, k') \in \overline{1,K}^2$. This is justified by the fact
that two or more $\mathscr{G}_k$s with common $( \vect{m}_k, \vect{C}_k )$ can be consolidated into a
single one with an aggregated mixing weight $\xi_k$. Actually, the following
\cref{thm:rkhs.linear.independence} has more to say about the Q-function representation in
\eqref{eq:gmm-q}.

\begin{theorem}\label{thm:rkhs.linear.independence}
  Consider $\{ \vect{C}_k \}_{k=1}^K \subset \PS^{D_z}$, with $\vect{C}_k \neq \vect{C}_{k'}$, $\forall
  k\neq k'$. Let $\sum\nolimits_{k=1}^{K} \beta_k Q_k = 0$, for some $\{ \beta_k \}_{k=1}^K \subset
  \Real$, where $Q_k \in \Hilbert_{\vect{C}_k}^{\text{pre}} \setminus \{ 0\}$, $\forall k \in
  \overline{1,K}$, and $\Hilbert_{\vect{C}_k}^{\text{pre}}$ is defined by
  \eqref{pre.gaussian.rkhs}. Then, $\beta_k = 0$, $\forall k \in \overline{1,K}$. Consequently,
  $\sum\nolimits_{k=1}^K \Hilbert_{\vect{C}_k}^{\text{pre}}$ turns out to be the direct sum
  $\bigoplus_{k=1}^K \Hilbert_{\vect{C}_k}^{\text{pre}}$.
\end{theorem}

\begin{proof}
  See \cref{app:direct.sum.proof}.
\end{proof}

\cref{thm:rkhs.linear.independence} suggests that any Q-function $Q$ in $\sum\nolimits_{k=1}^K
\Hilbert_{\vect{C}_k}^{\text{pre}}$, and certainly in its subset $\mathcal{Q}_K$, has a \textit{unique}\/
representation $Q = \sum\nolimits_{k=1}^{K} \xi_k Q_k$, where $\xi_k\in \Real$ and each $Q_k\in
\Hilbert_{\vect{C}_k}^{\text{pre}}$, $\forall k\in \overline{1, K}$.

The mixing weights $\Set{ \xi_k }_{k=1}^K$ in \eqref{eq:gmm-q} are not subject to any constraints, unlike
in distributional RL, where GMMs model PDFs and the mixing weights must be chosen so that $\int
Q(\vect{z})\, d\vect{z} = 1$~\cite{sato99em, mannor05rl,  agostini17gmmrl,
  choi19distRL}. Moreover, in contrast to nonparametric approaches~\cite{sutton88td, ormoneit02kernel,
  ormoneit:autom:02, bae:mlsp:11, lagoudakis03lspi, regularizedpi:16, onlineBRloss:16, xu07klspi, vu23rl,
  akiyama24proximal, akiyama24nonparametric}, where $\Set{ \vect{m}_k, \vect{C}_k }_{k=1}^K$ are
determined by the observed data and only $\Set{ \xi_k }_{k=1}^K$ are treated as learnable parameters,
GMM-QFs consider $\Set{ \vect{m}_k, \vect{C}_k }_{k=1}^K$ as learnable as well. The functional
representation in \eqref{eq:gmm-q} is also referred to in the literature as radial-basis-function
(RBF)~\cite{park93rbf} or elliptic-basis-function (EBF) networks~\cite{chen96ebf}.

Although GMM-QFs share a similar functional form with classical
  RBF-based KBRL methods~\cite{xu07klspi, vu23rl,
  ormoneit:autom:02, lagoudakis03lspi, bae:mlsp:11}, 
    the latter typically treat only the mixture
  weights as learnable, fixing the means and covariance matrices of the Gaussian kernels, as is
  characteristic of nonparametric methods where model size scales with the
  data~\cite{Gyorfi:NonParamReg:02}. By contrast, GMM-QFs learn \textit{all}\/ parameters for a fully
  adaptive parametrization, enabling greater representational capacity than typical RBF approximators
  while maintaining a fixed model size independent of the dataset. This learning is achieved via
  Riemannian optimization, as detailed in~\cref{sec:rpi}. 

Motivated by \eqref{eq:gmm-q}, define the following parameter space:
\begin{alignat}{3}
  \mathfrak{M}_K
  & {} \coloneqq {}
  && \Bigl\{ \vectgr{\Omega}
  && \coloneqq ( \xi_1, \ldots, \xi_K, \vect{m}_1, \ldots, \vect{m}_K, \vect{C}_1, \ldots, \vect{C}_K)  \mathop{} \big|
  \mathop{} \notag\\
  &&&&& \xi_k \in \Real, \vect{m}_k \in \Real^{D_z}, \vect{C}_k\in \PS^{D_z}, \forall k\in \overline{1,K} \Bigr\} \notag
  \\
  & = && \Real^K && \times (\Real^{D_z})^K \times ( \mathbb{S}_{++}^{D_z} )^K \,. \label{param.space}
\end{alignat}
The mapping
\begin{align}
  \mathscr{P}(\cdot) \colon \mathfrak{M}_K \to \mathcal{Q}_K \colon \vectgr{\Omega} \mapsto
  \sum\nolimits_{k=1}^K \xi_k \mathscr{G}( \cdot \mid \vect{m}_k, \vect{C}_k)
  \,, \label{param.space.to.Q_K}
\end{align}
is clearly surjective. For any Q-function $Q \in \mathcal{Q}_K$, the pre-image $\mathscr{P}^{-1}(Q)$ is
non-empty because $\mathscr{P} (\cdot)$ is surjective, though it may not be a singleton. For example,
when $K = 2$, parameter tuples $\vectgr{\Omega} \coloneqq (\xi_1, \xi_2, \vect{m}, \vect{m}, \vect{C},
\vect{C})$ and $\vectgr{\Omega}^{\prime} \coloneqq (\xi_1 + \xi_2, 0, \vect{m}, \vect{m}, \vect{C},
\vect{C})$ are both mapped to the same Q-function under $\mathscr{P}(\cdot)$. Therefore, if
$\mathscr{P}^{-1}(Q)$ is regarded as a single ``element''---more precisely, by introducing an equivalence
relation such that (s.t.) $\mathscr{P}^{-1}(Q)$ forms an equivalence class in the quotient space
$\mathfrak{M}_K$---then mapping $\mathscr{P}(\cdot)$, which maps equivalence classes to Q-functions, can
be also viewed as one-to-one, and hence bijective.

Interestingly, being the Cartesian product of well-known Riemannian manifolds in \eqref{param.space},
$\mathfrak{M}_K$ is itself Riemannian~\cite{RobbinSalamon:22}. Exploitation of the geometric structure of
Riemannian manifold of parameters constitutes a key objective of this work and will be discussed in depth
in~\cref{sec:rpi}.

{\color{black}
  \begin{remark}[Practical implementation]\label{remark:variant} In learning tasks with discrete and
    finite-action spaces, a natural and practical appealing variant of~\eqref{eq:gmm-q} assigns
    action-dependent mixture weights $\Set{\vectgr{\xi}^{(\vect{a})}}_{\vect{a} \in \mathfrak{A}}$ while
    sharing the means and covariance matrices $\Set{(\vect{m}_k, \vect{C}_k)}_{k=1}^{K}$ across all
    actions, yielding:
    \begin{align*}
      Q(\vect{s}, \vect{a}) \coloneqq \sum \nolimits_{k=1}^{K} \xi_k^{(\vect{a})} \mathscr{G}(\vect{s}
      \given \vect{m}_k, \vect{C}_k)\,, \quad \forall (\vect{s}, \vect{a})\,,
    \end{align*}
    as an effort to ``reduce the correlation'' between action-value estimates $Q(\zeta(\vect{s},
    \vect{a}))$, thereby allowing better discrimination among actions and improving the decision-making
    capability of agents. This variant is inspired by the output-layer of DQNs~\cite{mnih13dqn}, adopted
    also in~\cite{lagoudakis03lspi}. Owing to this new setting, the parameter space now becomes
    $\mathfrak{M}_K \coloneqq (\Real^{K})^{ \lvert \mathfrak{A} \rvert} \times (\Real^{D_s})^{K} \times
    (\PS^{D_s})^K$, with $D_s$ being the dimension of state $\vect{s}$, which retains the Riemannian
    product manifold structure, as the components over means and covariance matrices remain unchanged.
    Consequently, theoretical results, including~\cref{algo:PI,algo:armijo} carry over directly. For
    clarity, the theoretical development hereafter follows the original representation
    in~\eqref{eq:gmm-q}, while the action-dependent variant is reserved for the numerical tests
    in~\cref{sec:tests}.
  \end{remark}
}

\subsection{Representational capacity of GMM-QFs}\label{subsec:GMM-QFs-express}

The following assumption is commonly adopted in theoretical studies of RL, e.g., \cite{xu07klspi}, and is
motivated by practical hardware constraints, as sensors in real-world devices can only measure physical
quantities within finite operating ranges.

\begin{assumption}\label{ass:compact.Z}
    The state-action space $\mathfrak{Z} \subset
    \Real^{D_z}$ is compact.
\end{assumption}

The following theorem establishes that the proposed GMM-QFs in \eqref{eq:gmm-q} possess sufficient
representational capacity to approximate broad families of functions.

\begin{theorem}[Universal-approximation properties of GMM-QFs]\label{thm:universal.approx}\mbox{}
  \begin{thmlist}

  \item\label{thm:universal.approx.C} Under~\cref{ass:compact.Z}, the set $\cup_{K \in \IntegerPP}
    \mathcal{Q}_K$ is dense in the Banach space $C( \mathfrak{Z} )$ of all real-valued continuous
    functions defined on $\mathfrak{Z}$, equipped with the sup-norm $\norm{}_{\infty}$. In other words,
    for any $Q\in C( \mathfrak{Z} )$ and for any $\varepsilon \in \RealPP$, there exist a $K \in
    \IntegerPP$ and a $\tilde{Q} \in \mathcal{Q}_K$ s.t.\
    \begin{align}
      \norm{ Q - \tilde{Q} }_{\infty} \coloneqq \sup\nolimits_{\vect{z} \in \mathfrak{Z}} \lvert Q(\vect{z}) -
      \tilde{Q}(\vect{z}) \rvert \leq \varepsilon \,. \label{eq:approx.Q.dense.CZ}
    \end{align}

  \item\label{thm:universal.approx.L2.Borel} Under~\cref{ass:compact.Z}, $\cup_{ K\in \IntegerPP }
    \mathcal{Q}_K$ is dense in the Hilbert space $L_2( \mathfrak{P} ) \coloneqq L_2( \mathfrak{Z},
    \mathfrak{B}( \mathfrak{Z} ), \mathfrak{P} )$ of all real-valued square-integrable functionals on
    $\mathfrak{Z}$ w.r.t.\ the probability space $( \mathfrak{Z}, \mathfrak{B}( \mathfrak{Z} ),
    \mathfrak{P} )$, where $\mathfrak{B}( \mathfrak{Z} )$ is a Borel $\sigma$-algebra on $\mathfrak{Z}$
    and $\mathfrak{P}$ is a probability measure on $\mathfrak{B}( \mathfrak{Z}
    )$~\cite{Williams:ProMargin:91, Ash:ProMeasure:00}. In other words, for any $Q \in L_2( \mathfrak{P}
    )$ and for any $\varepsilon \in \RealPP$, there exist a $K \in \IntegerPP$ and a $\tilde{Q} \in
    \mathcal{Q}_K$ s.t.\
    \begin{align}
      \norm{ Q - \tilde{Q} }_{L_2( \mathfrak{P} ) } & \coloneqq \left(\, \int_{ \mathfrak{Z} } \lvert Q(\vect{z}) -
      \tilde{Q}(\vect{z}) \rvert^2 \, \mathfrak{P} (d\vect{z})\, \right)^{\frac{1}{2}} \notag \\
      & \coloneqq \bigl(\, \mathbb{E} \{ \lvert Q ( \cdot ) - \tilde{Q} ( \cdot ) \rvert^2 \}\, \bigr)^{\frac{1}{2}}
      \leq \varepsilon \,, \label{eq:approx.Q.dense.L2.Borel}
    \end{align}
    where $\mathbb{E}\{ \cdot\}$ denotes expectation, defined as the integral w.r.t.\ $\mathfrak{P}$.

  \item\label{thm:universal.approx.L2.Lebesgue} In the case where $\mathfrak{Z}$ becomes the non-compact
    $\Real^{D_z}$, $\cup_{ K\in \IntegerPP } \mathcal{Q}_K$ is dense in the Hilbert space $L_2 (
    \Real^{D_z} )$ of all real-valued square-integrable functionals on $\Real^{D_z}$ w.r.t.\ the
    classical Lebesgue measure~\cite{Williams:ProMargin:91, Ash:ProMeasure:00}. In
    other words, for any $Q \in L_2 ( \Real^{D_z} )$ and for any $\varepsilon \in \RealPP$, there exist a
    $K \in \IntegerPP$ and a $\tilde{Q} \in \mathcal{Q}_K$ s.t.\
    \begin{align}
      \norm{ Q - \tilde{Q} }_{L_2(\Real^{D_z}) } & \coloneqq \left( \int_{\Real^{D_z}} \lvert Q(\vect{z}) -
      \tilde{Q}(\vect{z}) \rvert^2 \, d\vect{z} \right)^{ \frac{1}{2} } \leq \varepsilon
      \,. \label{eq:approx.Q.dense.L2.Lebesgue}
    \end{align}

  \end{thmlist}
\end{theorem}

\begin{proof}
  See \cref{app:universal.approx.proof}.
\end{proof}

A similar result to \Cref{thm:universal.approx.L2.Lebesgue} is presented in~\cite{park93rbf}, where it is
shown that RBF networks with \textit{diagonal}\/ covariance matrices are dense in the spaces $L_p(
\Real^{D_z} )$, $p \in \IntegerPP$. \Cref{thm:universal.approx.L2.Lebesgue} extends the result
of~\cite{park93rbf} to general non-diagonal covariance matrices for the case of $L_2( \Real^{D_z} )$.

Motivated by \cref{thm:universal.approx.C}, the following proposition shows that it is sufficient to
search within the class of GMM-QFs, $\cup_{K \in \IntegerPP} \mathcal{Q}_K$, for a minimizer of the
empirical loss $\hat{\mathcal{L}}_{\mu}[T](\cdot)$ defined in~\eqref{eq:BR.loss}, rather than over the
broader class $C(\mathfrak{Z})$.

\begin{proposition}\label{prop:inf.Q_K} Under \cref{ass:compact.Z},
  \begin{align*}
    \inf_{ Q \in C( \mathfrak{Z} ) } \hat{\mathcal{L}}_{\mu}[T] ( Q ) = \inf_{ Q \in \cup_{K\in \IntegerPP} \mathcal{Q}_K }
    \hat{\mathcal{L}}_{\mu}[T] ( Q ) \,.
  \end{align*}
\end{proposition}

\begin{proof}
  See \cref{app:prop:inf.Q_K}.
\end{proof}

Guided by \cref{thm:universal.approx,prop:inf.Q_K}, and the discussion on \eqref{param.space.to.Q_K}, the
subsequent search for suitable GMM-QFs will be conducted over the Riemannian parameter space
$\mathfrak{M}_K$ defined in \eqref{param.space}, for a sufficiently large value of $K$.

\section{Policy Iteration by Riemannian Optimization}\label{sec:rpi}

By \eqref{param.space} and \eqref{param.space.to.Q_K}, the empirical loss
$\hat{\mathcal{L}}_{\mu}[T](\cdot)$ of \eqref{eq:BR.loss} is minimized over the parameter space
$\mathfrak{M}_K$ for a fixed user-defined $K$---a Riemannian manifold according to the discussion at the
end of \cref{sec:GMM-QFs.def}---as follows:
\begin{align}
  \min_{\vectgr{\Omega} \in \mathfrak{M}_K} \bigg(
  \hat{\mathcal{L}}_{\mu}[T]( \vectgr{\Omega} ) \coloneqq \tfrac{1}{T} \sum_{t=1}^{T} \Bigl[
    & g_t + \alpha \sum_{k=1}^{K}\xi_k \mathscr{G}_k(\vect{z}^\prime_t) \notag\\
    & - \sum_{k=1}^{K}\xi_k \mathscr{G}_k(\vect{z}_t) \Bigr]^2 \bigg) \,, \label{parameters.task}
\end{align}
where $\mathscr{G}_k(\cdot) \coloneqq \mathscr{G}(\cdot \given \vect{m}_k, \vect{C}_k)$ is defined in
\eqref{eq:gaussian}, while $\vect{z}_t \coloneqq \zeta(\vect{s}_t, \vect{a}_t)$ and $\vect{z}_t^{\prime}
\coloneqq \zeta( \vect{s}_t^{\prime}, \mu(\vect{s}_t^{\prime}) )$. Task~\eqref{parameters.task},
addressed through \cref{algo:armijo}, is integrated into the policy-evaluation stage
(line~\ref{algo:update.omega}) of the PI \cref{algo:PI}---apparently for the first time in the RL
literature. Because of the discussion surrounding \eqref{param.space.to.Q_K}, variables $\vectgr{\Omega}$
and $Q$ will be used interchangeably as arguments of $\hat{\mathcal{L}}_{\mu}[T](\cdot)$ in the
sequel.

The remainder of this section details the Riemannian optimization task~\eqref{parameters.task}, beginning
with a brief review of fundamental concepts from Riemannian geometry.

\subsection{Preliminaries on Riemannian manifolds}\label{subsec:manifold}

Loosely speaking, a \textit{(smooth) manifold}\/ $\mathfrak{M}$ is a topological space such that any
point $\vectgr{\Omega} \in \mathfrak{M}$ has a neighborhood that is homeomorphic to an open subset of a
Euclidean space~\cite[\S2.8.1]{RobbinSalamon:22}. A tangent vector of $\mathfrak{M}$ at $\vectgr{\Omega}$
is the ``velocity'' $\dot{\gamma} (0)$ of a curve $\gamma\colon \Real \to \mathfrak{M} \colon t \mapsto
\gamma(t)$, where $\gamma(0) = \vectgr{\Omega}$~\cite[Def.~2.8.17]{RobbinSalamon:22}. The \textit{tangent
  space}\/ $T_{\vectgr{\Omega}} \mathfrak{M}$ of $\mathfrak{M}$ at $\vectgr{\Omega}$ is the linear vector
space of all tangent vectors of $\mathfrak{M}$ at
$\vectgr{\Omega}$~\cite[Def.~2.8.14]{RobbinSalamon:22}. The \textit{retraction}\/ mapping $R$ on
$\mathfrak{M}$ is a smooth mapping from the tangent bundle $T\mathfrak{M}$---the collection of all
tangent spaces of $\mathfrak{M}$---to $\mathfrak{M}$ itself s.t.\ specific properties hold
true~\cite[Def.~4.1.1]{Absil:OptimManifolds:08}. Its restriction to $T_{ \vectgr{\Omega} }
\mathfrak{M}$ is denoted by $R_{ \vectgr{\Omega} }$.

A Riemannian manifold $\mathfrak{M}$ is a smooth manifold equipped with a Riemannian metric, that is, a
collection of inner products $\innerp{ \cdot }{ \cdot }_{ \vectgr{\Omega} }$, one for every
$\vectgr{\Omega} \in \mathfrak{M}$ and defined on pairs of tangent vectors in $T_{\vectgr{\Omega}}
\mathfrak{M}$, s.t.\ it induces a specific smooth map for every pair of vector fields on
$\mathfrak{M}$~\cite[Def.~3.7.1]{RobbinSalamon:22}. The metric naturally defines the norm $\norm{\cdot}_{
  \vectgr{\Omega} } \coloneqq \innerp{ \cdot }{ \cdot }_{ \vectgr{\Omega} }^{1/2}$. For a Riemannian
manifold, its most celebrated retraction is its \textit{exponential}\/ map~\cite{Absil:OptimManifolds:08,
  RobbinSalamon:22}. Further details on Riemannian geometry fall outside the scope of this study; the
interested reader is referred to~\cite{Absil:OptimManifolds:08, RobbinSalamon:22}.

\begin{example}[\cite{Absil:OptimManifolds:08}]\label{ex:euclid.manifold}
  The Euclidean space $\Real^K$ is a Riemannian manifold with $T_{ \vectgr{\xi} } \Real^K = \Real^K$,
  $\forall \vectgr{\xi} \in \Real^K$, and with Riemannian metric $\innerp{ \vectgr{\theta}_1 }{
    \vectgr{\theta}_2 }_{ \vectgr{\xi} } \coloneqq \vectgr{\theta}_1^{\intercal} \vectgr{\theta}_2$,
  $\forall (\vectgr{\theta}_1, \vectgr{\theta}_2, \vectgr{\xi} ) \in (\Real^K)^3$. The mapping
  $R_{\vectgr{\xi}} (\vectgr{\theta}) \coloneqq \vectgr{\xi} + \vectgr{\theta}$, $\forall
  (\vectgr{\theta}, \vectgr{\xi} ) \in (\Real^K)^2$, serves as a retraction on $\Real^K$.
\end{example}

\begin{example}[{\cite[\S6.5.3]{RobbinSalamon:22}}]\label{ex:PS.manifold}
  The set of all positive-definite matrices $\PS^{D_z}$ is a Riemannian manifold, with $T_{ \vect{C} }
  \PS^{D_z} = \{ \vectgr{\Gamma} \in \Real^{D_z \times D_z} \given \vectgr{\Gamma}^{\intercal} =
  \vectgr{\Gamma} \} \eqqcolon \mathbb{S}^{ D_z }$, $\forall \vect{C} \in \PS^{D_z}$. The most popular
  metric is the affine invariant (AffI) one~\cite[(6.5.11)]{RobbinSalamon:22}: $\forall \vect{C} \in
  \PS^{D_z}$, $\forall (\vectgr{\Gamma}_1, \vectgr{\Gamma}_2) \in (T_{\vect{C}}\PS^{D_z})^2$,
  \begin{equation}
      \innerp{\vectgr{\Gamma}_{1}}{\vectgr{\Gamma}_{2}}_{\vect{C}}
      = \innerp{\vectgr{\Gamma}_{1}}{\vectgr{\Gamma}_{2}}^{\textnormal{AffI}}_{\vect{C}} \coloneqq \trace(\vect{C}^{-1}
    \vectgr{\Gamma}_{1} \vect{C}^{-1} \vectgr{\Gamma}_{2}) \,, \label{eq:ai.metric}
  \end{equation}
  where $\trace (\cdot)$ denotes the trace of a square matrix.
  Note that, other options for Riemannian metric, such as
    the Bures-Wasserstein
    one~\cite{MalagoMontrucchioPistone:18}, can also be used
    in this framework. However, this paper only considers
    the most popular AffI metric due to its availability and
    simple closed-form.

  The exponential map $\exp_{\vect{C}}(\cdot) \colon T_{\vect{C}}\PS^{D_z} \to \PS^{D_z} \colon
  \vectgr{\Gamma} \mapsto \exp_{\vect{C}}(\vectgr{\Gamma})$
    under the AffI is defined as
  follows~\cite[(6.5.18)]{RobbinSalamon:22}: $\forall \vect{C} \in
  \PS^{D_z}$, $\forall \vectgr{\Gamma} \in T_{\vect{C}}\PS^{D_z}$,
    \begin{align}
      \exp_{ \vect{C} } ( \vectgr{\Gamma} )
      & \coloneqq \vect{C}^{1/2}\, \Exp (\vect{C}^{-1/2} \vectgr{\Gamma} \vect{C}^{-1/2})\,
      \vect{C}^{1/2}\,, \label{exp.PD}
    \end{align}
  where $\Exp(\cdot)$ stands for the matrix exponential~\cite[\S2.5.2]{RobbinSalamon:22}.
\end{example}

\begin{figure}[t]
    \centering
    \includegraphics[width=.7\columnwidth]{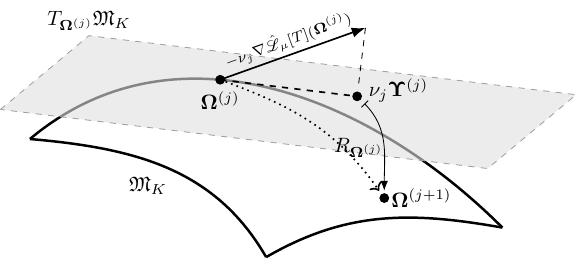}
    \caption{Gradient of the loss $\hat{\mathcal{L}}_{\mu}[T]( \cdot )$ at $\vectgr{\Omega}^{(j)}$, with
      $\nu_j$ being the step-size. In general, the gradient is first projected
      onto the tangent space $T_{ \vectgr{\Omega}^{(j)} } \mathfrak{M}_K$ and then retracted back to the
      manifold $\mathfrak{M}_K$. In the present case, however, this projection is unnecessary because, as
      shown in \cref{prop:gradients}, the computed gradient already lies in the tangent
      space. }\label{fig:retraction}
\end{figure}

While the positive definiteness of a covariance matrix $\vect{C}$ can be enforced via log-Euclidean
parameterization~\cite{log-euclidean} or Cholesky factorization~\cite{cholesky-param}, both approaches
exhibit shortcomings in the current setting. Log-Euclidean parameterization requires eigendecomposition
to compute gradients through the matrix exponential, which is numerically unstable near zero
eigenvalues~\cite{vanish-grad, eccv, backprop-friendly}. This instability is intensified in the current
setting because GMM-QFs already involve exponential terms, and applying log-Euclidean parameterization to
the Gaussian exponents creates nested exponential-logarithm compositions which amplify gradient scaling,
analogous to exploding/vanishing gradient issues~\cite{vanish-grad}. Cholesky-based factorizations suffer
similar issues due to the diagonal entries of triangular factors~\cite{lc-metric}. By contrast, the
Riemannian approach avoids these ``double-exponential'' complications, and operates directly on a
Riemannian product manifold $\mathfrak{M}_K$ which arises effortlessly from the nature of the problem;
$\PS^{D_z}$ is intrinsically a Riemannian manifold.

\subsection{Algorithms}

Notice that for the Riemannian manifold $\mathfrak{M}_K$, its tangent space at a point $\vectgr{\Omega}
\coloneqq ( \vectgr{\xi}, \vect{m}_{1}, \dots, \vect{m}_{K}, \vect{C}_{1}, \dots, \vect{C}_{K} ) \in
\mathfrak{M}_K$ becomes $T_{\vectgr{\Omega}} \mathfrak{M}_K = T_{\vectgr{\xi}} \Real^K \times
T_{\vect{m}_{1}} \Real^{D_z} \times \cdots \times T_{\vect{m}_{K}} \Real^{D_z} \times T_{ \vect{C}_{1} }
\PS^{D_z} \times \cdots \times T_{ \vect{C}_{K} } \PS^{D_z} = \Real^K \times (\Real^{D_z})^K \times
(\mathbb{S}^{D_z})^K$, where it is reminded here that $\mathbb{S}^{D_z}$ stands for the set of all $D_z
\times D_z$ symmetric matrices. To run computations on $\mathfrak{M}_K$, the following straightforward
metric is adopted: $\forall \vectgr{\Upsilon}_i \coloneqq ( \vectgr{\theta}_i, \vectgr{\mu}_{i1}, \dots,
\vectgr{\mu}_{iK}, \vectgr{\Gamma}_{i1}, \dots, \vectgr{\Gamma}_{iK} ) \in T_{\vectgr{\Omega}}
\mathfrak{M}_K$, $i \in \overline{1, 2}$,
\begin{alignat}{2}
  \innerp{ \vectgr{\Upsilon}_1} { \vectgr{\Upsilon}_2 }_{ \vectgr{\Omega} } & {} \coloneqq {} &&
  \vectgr{\theta}_1^{\intercal} \vectgr{\theta}_2 + \sum\nolimits_{k=1}^K \vectgr{\mu}_{1k}^{\intercal}
  \vectgr{\mu}_{2k} \notag\\
  &&& + \sum\nolimits_{k=1}^K \innerp{ \vectgr{\Gamma}_{1k} }{ \vectgr{\Gamma}_{2k} }_{ \vect{C}_k }
  \,, \label{Rmetric.on.M}
\end{alignat}
where $\innerp{ \cdot }{ \cdot }_{ \vect{C}_k }$ can be any user-defined Riemannian metric on
$\PS^{D_z}$. Here, the AffI given in \eqref{eq:ai.metric}
will serve as the basis for computations.

The solution to \eqref{parameters.task} relies on a standard
steepest-gradient-descent scheme on Riemannian manifold
detailed in \cref{algo:armijo} and
invoked in line~\ref{algo:update.omega} of the policy-evaluation step in the PI \cref{algo:PI}---see also
\cref{fig:retraction}. To apply \cref{algo:armijo}, the gradients in line~\ref{algo:compute.grads} of the
same algorithm must be evaluated, and their closed-form expressions are provided in the following
\cref{prop:gradients}. To avoid also unnecessary clutter in notations, the superscript $(n)$ on the
generated data is often omitted.

\begin{algorithm}[t!]
  \caption{Policy iteration by Riemannian optimization}\label{algo:PI}
  \begin{algorithmic}[1]
    \renewcommand{\algorithmicindent}{1em}

    \State{Arbitrarily initialize $\vectgr{\Omega}_0 \in \mathfrak{M}_K$, $\mu_1 \in \mathcal{M}$.}

    \While{$n \in \IntegerPP$} \label{line:iter}

      \State{\textbf{Policy evaluation} Use the current policy $\mu_n$ to generate the dataset
        $\mathcal{D}_{\mu_n} [T] \coloneqq \Set{( \vect{s}_t^{(n)}, \vect{a}_t^{(n)}, g_t^{(n)},
          \vect{s}_t^\prime{}^{(n)} )}_{t=1}^{T}$.}

      \State{Update $\vectgr{\Omega}_{n}$ via~\cref{algo:armijo}.}\label{algo:update.omega}

      \State{Given $\vectgr{\Omega}_{n}$, define $Q_{n} \coloneqq \mathscr{P}( \vectgr{\Omega}_{n} )$
        via~\eqref{param.space.to.Q_K}.}\label{algo:ROforPE}

      \State{\textbf{Policy improvement} Update policy
      $\mu_{n+1}$ via~\eqref{eq:policy-impr}. }
      \label{line:policy-impr}

      \State{Increase $n$ by one, go to~\cref{line:iter}.}

    \EndWhile
  \end{algorithmic}
\end{algorithm}

\begin{algorithm}[t!]
  \caption{Solving~\eqref{parameters.task}}\label{algo:armijo}
  \begin{algorithmic}[1]
    \renewcommand{\algorithmicindent}{1em}

    \State{\textbf{Require:} 
      Number of steps $J\in \IntegerPP$.}

    \State{$\vectgr{\Omega}^{(0)} \coloneqq \vectgr{\Omega}_{n-1}$.}

    \For{$j=0$ to $J-1$}

      \State{$\vectgr{\Omega}^{(j)} \eqqcolon ( \vectgr{\xi}^{(j)}, \vect{m}_1^{(j)}, \dots, \vect{m}_K^{(j)},
        \vect{C}_1^{(j)}, \dots, \vect{C}_K^{(j)})$.}\label{algo:jth.estimate}

      \State{Use \cref{prop:gradients} to compute:
        \begin{align*}
          \nabla \hat{\mathcal{L}}_{\mu_n}[T] ( \vectgr{\Omega}^{(j)} ) = \Bigl( & \tfrac{
            \partial\hat{\mathcal{L}}_{\mu_n}[T] } { \partial \vectgr{\xi} } ( \vectgr{\Omega}^{(j)} ), \ldots, \tfrac{
            \partial\hat{\mathcal{L}}_{\mu_n}[T] } { \partial \vect{m}_k } ( \vectgr{\Omega}^{(j)} ), \\
          & \ldots, \tfrac{ \partial\hat{\mathcal{L}}_{\mu_n}[T] } { \partial \vect{C}_k } ( \vectgr{\Omega}^{(j)} ),
          \ldots \Bigr) \,.
        \end{align*}
      }\label{algo:compute.grads}

      \State{Let
        \begin{align*}
          \vectgr{\Upsilon}^{(j)} & \coloneqq ( \vectgr{\theta}^{(j)}, \vectgr{\mu}_{1}^{(j)}, \ldots,
          \vectgr{\mu}_{K}^{(j)}, \vectgr{\Gamma}_{1}^{(j)}, \ldots, \vectgr{\Gamma}_{K}^{(j)} ) \\
          & \coloneqq -\nabla \hat{\mathcal{L}}_{\mu_n}[T](\vectgr{\Omega}^{(j)}) \,.
        \end{align*}
      }

      \State{Define the step-size $\nu_j \in \RealPP$. 
      }

      \State{Update $\vectgr{\Omega}^{(j+1)} \coloneqq R_{\vectgr{\Omega}^{(j)}} (\nu_j
        \vectgr{\Upsilon}^{(j)} )$ via~\eqref{eq:retraction}.}

    \EndFor

    \State{$\vectgr{\Omega}_{n} \coloneqq \vectgr{\Omega}^{(J)}$.}

  \end{algorithmic}
\end{algorithm}

\begin{proposition}[Computing gradients]\label{prop:gradients}
  Consider the estimate $\vectgr{\Omega}^{(j)} \in \mathfrak{M}_K$ in line~\ref{algo:jth.estimate} of
  \cref{algo:armijo}, and its associated GMM-QF
    $
    Q^{(j)}(\cdot) \coloneqq \sum\nolimits_{k=1}^K \xi_k^{(j)} \mathscr{G}( \cdot \mid \vect{m}_k^{(j)},
    \vect{C}_k^{(j)}) \,.$
  Let also
  \begin{alignat*}{2}
    \delta_t^{(j)} & {} \coloneqq {} && g_t + \alpha Q^{(j)} (\vect{z}_t^\prime) - Q^{(j)} (\vect{z}_t) \,, \\
    \vect{d}_{tk}^{(j)} & \coloneqq && \alpha \mathscr{G}(\vect{z}^\prime_t \given
    \vect{m}^{(j)}_k, \vect{C}^{(j)}_k) \cdot (\vect{z}^\prime_t - \vect{m}^{(j)}_k) \\
    &&& - \mathscr{G}(\vect{z}_t \given \vect{m}^{(j)}_k, \vect{C}^{(j)}_k) \cdot (\vect{z}_t - \vect{m}^{(j)}_k) \,, \\
    \bar{\vect{d}}_{k}^{(j)} & \coloneqq && \tfrac{1}{T} \sum\nolimits_{t=1}^{T} \delta_t^{(j)} \vect{d}_{tk}^{(j)} \,, \\
    \vect{B}_{tk}^{(j)} & \coloneqq && \alpha \mathscr{G} (\vect{z}^\prime_t \given \vect{m}_k^{(j)}, \vect{C}_k^{(j)} )
    \cdot (\vect{z}^\prime_t - \vect{m}_k^{(j)} ) (\vect{z}^\prime_t - \vect{m}_k^{(j)} )^\intercal  \\
    &&& - \mathscr{G} (\vect{z}_t \given \vect{m}_k^{(j)}, \vect{C}_k^{(j)} ) \cdot (\vect{z}_t - \vect{m}_k^{(j)} )
    (\vect{z}_t - \vect{m}_k^{(j)} )^\intercal \,,\\
    \bar{\vect{B}}_k^{(j)} & \coloneqq && \tfrac{1}{T} \sum\nolimits_{t=1}^{T} \delta_t^{(j)} \vect{B}_{tk}^{(j)} \,.
  \end{alignat*}

  \begin{subequations}\label{all.gradients}
    \begin{proplist}

    \item Notice first that the loss in~\eqref{parameters.task} can be recast as
      $\hat{\mathcal{L}}_{\mu}[T]( \vectgr{\Omega}^{(j)} ) = (1/T)\, \norm{\vect{g} + \vectgr{\Delta}
      \vectgr{\xi}^{(j)} }^2$, where $\vect{g} \coloneqq [g_1, \ldots, g_T]^\intercal$ and
      $\vectgr{\Delta}^{(j)}$ is the $T \times K$ matrix having its $(t,k)$th entry as
      \begin{align*}
        [\vectgr{\Delta}^{(j)}]_{tk} \coloneqq \alpha \mathscr{G} (\vect{z}_t^\prime \given \vect{m}_k^{(j)},
        \vect{C}_k^{(j)} ) - \mathscr{G} (\vect{z}_t \given \vect{m}_k^{(j)}, \vect{C}_k^{(j)}) \,.
      \end{align*}
      Then,
      \begin{align}
        \frac{\partial\hat{\mathcal{L}}_{\mu}[T]}{\partial \vectgr{\xi}} ( \vectgr{\Omega}^{(j)} ) = \tfrac{2}{T}
        \vectgr{\Delta}^{(j)} {}^{\intercal} (\vect{g} + \vectgr{\Delta}^{(j)} \vectgr{\xi}^{(j)} )
        \,. \label{eq:dL.dxi}
      \end{align}

    \item $\forall k \in \overline{1,K}$,
      \begin{align}
        \frac{\partial \hat{\mathcal{L}}_{\mu}[T]}{\partial \vect{m}_k} (\vectgr{\Omega}^{(j)} ) =
        4 \xi_k^{(j)} (\vect{C}_k^{(j)})^{-1} \bar{\vect{d}}_{k}^{(j)} \,. \label{eq:dL.dm}
      \end{align}

    \item Under the AffI metric and $\forall k \in \overline{1,K}$,
      \begin{align}
        \frac{\partial \hat{\mathcal{L}}_{\mu}[T]}{\partial \vect{C}_k} (\vectgr{\Omega}^{(j)}) = 2\xi^{(j)}_k
        \bar{ \vect{B} }_k^{(j)}\ \in \mathbb{S}^{D_z} \,. \label{eq:dL.dC.ai}
      \end{align}
    \end{proplist}
  \end{subequations}

\end{proposition}

\begin{proof}
    See~\cref{dL.deriv}.
\end{proof}

Owing to the Cartesian structure of $\mathfrak{M}_K$, it is straightforward to verify that the retraction
mapping on $\mathfrak{M}_K$ becomes $R_{ \vectgr{\Omega} }(\cdot) = (\, R_{ \vectgr{\xi}}(\cdot), \ldots,
R_{ \vect{m}_k } (\cdot), \ldots, R_{ \vect{C}_k }(\cdot), \ldots \,)$. More specifically, within the
context of \cref{algo:armijo}, and motivated by \cref{ex:euclid.manifold,ex:PS.manifold}, 
\begin{subequations}\label{eq:retraction}
  \begin{alignat}{2}
    R_{ \vectgr{\Omega}^{(j)} }( \nu_j \vectgr{\Upsilon}^{(j)} ) & {} = {} && \Bigl( R_{ \vectgr{\xi}^{(j)} }(
    \nu_j \vectgr{\theta}^{(j)} ), \ldots, R_{ \vect{m}_k^{(j)} } ( \nu_j \vectgr{\mu}_k^{(j)} ),
    \notag\\
    &&& \hphantom{ \Bigl( } \ldots, R_{ \vect{C}_k^{(j)} }( \nu_j \vectgr{\Gamma}_k^{(j)} ), \ldots \Bigr)\,,
    \\
    R_{ \vectgr{\xi}^{(j)} }( \nu_j \vectgr{\theta}^{(j)} ) & {} \coloneqq {} && \vectgr{\xi}^{(j)} +
    \nu_j \vectgr{\theta}^{(j)} \,, \\
    R_{ \vect{m}_k^{(j)} }( \nu_j \vectgr{\mu}_k^{(j)} ) & \coloneqq && \vect{m}_k^{(j)} + \nu_j
    \vectgr{\mu}_k^{(j)} \,, \\
    R_{ \vect{C}_k^{(j)} }( \nu_j \vectgr{\Gamma}_k^{(j)} ) & \coloneqq && \exp_{ \vect{C}_k^{(j)} } (
      \nu_j \vectgr{\Gamma}_k^{(j)} ) \label{eq:exp.C} \,,
  \end{alignat}
\end{subequations}
where $\exp_{ \vect{C}_k^{(j)} }( \cdot )$ stands for any choice of the exponential map on $\PS^{D_z}$;
see, for example, \eqref{exp.PD}.

Although the exponential map~\eqref{eq:exp.C} following~\eqref{exp.PD} involves eigendecomposition under
AffI metric, this does not reintroduce the numerical instability that motivates Riemannian approach.
In contrast to log-Euclidean parametrization which introduces eigendecomposition directly to gradient
computations, the retraction via~\eqref{exp.PD} is applied \textit{after}\/ the descent step, operating on
the scaled update $\nu_j \vectgr{\Upsilon}^{(j)}$ whose magnitude is controlled by choosing appropriate
$\nu_j$, e.g., see~\cref{subsec:stepsize}. The eigendecomposition in the retraction on $\PS^{D_z}$ is
therefore evaluated in a numerically bounded regime and does not propagate instability into gradients.

  of~\cref{algo:PI} operates in a discrete, finite action space $\mathfrak{A}$, consistent with standard
  PI literature~\cite{Bertsekas:RLandOC:19}. While this simplifies the update to direct minimization, the
  proposed framework is naturally amenable to continuous-action settings. For instance, an actor-critic
  approach~\cite{konda99ac} could employ a separate policy approximator alongside GMM-QFs. Developing
  such extensions is a promising direction currently being
  pursued. 

\subsection{Step-size strategy}\label{subsec:stepsize}

A popular strategy to choose the step-size $\nu_j$ is to follow the Armijo
line-search algorithm~\cite[\S4.6.3]{Absil:OptimManifolds:08}. In particular, for user-defined
hyperparameters $\bar{\alpha}\in \RealPP, (\beta, \sigma_{\textnormal{A}}) \in (0,1)^2$,
$\nu_j$ is defined via the
quantity $M_{\textnormal{a}} \in \IntegerPP$ as $\nu_j
\coloneqq \bar{\alpha} \beta^{M_{\textnormal{a}}}$ such that,
\begin{align}
    & \hat{\mathcal{L}}_{\mu_n}[T](\vectgr{\Omega}^{(j)}) - \hat{\mathcal{L}}_{\mu_n}[T] \left(
    R_{\vectgr{\Omega}^{(j)}} ( \bar{\alpha} \beta^{M_\textnormal{a}} \vectgr{\Upsilon}^{(j)} )
    \right) \notag \\
    & \geq - \sigma_{\textnormal{A}} \innerp{\nabla \hat{\mathcal{L}}_{\mu_n} [T] ( \vectgr{\Omega}^{(j)}
    )}{ \bar{\alpha} \beta^{M_\textnormal{a}} \vectgr{\Upsilon}^{(j)} }_{\vectgr{\Omega}^{(j)}}
    \notag \\
    & = \sigma_{\textnormal{A}} \bar{\alpha} \beta^{M_\textnormal{a}}\, \norm{ \nabla \hat{\mathcal{L}}_{\mu_n} [T]
    ( \vectgr{\Omega}^{(j)} )}_{ \vectgr{\Omega}^{(j)} }^2 \label{algo:armijo.condition} \,.
\end{align}

\begin{theorem}\label{thm:armijo}
  Let $\mathfrak{C}_n$ denote the set of all accumulation/cluster points of the sequence generated by
  \cref{algo:armijo} per PI iteration $n$, for $J\to \infty$ and with the step-sizes chosen according
  to \eqref{algo:armijo.condition}. Then, every such accumulation point is a \textit{critical}\/ point of
  the loss, i.e., $\nabla \hat{\mathcal{L}}_{\mu_n}[T] (\vectgr{\Omega}_n^{*}) = \vect{0}$, $\forall
  \vectgr{\Omega}_n^{*} \in \mathfrak{C}_n$.
\end{theorem}

\begin{proof}
    See~\cite[\S4.3.3]{Absil:OptimManifolds:08}.
\end{proof}

Note that, due to the non-convexity of the objective
$\hat{\mathcal{L}}_{\mu_n} [T]$, locating its critical points does not guarantee the identification of
its global minimizers. Nevertheless, the Armijo condition~\cite{Absil:OptimManifolds:08,
  Boumal:IntroMan:23} in~\eqref{algo:armijo.condition}, used as a backtracking
strategy, guarantees a sufficient decrease of the loss~\cite[Cor.~4.13,
  p.~63]{Boumal:IntroMan:23}. Moreover, rather than letting $J \to \infty$ to locate a critical point as
suggested by \cref{thm:armijo}, \cref{algo:armijo} adopts a more pragmatic approach by taking the
iteration budget $J$ to be sufficiently large, effectively implementing an ``early-stopping'' strategy
similar to that commonly used in neural-network training to mitigate overfitting~\cite{Prechelt:12}.

\subsection{Performance analysis of \cref{algo:PI}}\label{subsec:performance}

For the following discussion, a probability space $(\Lambda, \mathfrak{F}, \mathbb{P})$ is considered,
where $\Lambda$ is the sample space, with $\lambda\in \Lambda$, $\mathfrak{F}$ is the $\sigma$-algebra of
the events on $\Lambda$, and $\mathbb{P}$ denotes a probability measure on
$\mathfrak{F}$~\cite{Williams:ProMargin:91, Ash:ProMeasure:00}. State vectors
$\vect{s}$ and action vectors $\vect{a}$ will henceforth be regarded as observations (realizations) of
the random variables (RVs) $\bm{s}$ and $\bm{a}$ on $(\Lambda, \mathfrak{F}, \mathbb{P})$,
respectively. Likewise, the slanted $\bm{z}$ will be used to denote the ``state-action'' RV whose
observation is $\vect{z} = \zeta (\vect{s}, \vect{a})$.

To sidestep unnecessary technical complications, it is assumed that all RVs admit a probability density
function (PDF)~\cite[\S6.12]{Williams:ProMargin:91}---for example, the PDF of the RV $\bm{z}$ is denoted
by $p_{ \bm{z} }(\cdot) \colon \mathfrak{Z} \to [0, +\infty] \colon \vect{z} \mapsto p_{ \bm{z} }(
\vect{z} )$. As such, $\mathbb{E}\{ \lvert Q(\bm{z}) \rvert^2 \} \coloneqq \int_{ \Lambda } \lvert Q
(\bm{z}(\lambda) ) \rvert^2\, \mathbb{P}( d\lambda) = \int_{ \mathfrak{Z} } \lvert Q ( \vect{z} )
\rvert^2 p_{ \bm{z} } ( \vect{z} )\, d \vect{z}$~\cite[\S6.12]{Williams:ProMargin:91}, where
$\mathbb{E}\{ \cdot \}$ stands for expectation, defined as the integral w.r.t.\ $\mathbb{P}$. Recall also
the well-known Hilbert space $L_2( \mathbb{P} ) \coloneqq L_2( \Lambda, \mathfrak{F}, \mathbb{P} )
\coloneqq \{ Q(\bm{z}) \given \mathbb{E}\{ \lvert Q(\bm{z}) \rvert^2 \} < + \infty\}$, endowed with the
inner product $\innerp{ Q(\bm{z}) }{ Q^{\prime} (\bm{z}^{\prime}) }_{ L_2( \mathbb{P} ) } \coloneqq
\mathbb{E}\{ Q(\bm{z}) Q^{\prime} (\bm{z}^{\prime})\}$~\cite{Williams:ProMargin:91}.

\begin{lemma}\label{lemma:L2.spaces.equiv}
  The RV $Q(\bm{z}) \in L_2( \Lambda, \mathfrak{F}, \mathbb{P} )$ if and only if the Q-function $Q \in
  L_2 (\mathfrak{P}_{ \bm{z} }) \coloneqq L_2( \mathfrak{Z}, \mathfrak{B}(\mathfrak{Z}), \mathfrak{P}_{
    \bm{z} } )$, where $\mathfrak{B} (\mathfrak{Z})$ stands for the Borel $\sigma$-algebra on
  $\mathfrak{Z}$, and the Borel probability measure $\mathfrak{P}_{ \bm{z} }$ is induced by the PDF $p_{
    \bm{z} }(\cdot)$ as $\mathfrak{P}_{ \bm{z} } ( \mathscr{B} ) \coloneqq \int_{ \mathscr{B} } p_{
    \bm{z} }( \vect{z} )\, d \vect{z}$, $\forall \mathscr{B}\in \mathfrak{B}(\mathfrak{Z})$.
\end{lemma}

\begin{proof}
  The proof follows verbatim the proof of the lemma in \cite[\S6.12]{Williams:ProMargin:91}.
\end{proof}

\begin{assumptions}\mbox{}
  \begin{assslist}

  \item\label{ass:id} Any two state-action RVs $\bm{z}, \bm{z}^{\prime}$ are identically distributed,
    that is, $p_{ \bm{z} }( \cdot ) = p_{ \bm{z}^{\prime} }( \cdot )$, almost everywhere (a.e.) on
    $\mathfrak{Z}$. Consequently, the induced measure $\mathfrak{P}_{\bm{z}}$ in
    \cref{lemma:L2.spaces.equiv} remains invariant w.r.t.\ the choice of the RV $\bm{z}$ and will be
    denoted hereafter by $\mathfrak{P}$.

  \item\label{ass:one.step.loss.L2} The one-step loss $g(\cdot)$ belongs to $L_2 (\mathfrak{P})$.

  \end{assslist}
\end{assumptions}

Owing to \cref{ass:id}, \cref{lemma:L2.spaces.equiv} now provides the direct connection between the
present discussion and \cref{thm:universal.approx.L2.Borel}. In particular, under
\cref{ass:compact.Z,ass:id}, \cref{thm:universal.approx.L2.Borel} shows that $\cup_{K\in \IntegerPP}
\mathcal{Q}_K$ is dense in $L_2(\mathfrak{Z}, \mathfrak{B}(\mathfrak{Z}),
\mathfrak{P})$. \cref{ass:one.step.loss.L2} is used in \cref{prop:Bellman.contraction.L2} to demonstrate
that the classical Bellman mappings $\Bellman^{\diamond}_{\mu}$ and $\Bellman^{\diamond}$ are well
defined on $L_2(\mathfrak{P})$.

\begin{proposition}\label{prop:Bellman.contraction.L2}
  Under \cref{ass:id,ass:one.step.loss.L2}, the Bellman mappings $\Bellman^{\diamond}_{\mu},
  \Bellman^{\diamond}$ in~\eqref{Bellman.maps.standard} map the Hilbert space $L_2 (\mathfrak{P})$ into
  itself. More specifically, they are contractions~\cite{hb.plc.book}---that is, $\forall Q_1, Q_2 \in
  L_2( \mathfrak{P} )$, $\forall \mu \in \mathcal{M}$, and regardless of the choice of the state-action
  RV $\bm{z}$,
  \begin{subequations}
    \begin{align}
      & \norm{ \Bellman^{\diamond}_{\mu} Q_1 - \Bellman^{\diamond}_{\mu} Q_2 }_{ L_2( \mathfrak{P} ) }
      \notag\\
      & = \left(\, \mathbb{E} \{\, \lvert (\Bellman^{\diamond}_{\mu} Q_1) (\bm{z}) - (\Bellman^{\diamond}_{\mu} Q_2)
      (\bm{z}) \rvert^2\, \}\, \right)^{ \frac{1}{2} } \notag \\
      & \leq \alpha \left(\, \mathbb{E} \{\, \lvert Q_1(\bm{z}) - Q_2(\bm{z}) \rvert^2\, \}\,
      \right)^{ \frac{1}{2} } = \alpha \norm{ Q_1 - Q_2 }_{ L_2( \mathfrak{P} ) } \label{eq:Bellman.mu.contract} \,,
      \\
      & \norm{ \Bellman^{\diamond} Q_1 - \Bellman^{\diamond} Q_2 }_{ L_2( \mathfrak{P} ) }
      \notag\\
      & = \left(\, \mathbb{E} \{\, \lvert (\Bellman^{\diamond} Q_1) (\bm{z}) - (\Bellman^{\diamond} Q_2)
      (\bm{z}) \rvert^2\, \}\, \right)^{ \frac{1}{2} } \notag \\
      & \leq \alpha \left(\, \mathbb{E} \{\, \lvert Q_1(\bm{z}) - Q_2(\bm{z}) \rvert^2\, \}\,
      \right)^{ \frac{1}{2} } = \alpha \norm{ Q_1 - Q_2 }_{ L_2( \mathfrak{P} ) } \label{eq:Bellman.contract} \,,
    \end{align}
  \end{subequations}
  where the discount factor $\alpha \in [0, 1)$.
\end{proposition}

\begin{proof}
  See~\cref{prop:Bellman.contraction.L2.proof}.
\end{proof}

Because $\Bellman^{\diamond}_{\mu}$ and $\Bellman^{\diamond}$ are contractions in the Hilbert space
$L_2(\mathfrak{P})$, they each have a unique fixed point, denoted $Q^{\diamond}_{\mu}$ and
$Q^{\diamond}$, respectively~\cite{hb.plc.book}. The classical method to compute these fixed points is
the Banach-Picard iteration: for any arbitrarily fixed $Q \in L_2(\mathfrak{P})$, sequences $(\,
(\Bellman^{\diamond}_{\mu})^i Q \,)_{i \in \IntegerPP}$ and $(\, (\Bellman^{\diamond})^i Q \,)_{ i\in
  \IntegerPP}$ converge to $Q^{\diamond}_{\mu}$ and $Q^{\diamond}$, respectively, as $i \to
\infty$~\cite{hb.plc.book}.

If data $\Set{(\vect{z}_t, \vect{z}^{\prime}_t)}_{t=1}^{T}$ are assumed to be observations (realizations)
of the independent and identically distributed (IID) RVs $\Set{ (\, \bm{z}_t = \zeta( \bm{s}_t,
  \bm{a}_t), \bm{z}_t^{\prime} = \zeta ( \bm{s}_t^{\prime}, \mu(\bm{s}_t^{\prime})) \, ) }_{t=1}^T$, for
some stationary policy $\mu$, a strong connection between the empirical loss
$\hat{\mathcal{L}}_{\mu}[T](\cdot)$ in \eqref{eq:BR.loss} and the ensemble $\mathcal{L}_{\mu} (\cdot)$
one in \eqref{ensemble.loss} emerges. This is because, according to the strong law of large
numbers~\cite[\S6.2.5]{Ash:ProMeasure:00}, $\forall Q\in L_2 (\mathfrak{P})$,
\begin{align}
    & \hat{\mathcal{L}}_{\mu}[T] ( Q ) = \tfrac{1}{T} \sum_{t=1}^{T} \big[ g(\bm{z}_t) + \alpha Q
      (\bm{z}_t^{\prime}) - Q (\bm{z}_t) \big]^2 \notag \\
    & \xrightarrow[T\to \infty]{\text{a.s.}} \mathop{} \mathbb{E}\{\, [ g(\bm{z}) + \alpha
      Q(\bm{z}^{\prime}) - Q(\bm{z}) ]^2\, \} \eqqcolon \mathcal{L}_{\mu} (Q) \,, \label{slln}
\end{align}
where convergence holds almost surely (a.s.).

The following lemma provides an equivalent formulation of the ensemble loss, thereby aiding the subsequent
analysis.

\begin{lemma}\label{lemma:ensemble.loss}
  For any $Q\in L_2 ( \mathfrak{P} )$ and any stationary policy $\mu$,
  \begin{alignat*}{2}
    \mathcal{L}_{\mu} (Q)
    & = && \alpha^2\, \mathbb{E} \{\, \mathbb{V}_{ \bm{s}^{\prime} \given \bm{z} } \{ Q(\, \zeta( \bm{s}^{\prime},
    \mu(\bm{s}^{\prime})) \,) \} \, \} \\
    &&& + \norm{Q - \Bellman^{\diamond}_{\mu} Q}^2_{ L_2(\mathfrak{P}) } \,,
    \intertext{where the conditional variance}
    \mathbb{V}_{ \bm{s}^{\prime} \given \bm{z} } \{ Q( \bm{z}^{\prime} ) \}
    & {} \coloneqq {} && \mathbb{E}_{ \bm{s}^{\prime} \given \bm{z} } \{\, [\, Q(\bm{z}^{\prime}) - \mathbb{E}_{
        \bm{s}^{\prime} \given \bm{z} } \{ Q(\bm{z}^{\prime}) \}\, ]^2\, \} \,,
  \end{alignat*}
  with $\bm{z}^{\prime} = \zeta( \bm{s}^{\prime}, \mu(\bm{s}^{\prime}) )$.
\end{lemma}

\begin{proof}
  See~\cref{proof:lemma:ensemble.loss}.
\end{proof}

\cref{lemma:ensemble.loss} indicates that a minimizer of the ensemble BR loss $\mathcal{L}_{\mu}(\cdot)$
does not coincide in general with the unique fixed point $Q_{\mu}^{\diamond}$ of
$\Bellman^{\diamond}_{\mu}$---recall that $Q_{\mu}^{\diamond} = \Bellman^{\diamond}_{\mu}
Q_{\mu}^{\diamond}$. In other words, minimizing the empirical loss $\hat{\mathcal{L}}_{\mu}[T]$, which
lies close to $\mathcal{L}_{\mu}$ for large values of the number $T$ of data samples according to
\eqref{slln}, often yields a biased estimate of $Q_{\mu}^{\diamond}$. This issue is well known in BR
schemes and is typically addressed either through the introduction of auxiliary
functions~\cite{regularizedpi:16, auxiBRM} or by employing double sampling
techniques~\cite{akiyama24nonparametric, vu23rl, Sutton:IntroRL:18}.

Motivated by the previous discussion, the following assumptions are introduced.

{\color{black}
  \begin{assumptions}\label{ass:mildly.stochastic} \mbox{}
    \begin{assslist}

    \item\label{ass:data.from.event} For any iteration number $n \in \IntegerPP$, any number $T\in
      \IntegerPP$ of data samples $\Set{(\vect{z}_t, \vect{z}^{\prime}_t)}_{t=1}^{T}$, which are
      realizations of the RVs $\Set{ (\, \bm{z}_t = \zeta( \bm{s}_t, \bm{a}_t), \bm{z}_t^{\prime} = \zeta (
        \bm{s}_t^{\prime}, \mu(\bm{s}_t^{\prime})) \, ) }_{t=1}^T$, and any $\delta_{\sharp} \in \RealPP$,
      define the event
      \begin{align*}
        \mathcal{E}_n & \coloneqq \mathcal{E}_n ( T, \delta_{\sharp} ) \\
        & \coloneqq \Bigl \{ \lambda \in \Lambda \given \bigl
        \lvert \mathcal{L}_{\mu_n} (Q_n) - \hat{\mathcal{L}}_{\mu_n} [T] (Q_n) \bigr \rvert \leq \delta_{
          {\sharp} } \Bigr \} \,,
      \end{align*}
      as well as the following event $\mathcal{E}_{ \text{ev} }$, called
      ``$(\mathcal{E}_n$~eventually)''~\cite{Williams:ProMargin:91},
      \begin{align*}
        \mathcal{E}_{ \text{ev} } \coloneqq \lim\inf_{ n\to\infty } \mathcal{E}_n \coloneqq
        \bigcup\nolimits_{n\in \IntegerPP} \bigcap\nolimits_{ n^{\prime} \geq n } \mathcal{E}_{ n^{\prime} } \,.
      \end{align*}

      Assume that there exists a tuple $(T, \delta_{ {\sharp} }) \in \IntegerPP \times \RealPP$
      s.t.\ $\mathcal{E}_{ \text{ev} } \neq \emptyset$. All necessary data for \cref{algo:PI} are
      realizations of random variables (RVs) evaluated at sample points of the nonempty event
      $\mathcal{E}_{\text{ev}}$.

    \item\label{ass:critical.bound} Let $\mathfrak{C}_n$ be the set of all accumulation/cluster points as
      defined in the statement of \cref{thm:armijo}. Given $\delta_{\text{c}} \in \RealPP$,
      \cref{algo:armijo} is let to run per PI iteration $n$ for a sufficiently large number $J_n \in
      \IntegerPP$ of iterations, so that if $\vectgr{\Omega}_{n}$ denotes the update in
      line~\ref{algo:update.omega} of \cref{algo:PI}, there exists an accumulation point
      $\vectgr{\Omega}_n^{*} \in \mathfrak{C}_n$ s.t.\ $\forall n$, and a.s.,
      \[
      \vert \hat{\mathcal{L}}_{\mu_n}[T](Q_n^*) - \hat{\mathcal{L}}_{\mu_n}[T](Q_n)\vert \leq
      \delta_{\text{c}} \,,
      \]
      where $Q^*_n \coloneqq \mathscr{P} ( \vectgr{\Omega}_n^{*})$ according to
      \eqref{param.space.to.Q_K}.

    \item\label{ass:value.bounded} There exists $\Delta \in \RealPP$ s.t.\ $\forall
      \vectgr{\Omega}_n^{*}\in \mathfrak{C}_n$, $\forall n$, and a.s.,
      \[
      \hat{\mathcal{L}}_{\mu_n}[T](Q_n^*) \leq \Delta \,,
      \]
      where $Q^*_n \coloneqq \mathscr{P} ( \vectgr{\Omega}_n^{*})$ according to
      \eqref{param.space.to.Q_K}.

    \item\label{ass:classic.Bellman.n.n} There exists $\Delta_2 \in \RealPP$ s.t.\ for all sufficiently
      large $n$,
      \begin{align*}
        \norm{ Q^{\diamond}_{\mu_{n+1} } - Q^{\diamond}_{\mu_n} }_{L_2(\mathfrak{P})} \leq \Delta_2 \,.
      \end{align*}

    \end{assslist}

  \end{assumptions}
}


  \cref{ass:data.from.event} is motivated by~\eqref{slln}, which is commonly used in performance analyses
  within and beyond RL. In general, the number $T$ of data samples depends on $n$ per event
  $\mathcal{E}_n$, and must be taken sufficiently large to ensure that the bound $\delta_{\sharp}$ is
  achieved with high probability according to~\eqref{slln}. Beyond this, \cref{ass:data.from.event}
  imposes a uniform constraint by assuming a common $T$ for all sufficiently large $n$ via $\mathcal{E}_{
    \text{ev} }$. Uniform constraints are not new in the RL literature; see, for example, \cite[Lemmata 5
    and 10]{auxiBRM}, where several bounds are defined, uniformly, by taking the supremum over both
  Q-functions and policies.

  \Cref{ass:critical.bound} is motivated by \cref{thm:armijo}, which guarantees that the sequence
  generated by \cref{algo:armijo} possesses accumulation points $\mathfrak{C}_n \neq \emptyset$ a.s.,
  i.e., it admits convergent subsequences. \Cref{ass:critical.bound} presumes that $\vectgr{\Omega}_{n}$
  is drawn from such a subsequence, and can therefore come arbitrarily close to a cluster point, provided
  \cref{algo:armijo} runs for a sufficiently large number $J_n$ of iterations. This is certainly the
  case, for example, whenever the sequence converges, i.e., it admits a single cluster point. Under this
  setting, the presumed inequality in \Cref{ass:critical.bound} follows from the continuity of
  $\hat{\mathcal{L}}_{\mu_n}[T](\cdot)$ on $\mathfrak{M}_K$.

  \Cref{ass:value.bounded} imposes the mild condition that the loss functions are bounded a.s.\ at the
  cluster points of the sequence generated by \cref{algo:armijo}. This condition has been verified in
  practice throughout all numerical tests conducted in this work. Finally, \cref{ass:classic.Bellman.n.n}
  is motivated by a related discussion in \cite{akiyama24nonparametric} and imposes a smoothness
  condition on the fixed points $Q_{\mu_n}^{\diamond}$ for all sufficiently large
  $n$. 

\begin{proposition}\label{prop:limsups}\mbox{}
  \begin{proplist}

  \item\label{prop:Qn.Qmun} Under
      \cref{ass:compact.Z,ass:data.from.event,ass:critical.bound,ass:value.bounded},
    \begin{align*}
      \limsup_{n \to \infty}\, \norm{Q_n - Q^{\diamond}_{\mu_n}}_{ L_2(\mathfrak{P})} \leq \tfrac{
        \sqrt{ \Delta + \delta_{\text{c}} + \delta_{{\sharp}} }
        }{1 - \alpha} \eqqcolon
        \Delta_1  \,.
    \end{align*}

    \item\label{thm:limsup.Qn} Under~\cref{ass:compact.Z,ass:mildly.stochastic},
      \begin{align*}
        \limsup_{n \to \infty}\, \norm{ Q_n - Q^{\diamond} }_{ L_2(\mathfrak{P}) } \leq \Delta_1 +
        \tfrac{ \alpha }{1 - \alpha} ( 2 \Delta_1 + \Delta_2 ) \,,
      \end{align*}
      where $Q^{\diamond}$ is the unique fixed point of $\Bellman^{\diamond}$---see
      \cref{prop:Bellman.contraction.L2}.

  \end{proplist}
\end{proposition}

\begin{proof}
  See~\cref{proof:prop:limsups}.
\end{proof}

A result similar to the one in \cref{prop:Qn.Qmun} appears as an assumption
in~\cite[(5.11)]{Bertsekas:RLandOC:19} on the accuracy of policy evaluation, used in the proof
of~\cite[Prop. 5.1.4]{Bertsekas:RLandOC:19} to establish convergence of approximate PI. However, the
analysis in~\cite[Prop. 5.1.4]{Bertsekas:RLandOC:19} is conducted in the context of $\norm{}_{\infty}$,
whereas the present discussion is carried out in $L_2(\mathfrak{P})$, with assumptions pertaining to
$\norm{}_{L_2(\mathfrak{P})}$ norm.

The upper bound $\Delta_1$ in \cref{prop:Qn.Qmun} admits
  the following intuitive interpretation. To push $\delta_{{\sharp}}$ to small values, the amount $T$ of
  data needs to be increased via \cref{ass:data.from.event}. The user-defined $\delta_{\text{c}}$ can
  also be set to small values, at the expense of increasing the iteration budget $J_n$ according to
  \cref{ass:critical.bound}, and $\Delta$ may be potentially decreased as $K$ increases, since
  $\mathcal{Q}_K$ covers a larger portion of the ambient functional space, such as $C(\mathfrak{Z})$
  (see~\cref{subsec:GMM-QFs-express}). These parameters can be tuned separately according to the
  available computational budget. Nevertheless, $\Delta_1$ persists rather than vanishing in the limit,
  as \cref{prop:Qn.Qmun} demonstrates. This aligns with the well-known feature of BRM approaches that
  minimizing $\hat{\mathcal{L}}_{\mu_n}[T]$, or its ensemble counterpart, does not in general recover the
  fixed point $Q^{\diamond}_{\mu_n}$ of $\Bellman^{\diamond}_{\mu_n}$~\cite{regularizedpi:16, auxiBRM}.

  \Cref{thm:limsup.Qn} bounds the distance between the sequence of estimates
  $(Q_n)_{n \in \IntegerPP}$ and the fixed point $Q^{\diamond}$, thereby establishing the stability and
  consistency of \cref{algo:PI} with respect to the classical Bellman mapping
  $\Bellman^{\diamond}$. 

\subsection{Computational complexity of \cref{algo:PI}}\label{subsec:compute-complex}

The computational cost of~\cref{algo:PI} at each policy iteration $n$ is governed by the complexity of
\cref{algo:armijo}. Per iteration $j$, the dominant cost arises from computing Riemannian gradients and
performing retractions on $\mathfrak{M}_K$.  In particular, $\partial \hat{\mathcal{L}}_{\mu_n}[T](\cdot)
/ \partial \vectgr{\xi}$ in~\eqref{eq:dL.dxi} requires $\mathcal{O}(TK)$ operations due to averaging over
$T$ samples, where $\mathcal{O}(\cdot)$ denotes the big-Oh notation. For each Gaussian component
$\mathscr{G}_k (\cdot)$, $\partial \hat{\mathcal{L}}_{\mu_n}[T] (\cdot) / \partial \vect{m}_k$
in~\eqref{eq:dL.dm} and $\partial \hat{\mathcal{L}}_{\mu_n}[T] (\cdot) / \partial \vect{C}_k$
in~\eqref{eq:dL.dC.ai} incurs a cost of $\mathcal{O} (\, TK(D_z + D_z(D_z + 1)/2) \,) = \mathcal{O} (TK
D_z^2)$. The inversion of $\{ \vect{C}_k \}_{k=1}^K \subset \PS^{D_z}$ scales on the order of
$\mathcal{O} (KD_z^3)$. Consequently, the total complexity of gradient computation is $\mathcal{O} (
KD_z^2(D_z + T) )$.

Retractions for $\vectgr{\xi} \in \Real^K$ and $\vect{m}_k \in \Real^{D_z}$ in Euclidean spaces are
simple with cost $\mathcal{O} ( K(1 + D_z) )$. On the other hand, retractions the manifold $\PS^{D_z}$
in~\eqref{exp.PD} require either computing a matrix exponential (AffI metric) with complexity of $\mathcal{O} (D_z^3)$ per Gaussian component. Since
the Armijo line-search~\cite{Absil:OptimManifolds:08} in may require $M_\textnormal{a}$ iterations to
identify the Armijo step-size $\nu_j$, the overall computational complexity associated with retractions is
$\mathcal{O} (\, KM_{\text{a}}(1 + D_z + D_z(D_z+1)/2) + K D_z^3 \,) = \mathcal{O} (KD_z^2(D_z + M_{\text{a}}))$. Combining
the above terms, running \cref{algo:armijo} with $J$ iterations incurs an overall computational
complexity of $\mathcal{O}(\, JKD_z^2 (M_{\text{a}}+ D_z + T)\, )$. Although $M_\textnormal{a}$ may be unbounded,
in practice the total cost can always be reduced by imposing an upper limit on the line-search depth in
\eqref{algo:armijo.condition}. 

It is worth stressing that the computational
burden of GMM-QFs under Riemannian optimization grows
significantly with the state-action dimensionality
$D_z$ and the number of Gaussian components $K$.
Consequently, high-dimensional inputs -- such as
pixel data or LiDAR measurements with hundreds of features
-- can become prohibitively expensive. Possible remedies
include dimensionality reduction using pre-trained
auto-encoders and updating randomly selected subsets of
gradient components, though these directions are left for
future work.

Computational complexity of deep RLs is primarily determined by the underlying DNN architecture. For a
network with layer sizes $[h_0, h_1, \ldots, h_L]$, with $h_0 \coloneqq D_{s}$, $h_L\coloneqq \lvert
\mathfrak{A} \rvert$, the cost, which scales also with the number of samples $T$, is $\mathcal{O}
(T\sum_{l=1}^{L} h_{l-1} h_l)$. In practice, capturing task-specific nonlinearities often requires deep
and wide networks, leading to significantly conmputational overhead.

In contrast, nonparametric methods, such as KBRL and distributional RL, typically incur lower
computational complexity at early stages of learning. However, their complexity is dominated by the size
of their approximation basis, which grows upon the arrival of new data. For example, distributional RL
methods that utilize GMMs via EM require a per-iteration computational cost of $\mathcal{O} (
K_{\text{dist}} D_z^3 )$, with $K_{\text{dist}}$ being the size of the current basis.  A detailed
analysis of computational complexity of KBRL can be found in~\cite{akiyama24nonparametric}.

\section{Numerical Tests}\label{sec:tests}

Two popular RL benchmark tasks, the
\textit{acrobot}~\cite{spong95acrobot} and the \textit{flappy
bird}~\cite{flappy} are selected to validate the proposed~\cref{algo:PI} against:
\begin{enumerate*}[label=\textbf{(\roman*)}]

\item Kernel-based least-squares policy iteration (KLSPI)~\cite{xu07klspi}, which utilizes LSTD in an
  RKHS;

\item online Bellman residual (OBR)~\cite{onlineBRloss:16};

\item the deep Q-network (DQN)~\cite{mnih13dqn};

  \item the dueling double (D)DQN~\cite{duelingddqn}, an
      enhanced DQN variant that uses experience data to train the Q-functions;

\item the proximal policy optimization (PPO)~\cite{PPO}, a representative of policy-based methods; and

\item the GMM-based RL~\cite{agostini17gmmrl} via an online EM algorithm (EM-GMMRL).
\end{enumerate*}

To validate performance, policies generated at iteration $n$
are tests in a separate environment whose data are excluded
from training.
In~\Cref{fig:acrobot-perform,fig:grad-calculate,fig:flappy-perform,fig:flappy-grad-calculate},
the vertical axis represents the total cost incurred by the
agent when executing policy $\mu_n$ from a predefined
initial state. Following~\cite{Bertsekas:RLandOC:19}, the
objective is to minimize total loss rather than maximize
cumulative reward, i.e., loss $\coloneqq-$reward. Thus,
curves closer to the bottom-left corner indicate hight
learning efficiency. Results are averaged over multiple
independent runs, and the software was implemented in
Julia~\cite{julia17} and Python.

\begin{figure}[t]
    \centering
    \subfloat[Acrobot]{
        \includegraphics[height=.11\paperheight]{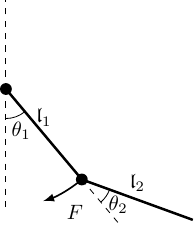}
        \label{fig:acrobot}
    }
    \quad
    \subfloat[Flappy bird]{
        \includegraphics[height=.11\paperheight]{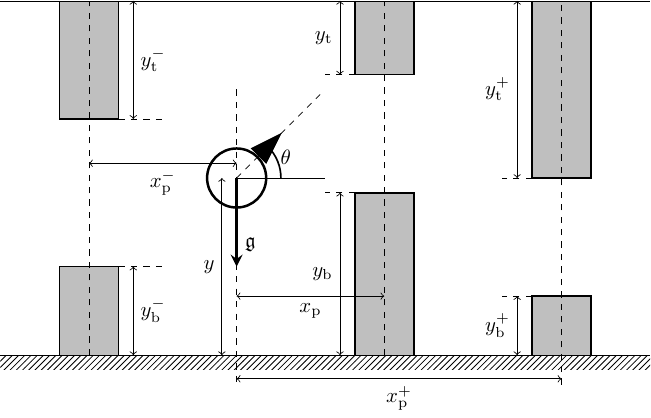}
        \label{fig:flappy}
    }
    \caption{RL benchmarks used in~\cref{sec:tests}.
    Dynamics of the acrobot and flappy bird environments
    can be found in~\cite{spong95acrobot} and~\cite{flappy},
    respectively.}
\end{figure}

\subsection{Acrobot}\label{subsec:acrobot}

The acrobot, or \textit{double-linked pendulum}, is a widely
used underactuated robot benchmark in optimal control. It
consists of two linked segments with a fixed base and a
joint actuator that swings the free end from a downward
position to an upright state. The goal is to reach the
target position in as few steps as possible without prior
knowledge of the dynamics. Its strong nonlinearity and
chaotic behavior make it a popular benchmark for evaluating
RL agents.

The state $\vect{s} \coloneqq [\theta_1, \theta_2, \dot{\theta}_1, \dot{\theta}_2]^\intercal$, where
$\theta_i \in [-\pi, \pi]$ is the angular position, and $\dot{\theta}_i$ is the angular velocity of link
$i$ with $i=1,2$. The angular velocity is bounded differently between 2 links, as $\dot{\theta}_1 \in
[-4\pi, 4\pi]\text{s}^{-1}$ and $\dot{\theta}_2 \in [-9\pi, 9\pi]\text{s}^{-1}$.  In this experiement,
the goal is achieved when $\vect{s} \in \mathfrak{S}_g \coloneqq \{[\theta_1, \theta_2, \dot{\theta}_1,
  \dot{\theta}_2]^\intercal \given -\cos(\theta_1) - \cos(\theta_1 + \theta_2) > 1 \}$. The one-step loss
is defined accordingly, as $g(\zeta(\vect{s}, \vect{a})) \coloneqq 0$ if $\vect{s} \in \mathfrak{S}_g$,
and $g(\zeta(\vect{s}, \vect{a})) \coloneqq 1$ otherwise.

To collect data samples $\mathcal{D}_{\mu_n}[T]$ in~\cref{algo:PI}, the acrobot starts from an angular
position and explores a number of actions under the current policy $\mu_n$. This exploration is called an
episode, and per iteration $n$ in ~\cref{algo:PI}, data $\mathcal{D}_{\mu_n}[T]$ with $T \coloneqq (
\text{number of episodes}) \times (\text{number of actions}) = 20 \times 70 = 1400$ are collected.
KLSPI~\cite{xu07klspi} and OBR~\cite{onlineBRloss:16} use the Gaussian kernel with bandwidth
$\sigma_\kappa = 2$, while their ALD threshold is $\delta_{\textnormal{ALD}}=0.01$. KLSPI and OBR need $T
= 5000$ to reach their ``optimal'' performance for the task at hand. DQN~\cite{mnih13dqn} and dueling
DDQN~\cite{duelingddqn} use a fully-connected neural network for Q-functions with \num{2} hidden layers
of size \num{128}, with batch size of \num{64}, and a replay buffer (experience data) of size
\num{1e4},
while PPO~\cite{PPO} uses an additional neural network with the same configuration for the policy. For
EM-GMMRL~\cite{agostini17gmmrl}, $T = 1500$, while its threshold to add new Gaussian functions in its
dictionary is $10^{-6}$. 

As shown in~\cref{fig:acrobot-perform}, \cref{algo:PI} with $K = 50$ and $K = 500$, together with
DQN~\cite{mnih13dqn} and PPO~\cite{PPO}, outperforms the other methods by reaching a stable regime
starting from approximately the \num{70}th trial. The policies learned by OBR~\cite{onlineBRloss:16} and
dueling DDQN~\cite{duelingddqn} improve slowly, even when both new and replayed experience data are
provided during training. In contrast, KLSPI~\cite{xu07klspi} exhibits suboptimal performance, while
EM-GMMRL~\cite{agostini17gmmrl} fails to achieve the desired behavior.

Notably, GMM-QFs with $K = 500$ slightly underperform those with $K = 50$. This apparent suboptimality
stems from the biased nature of the BR loss toward the current policy. {\color{black}
This observation highlights a practical challenge: the optimal number of
  Gaussian components depends on task complexity and the
  bias characteristics of the BR
  objective, raising model-size selection issues. A practical heuristic is to start with a small $K$,
  scaled proportional to the state-action dimensionality, and increase it only if the learning curve does
  not plateau after a reasonable number of policy iterations. A systematic treatment of model selection
  via sparsification for GMM-QFs is discussed in the recent
  work~\cite{Vu:eusipco:26}. 

\begin{figure}
    \centering
    \includegraphics[width=.95\columnwidth]{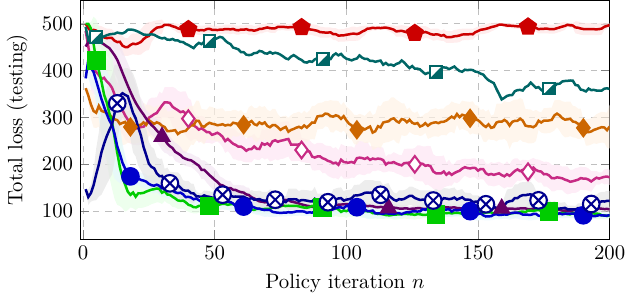}
    \caption[]{Acrobot benchmark. Curve markers:~\cref{algo:PI} with $K=50$:\proposed,
      $K=500$:\quicksymbol{proposed!70!black}{mark=otimes*,
      mark options={line width=1pt, fill=white}},
      KLSPI~\cite{xu07klspi}:\klspi, OBR~\cite{onlineBRloss:16}:\obr, DQN~\cite{mnih13dqn}:\dqn, Dueling
      DDQN~\cite{duelingddqn}:\quicksymbol{teal!80!black}{mark=halfsquare*, mark options={rotate=45, line
        width=1pt}}, PPO~\cite{PPO}:\ppo, EM-GMMRL~\cite{agostini17gmmrl}:\emgmm.
    }\label{fig:acrobot-perform}
\end{figure}

Among the successful methods, DQN~\cite{mnih13dqn} and PPO~\cite{PPO} both require at least a two-layer
fully connected network with \num{128} hidden neurons per layer to achieve the performance reported
in~\cref{fig:acrobot-perform}, resulting in a large number of learnable parameters. In
contrast,~\cref{algo:PI} attains comparable performance with significantly fewer parameters, leading to
substantially lower computational resource consumption
(see~\cref{tab:acrobot-exp-noparam}).

Beyond model size, performance vs.\ cumulative
  computations is recorded in~\cref{fig:grad-calculate}. The horizontal axis tracks the total
  floating-point operations (FLOPs) accrued from the start of training up to a given iteration, where one
  FLOP corresponds to a single floating-point arithmetic operation. In short, the horizontal records the
  total computational cost accumulated throughout learning. The resulting curves characterize the
  relationship between computational investment and agents' performance, providing a compute-centric
  perspective on learning progress. Detailed computation of FLOPs is given in~\cref{app:flops}.}
Here,~\cref{algo:PI} with GMM-QFs ($K=50$) is computationally more efficient
than deep RL methods, as its complexity is dominated by policy evaluation, which scales
with a fixed iteration budget $J$ and the line-search budget $M_{\text{a}}$
(\cref{subsec:compute-complex}). By comparison, the computational burden of DQN and PPO depends
primarily on their network architectures. Although this overhead can be reduced via ``tiny''
configurations, such reductions typically come at the expense of degraded performance
(see~\cref{fig:grad-calculate}). On the other hand, GMM-QFs with more Gaussian components $K$ generally
require a larger $J$ to reach optimal performance, ultimately increasing the
computational complexity.

\begin{table}[h]
    \centering
    \caption{Number of learned parameters (acrobot)}\label{tab:acrobot-exp-noparam}
    \begin{tabular}{l  r}
        \toprule
        \textbf{Methods} & \# parameters \\
        \midrule
        \Cref{algo:PI} ($K=50$) & \textbf{850} (\num{50}
        Gaussian functions) \\
        \Cref{algo:PI} ($K=500$) & \textbf{8500} (\num{500}
        Gaussian functions) \\
        DQN~\cite{mnih13dqn} & {17795} (2 128-neuron
        layers) \\
        PPO~\cite{PPO} & {35332} (2 networks, 2
        128-neuron layers) \\
        \bottomrule
    \end{tabular}
\end{table}

\begin{figure}[t]
    \centering
    \includegraphics[width=.95\columnwidth]{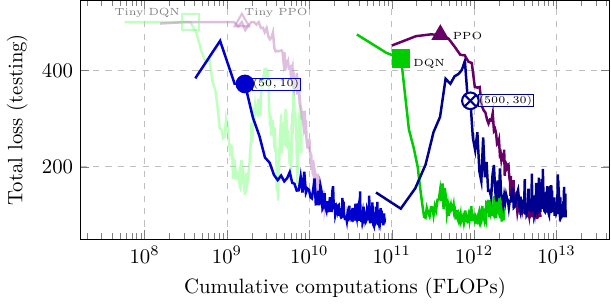}
    \caption[]{Performance of~\cref{algo:PI},
    DQN~\cite{mnih13dqn}, PPO~\cite{PPO} for acrobot against cumulative
      FLOPs.  Markers: ``tiny'' DQN (2 hidden layers of 16
      neurons): \quicksymbol{green!25!white}{mark=square, mark options={line width=1pt}}, ``tiny'' PPO
      (Q-network of 2 16-neuron layers and policy network of 2 8-neuron layers):
      \quicksymbol{violet!25!white}{mark=triangle, mark options={line width=1pt}}.  Others follow ones
      of~\cref{fig:acrobot-perform}, with annotation on~\cref{algo:PI} indicating both the number of
      Gaussians $K$ and iteration budget $J$ in form of $(K,J)$. GMM-QFs with larger $K$ require a larger
      $J$, at the cost of increased computational complexity.  } \label{fig:grad-calculate}
\end{figure}

\subsection{Flappy bird}\label{subsec:flappy}

The flappy-bird task is a challenging RL benchmark, inspired by the popular arcade game, where an agent
controls a bird navigating an endless sequence of pipes under aggressive gravity. The goal is to pass
through as many pipes as possible without collisions or leaving the screen boundaries
(see~\cref{fig:flappy}). Its continuous dynamics, delayed action effects, and strict timing constraints
make it significantly harder than classical control tasks. Due to this challenges, the online version
of~\cref{algo:PI} is considered to access data on-the-fly. For consistency and fair comparison, online
variants of all competing methods are also evaluated.

The observation state consists of \num{12} variables, as $\vect{s} \coloneqq (x^{-}_\text{p},
y^{-}_\text{t}, y^{-}_\text{b}, x_\text{p}, y_\text{t}, y_\text{b}, x^{+}_\text{p}, y^{+}_\text{t},
y^{+}_\text{b}, y, \dot{y}, \theta)$, in which $y$ is the vertical position, $\dot{y}$ is the vertical
velocity and $\theta$ is the orientation of the bird.  Other attributes include $(x^{-}_{\text{p}},
y^{-}_{\text{t}}, y^{-}_{\text{b}})$, $(x_{\text{p}}, y_{\text{t}}, y_{\text{b}})$ and
$(x^{+}_{\text{p}}, y^{+}_{\text{t}}, y^{+}_{\text{b}})$ which describe the previous, upcoming, and
subsequent pipes in term of horizontal distance to the bird and top/bottom pipe positions.
The action is discrete, with $\vect{a} \in \mathfrak{A} \coloneqq \{0, 1\}$, where $\vect{a} = 1$
triggers an upward flapping behavior.
The one-step loss is defined to encourage the long survival and successful navigation: at each time step
(every time taking action), the agent is charged a negative one-step loss of $g(\,
\zeta(\vect{s},\vect{a})\, ) = -0.1$ for ``staying alive,'' loss $-1$ for passing a pipe
successfully, $1$ upon the collision and termination (game-over), and an additional loss $0.5$ for
touching the top boundary of the screen. The event of game-over is only triggered when the bird collides
with a pipe or exits the playable area.

Under the online setting, the dataset $\mathcal{D}_{\mu_n}[T]$ now plays the role of a replay
buffer~\cite{lin93experience}, which stores also the transition tuple $(\vect{s}, \vect{a}, g,
\vect{s}')$ from the past policies up to a designated capacity $T = 1 \times 10^4$. If the number of
state transitions within $\mathcal{D}_{\mu_n}[T]$ exceeds $T$, the oldest one will be discarded. Here,
a similar buffer capacity is chosen for~\cref{algo:PI}, DQN~\cite{mnih13dqn}, dueling
DDQN~\cite{duelingddqn} and EM-GMMRL~\cite{agostini17gmmrl}. For KLSPI~\cite{xu07klspi},
OBR~\cite{onlineBRloss:16}, PPO~\cite{PPO}, a rollout buffer of $T = 4~\text{(episodes)}\times
512~\text{(state transitions)}$ is chosen. Unlike replay buffer, rollout buffer only stores data
generated by the current policy $\mu_n$. Increasing the rollout length would delay policy updates and
demand more data, making it unsuitable for online operation.
Owing to difficulty of the flappy bird task,
deep RL methods employ more sophisticated network
architectures: DQN, dueling DDQN and PPO
use Q-networks with \num{2} hidden layers of 512 neuron each, while PPO additionally uses an actor
network of similar structure. Bandwidth of Gaussian kernel for KLSPI and OBR is $\sigma_{\kappa}=1$, with
ALD threshold $\delta_{\text{ALD}}=10^{-4}$.  For the online variant of~\cref{algo:PI}, $J=1$ and $\nu_j = 0.001$ are chosen, while other configurations for deep RL methods are identical to
those of the acrobot case.

\begin{figure}[t]
    \centering
    \includegraphics[width=.95\columnwidth]{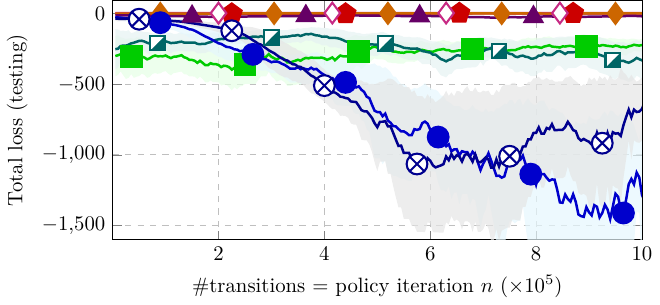}
    \caption[]{ Flappy bird benchmark. Curve markers follow ones of~\cref{fig:acrobot-perform}. Policy
      improvement is operated upon every arrival of new transition, i.e., number of transitions $=$
      policy iteration index $n$.  }
    \label{fig:flappy-perform}
\end{figure}

\Cref{fig:flappy-perform} demonstrates that~\cref{algo:PI} again outperforms all other competitors,
including deep RL schemes in the long-term.  DQN~\cite{mnih13dqn} and dueling DDQN~\cite{duelingddqn},
perform well in the early stage, but quickly settle into a sub-optimal performance as the arrivals of
streaming data; while KLSPI~\cite{xu07klspi}, OBR~\cite{onlineBRloss:16}, EM-GMMRL~\cite{agostini17gmmrl}
and PPO~\cite{PPO} behave poorly. The upward movements occasionally observed in the curves are common in
RL, particularly in complex environments, and reflect the non-stationarity of the policy iteration
process rather than divergence. Performance assessment is computed over a large volume of streaming data
(\num{1e6}) to ensure reliability. On-policy RL methods such as KLSPI, OBR, and PPO heavily rely on the
quality of data generated by the current policy. In environments with delayed action effects, such as
flappy bird, the policy may not be assessed efficiently under rollouts of limited length. Moreover,
PPO~\cite{PPO} is designed to learn from multiple rollout trajectories, making it better suited to
simulation-based training scenarios with a world model available, rather than pure online settings
where only a single trajectory can be generated through interaction with the environment.

A similar behavior of~\cref{algo:PI} with acrobot experiment is observed here, that GMM-QFs with $K=50$
slightly outperform the $K=500$ counterparts, implying the biased issue in BR minimization.
\Cref{tab:flappy-exp-noparam} highlights the parameter efficiency of GMM-QFs relative to deep RL methods,
demonstrating a highly compact configuration that maintains superior performance even for the flappy-bird
control task. However, a lean parameter structure does not automatically yield a lower computational
workload; the online learning process of GMM-QFs still incurs a significant computational footprint
(\cref{fig:flappy-grad-calculate}) due to its dependence on the state-action dimensions. To mitigate
this, future work will explore component-selective updates, akin to a model selection problem, to
effectively reduce the computational overhead.

\begin{figure}[t]
    \centering
    \includegraphics[width=.95\columnwidth]{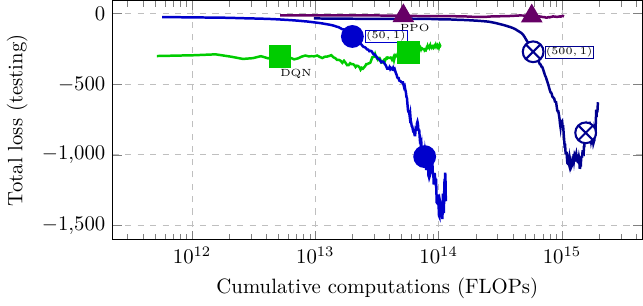}
    \caption[]{Performance of~\cref{algo:PI}, DQN~\cite{mnih13dqn}, PPO~\cite{PPO} for flappy bird
      against cumulative FLOPs. Curve markers follow those of~\cref{fig:grad-calculate}.}
      \label{fig:flappy-grad-calculate}
\end{figure}

\begin{table}[h]
    \centering
    \caption{Number of learned parameters (flappy bird)}\label{tab:flappy-exp-noparam}
    \begin{tabular}{l  r}
        \toprule
        \textbf{Methods} & \# parameters \\
        \midrule
        \Cref{algo:PI} ($K=50$) & \textbf{4600} (\num{50}
        Gaussian functions) \\
        \Cref{algo:PI} ($K=500$) & \textbf{46000} (\num{500}
        Gaussian functions) \\
        DQN~\cite{mnih13dqn} & {270338} (2 512-neuron
        layers) \\
        PPO~\cite{PPO} & {540163} (2 networks, 2
        512-neuron layers) \\
        \bottomrule
    \end{tabular}
\end{table}

\section{Conclusions}\label{sec:conclusion}

This paper introduced a novel class of Gaussian-mixture-model Q-functions (GMM-QFs) alongside a
Riemannian-optimization-based learning framework for parameter estimation, serving as an efficient
policy-evaluation step within a reinforcement learning (RL) policy-iteration scheme. The parametric
GMM-QFs possess universal approximation properties and deliver competitive representational capacity
compared to conventional deep RL methods, while requiring significantly fewer learnable parameters. This
lean configuration avoids the scalability issues common to nonparametric approaches and drastically
reduces memory footprint. Furthermore, the framework offers substantial architectural flexibility and
establishes a mathematically grounded connection between Q-function identification and Riemannian
optimization. Numerical evaluations demonstrated that GMM-QFs achieve comparable, and often superior
performance to state-of-the-art RL baselines.

While the framework may suffer from bias or overfitting due to Bellman
residual minimization, and computational overhead scales with state-action dimensionality, these
limitations outline clear potential research directions. Ongoing work focuses on improving computational
efficiency through component selection via sparsification and experience-replay strategies for robust
learning (see, e.g.,~\cite{Vu:eusipco:26}). Future extensions will incorporate state-representation
learning and dimensionality reduction techniques to enhance scalability in high-dimensional environments
and will generalize the framework to continuous action
spaces. 





\appendices
\crefalias{section}{appendix}
\crefalias{subsection}{appendix}

\section{ Proof of \cref{thm:rkhs.linear.independence} }\label{app:direct.sum.proof}

By the definition of $\Hilbert_{\vect{C}_k}^{\text{pre}}$ in \eqref{pre.gaussian.rkhs}, $Q_k = \sum\nolimits_{l=1}^{L_k}
\alpha_{kl} \mathscr{G} (\cdot \mid \vect{m}_{kl}, \vect{C}_k)$, for distinct $\{ \vect{m}_{kl} \}_{l=1}^{L_k}$, that
is, $\vect{m}_{kl} \neq \vect{m}_{kl^{\prime}}$, $\forall (l, l^{\prime}) \in \overline{1, L_k}^2$ with $l <
l^{\prime}$, for nonzero $\{ \alpha_{kl} \}_{l=1}^{L_k} \subset \Real$ and $L_k\in \IntegerPP$. As such, for any
$\vect{z}$,
\begin{alignat}{2}
  0 & {} = {} && \sum\nolimits_{k=1}^K \beta_k \sum\nolimits_{l=1}^{L_k}
  \alpha_{kl}\, \mathscr{G} ( \vect{z} \mid \vect{m}_{kl}, \vect{C}_k) \notag \\
  & = && \sum\nolimits_{k=1}^K \sum\nolimits_{l=1}^{L_k} \overbrace{\beta_k \alpha_{kl}}^{\vartheta_{kl}} \,
  \mathscr{G} ( \vect{z} \mid \vect{m}_{kl}, \vect{C}_k) \notag \\
  & = && \sum\limits_{k=1}^K \sum\limits_{l=1}^{L_k} \vartheta_{kl} \exp\big(
  -\norm{\vect{z}}^2_{\vect{C}_k^{-1}} + 2\innerp{\vect{z}}{\vect{m}_{kl}}_{\vect{C}_k^{-1}} \notag\\
  &&& \hphantom{ \sum\limits_{k=1}^K \sum\limits_{l=1}^{L_k} \vartheta_{kl} \exp\big( }\;
  -\norm{\vect{m}_{kl}}^2_{\vect{C}_k^{-1}} \big) \label{eq:origin.linear.eq} \,,
\end{alignat}
where
\begin{alignat*}{3}
  \norm{ \vect{z} }^2_{ \vect{C}_k^{-1} }
  & \coloneqq \vect{z}^{\intercal} \vect{C}_k^{-1}
  \vect{z}\,, && \norm{ \vect{m}_{kl} }^2_{ \vect{C}_k^{-1} } && \coloneqq
  \vect{m}_{kl}^{\intercal} \vect{C}_k^{-1} \vect{m}_{kl} \,,\\
  \innerp{\vect{z}}{\vect{m}_{kl}}_{\vect{C}_k^{-1}}
  & \coloneqq \vect{z}^{\intercal} \vect{C}_k^{-1}
  \vect{m}_{kl} \,. &&&&
\end{alignat*}

Define now the index set
\begin{align*}
  \mathcal{I}_0 \coloneqq \{ ( k, k^{\prime}, l, l^{\prime} ) \given \vect{C}_{k}^{-1}
  \vect{m}_{kl} \neq \vect{C}_{k^{\prime}}^{-1} \vect{m}_{k^{\prime}l^{\prime}}, k \leq k^{\prime}, l \leq l^{\prime} \}
  \,,
\end{align*}
and choose $\vect{z}_{\sharp} \in \Real^{D_z}$ such that
\begin{align}
  \vect{z}_{\sharp} \notin
  {} & {} \bigcup \nolimits_{ k < k^{\prime}} \{ \vect{z} \in \Real^{D_z} \given \vect{z}^{\intercal} (\vect{C}_{k}^{-1} -
  \vect{C}_{ k^{\prime} }^{-1} ) \vect{z} = 0 \} \notag\\
  {} & {} \cup \smashoperator[r]{ \bigcup_{ ( k, k^{\prime}, l, l^{\prime} ) \in \mathcal{I}_0 } }\, \{ \vect{z} \in
  \Real^{D_z} \given \vect{z}^{\intercal} ( \vect{C}_{k}^{-1} \vect{m}_{kl} - \vect{C}_{k^{\prime}}^{-1}
  \vect{m}_{k^{\prime} l^{\prime} } ) = 0 \} \,. \label{z.sharp.union}
\end{align}
The union in \eqref{z.sharp.union} involves a finite number of quadric hypersurfaces $\{ \vect{z} \in \Real^{D_z} \given
\vect{z}^{\intercal} (\vect{C}_{k}^{-1} - \vect{C}_{ k^{\prime} }^{-1} ) \vect{z} = 0 \}$ and hyperplanes $\{ \vect{z}
\in \Real^{D_z} \given \vect{z}^{\intercal} ( \vect{C}_{k}^{-1} \vect{m}_{kl} - \vect{C}_{k^{\prime}}^{-1}
\vect{m}_{k^{\prime} l^{\prime} } ) = 0 \}$, which are smooth $(D_z-1)$-dimensional manifolds that cross the
origin~\cite[Example~2.2.5]{RobbinSalamon:22}. Such a finite number of smooth manifolds cannot cover the whole ambient
space $\Real^{D_z}$, and a $\vect{z}_{\sharp}$ which satisfies \eqref{z.sharp.union} can always be chosen.

Let now $\vect{z} = t\vect{z}_{\sharp}$, with $t \in \RealPP$, so that \eqref{eq:origin.linear.eq} becomes
\begin{align}
  0 = \sum_{k=1}^{K} \sum_{l=1}^{L_k} \vartheta_{kl} \exp \big(
  & - t^2 \norm{\vect{z}_{\sharp}}^2_{\vect{C}_k^{-1}} + 2t\innerp{\vect{z}_{\sharp}}{\vect{m}_{kl}}_{\vect{C}_k^{-1}}
  \notag\\
  & - \norm{\vect{m}_{kl}}^2_{\vect{C}_k^{-1}} \big) \,. \label{eq:linear.eq.unroll.z.sharp}
\end{align}
Because of \eqref{z.sharp.union}, all $\{ \norm{\vect{z}_{\sharp}}_{\vect{C}_k^{-1}} \}_{k=1}^K$ are distinct, so that
it can be assumed, w.l.o.g., that $\norm{\vect{z}_{\sharp}}^2_{\vect{C}_1^{-1}} <
\norm{\vect{z}_{\sharp}}^2_{\vect{C}_k^{-1}}$, $\forall k\in \overline{2, K}$. Dividing both sides of
\eqref{eq:linear.eq.unroll.z.sharp} by $\exp(-t^2 \norm{\vect{z}_{\sharp}}_{\vect{C}_1^{-1}})$ yields
\begin{align}
  0 {} = {} & \sum \nolimits_{l=1}^{L_1} \vartheta_{1l} \exp\big( 2t
  \innerp{\vect{z}_{\sharp}}{\vect{m}_{1l}}_{\vect{C}_{1}^{-1}} - \norm{\vect{m}_{1l}}^2_{\vect{C}^{-1}_{1}} \big)
  \notag \\
  & + \sum_{k=2}^{K} \sum_{l=1}^{L_k} \vartheta_{kl} \exp \big( - t^2 (\norm{\vect{z}_{\sharp}}^2_{\vect{C}_k^{-1}} -
  \norm{\vect{z}_{\sharp}}^2_{\vect{C}_1^{-1}}) \notag \\
  & \hphantom{ + \sum_{k=2}^{K} \sum_{l=1}^{L_k} \vartheta_{kl} \exp \big( }\; + 2t
  \innerp{\vect{z}_{\sharp}}{\vect{m}_{kl}}_{\vect{C}_k^{-1}} - \norm{\vect{m}_{kl}}^2_{\vect{C}_k^{-1}} \big) \,. \label{eq:linear.eq.unroll.1}
\end{align}
Again because of \eqref{z.sharp.union}, all $\{ \innerp{\vect{z}_{\sharp}}{\vect{m}_{1l}}_{\vect{C}_{1}^{-1}}
\}_{l=1}^{L_1}$ are distinct, so that $\exists l_{\sharp} \in \overline{1, L_1}$ with
$\innerp{\vect{z}_{\sharp}}{\vect{m}_{1 l_{\sharp}}}_{\vect{C}_{1}^{-1}} >
\innerp{\vect{z}_{\sharp}}{\vect{m}_{1l}}_{\vect{C}_{1}^{-1}}$, $\forall l\in \overline{1, L_1} \setminus \{
l_{\sharp}\}$. Dividing both sides of~\eqref{eq:linear.eq.unroll.1} by
$\exp(2t\innerp{\vect{z}_{\sharp}}{\vect{m}_{1l_{\sharp}}}_{\vect{C}_{1}^{-1}})$ yields
\begin{alignat}{2}
  0 & {} = {} && \vartheta_{1l_{\sharp}} \exp\big( -\norm{\vect{m}_{1l_{\sharp}}}^2_{\vect{C}^{-1}_{1}} \big) \notag \\
  &&& + \smashoperator[r]{ \sum_{l\in \overline{1, L_1}\setminus\{ l_{\sharp} \} } } \vartheta_{1l} \exp\big(
  + 2t \innerp{\vect{z}_{\sharp}}{\vect{m}_{1l} - \vect{m}_{1l_{\sharp} }}_{\vect{C}_{1}^{-1}} \notag\\[-10pt]
  &&& \hphantom{ + \smashoperator[r]{ \sum_{l\in \overline{1, L_1}\setminus\{ l_{\sharp} \} } } \vartheta_{1l} \exp\big(
  }\; -\norm{\vect{m}_{1l}}^2_{\vect{C}^{-1}_{1}} \big) \notag\\[-10pt]
  &&& + \sum_{k=2}^{K} \sum_{l=1}^{L_k} \vartheta_{kl} \exp \big( -t^2 (\norm{\vect{z}_{\sharp}}^2_{\vect{C}_k^{-1}} -
  \norm{\vect{z}_{\sharp}}^2_{\vect{C}_1^{-1}}) \notag \\[-5pt]
  &&& \hphantom{ + \smashoperator[r]{ \sum_{l\in \overline{1, L_1}\setminus\{ l_{\sharp} \} } } \vartheta_{1l} \exp\big(
  }\; + 2t ( \innerp{\vect{z}_{\sharp}}{\vect{m}_{kl}}_{\vect{C}_k^{-1}} -
  \innerp{\vect{z}_{\sharp}}{\vect{m}_{1l_{\sharp} }}_{\vect{C}_{1}^{-1}} ) \notag \\
  &&& \hphantom{ + \smashoperator[r]{ \sum_{l\in \overline{1, L_1}\setminus\{ l_{\sharp} \} } } \vartheta_{1l} \exp\big(
  }\; - \norm{\vect{m}_{kl}}^2_{\vect{C}_k^{-1}} \big) \,. \label{eq:linear.eq.unroll.2}
\end{alignat}
It can be now readily verified that $\lim_{ t\to +\infty } \exp(\, 2t \innerp{\vect{z}_{\sharp}}{\vect{m}_{1l} -
  \vect{m}_{1l_{\sharp}}}_{\vect{C}_{1}^{-1}} \,) = \exp(-\infty) = 0$, and
\begin{alignat*}{2}
  & \lim\nolimits_{ t\to +\infty } \exp \big(\,
  && -t^2 (\norm{\vect{z}_{\sharp}}^2_{\vect{C}_k^{-1}} - \norm{\vect{z}_{\sharp}}^2_{\vect{C}_1^{-1}}) \\
  &&& + 2t ( \innerp{\vect{z}_{\sharp}}{\vect{m}_{kl}}_{\vect{C}_k^{-1}} -
  \innerp{\vect{z}_{\sharp}}{\vect{m}_{1l_{\sharp} }}_{\vect{C}_{1}^{-1}} )\, \big) \\
  &&& \negphantom{ \lim\nolimits_{ t\to +\infty } \exp \big(\, } = \exp( - \infty ) = 0 \,.
\end{alignat*}
Therefore, applying $\lim\nolimits_{ t\to +\infty }$ to both sides of \eqref{eq:linear.eq.unroll.2} yields
\begin{alignat*}{2}
  && 0 & = \vartheta_{1l_{\sharp}} \exp\big( -\norm{\vect{m}_{1l_{\sharp}}}^2_{\vect{C}^{-1}_{1}} \big) \Rightarrow 0 =
  \vartheta_{1l_{\sharp}} = \beta_1 \alpha_{1l_{\sharp}} \\
  \Rightarrow {} && 0 & = \beta_1\,,
\end{alignat*}
because $\alpha_{1l_{\sharp}}$ is nonzero. Consequently, $0 = \sum_{k=1}^K \beta_k Q_k \Rightarrow 0 = \sum_{k=2}^K
\beta_k Q_k$. Repeating the previous arguments leads to $0 = \beta_k$, $\forall k\in \overline{2, K-1}$. Finally, $0 =
\beta_K Q_K \Rightarrow 0 = \beta_K$, because $Q_K$ is nonzero. This completes the proof.

\section{Proof of \cref{thm:universal.approx}} \label{app:universal.approx.proof}

First, notice that for an arbitrarily fixed $K\in \IntegerPP$, any Q-function in $\mathcal{Q}_K$ belongs to $C(
\mathfrak{Z} )$; hence, $\cup_{K\in \IntegerPP} \mathcal{Q}_K \subset C( \mathfrak{Z} )$. Moreover, it can be readily
verified by the definition \eqref{eq:gmm-q} that $\cup_{K\in \IntegerPP} \mathcal{Q}_K$ is a linear vector
space. Notice also that $\forall (\vect{m}_1, \vect{C}_1), (\vect{m}_2, \vect{C}_2)\in \Real^{D_z} \times
\PS^{D_z}$ the point-wise multiplication
\begin{alignat}{2}
  &&& \negphantom{ {} = {} } \mathscr{G}( \vect{z} \mid \vect{m}_1, \vect{C}_1)\cdot \mathscr{G}( \vect{z} \mid
  \vect{m}_2, \vect{C}_2) \notag \\
  & {} = {} && \exp(\, -\vect{m}_1^{\intercal} \vect{C}_1^{-1} \vect{m}_1 - \vect{m}_2^{\intercal} \vect{C}_2^{-1}
  \vect{m}_2 + \vect{m}^{\intercal}(\vect{C}_1^{-1} + \vect{C}_2^{-1}) \vect{m}\, ) \notag \\
  &&& \cdot \mathscr{G}( \vect{z} \mid \vect{m}, \vect{C}) \in \cup_{K\in \IntegerPP} \mathcal{Q}_K
  \,, \label{subalgebra}
\end{alignat}
where
\begin{align*}
  \vect{m} & \coloneqq (\vect{C}_1^{-1} + \vect{C}_2^{-1})^{-1} ( \vect{C}_1^{-1}\vect{m}_1 + \vect{C}_2^{-1}\vect{m}_2) \,,
  \\
  \vect{C} & \coloneqq (\vect{C}_1^{-1} + \vect{C}_2^{-1})^{-1} \,.
\end{align*}
Result \eqref{subalgebra} together with the previous discussion render $\cup_{K\in \IntegerPP} \mathcal{Q}_K$ a
subalgebra~\cite{Conway:FuncAna:90} of $C( \mathfrak{Z} )$, which, however, does not include the unit/identity element
$1\in C( \mathfrak{Z} )$ w.r.t.\ point-wise mutliplication, defined as the constant function with value equal to $1$ on
$\mathfrak{Z}$. Nevertheless, $1$ belongs to the closure of $\cup_{K\in \IntegerPP} \mathcal{Q}_K$ w.r.t.\ the topology
induced by $\norm{\cdot}_{\infty}$, as the following discussion suggests: with $\Delta \coloneqq \sup_{ ( \vect{z},
  \vect{m} ) \in \mathfrak{Z}^2 } \norm{ \vect{z} - \vect{m} } < +\infty$, let the sequence of Gaussians $(\,
\mathscr{G}( \cdot \mid \vect{m}, k\vect{I}_{D_z} )\, )_{k\in \IntegerPP} \subset \cup_{K\in \IntegerPP} \mathcal{Q}_K$,
for some $\vect{m} \in \mathfrak{Z}$, and notice that
\begin{alignat*}{2}
  \forall \vect{z} \in \mathfrak{Z}\,, \quad
  &&& \negphantom{ {} \leq {} } \lvert 1 - \mathscr{G}( \vect{z} \mid \vect{m}, k\vect{I}_{D_z} ) \rvert = 1 -
  \mathscr{G}( \vect{z} \mid \vect{m}, k\vect{I}_{D_z} ) \\
  & {} = {} && 1 - \exp( - \norm{ \vect{z} - \vect{m} }^2 / k ) \\
  & {} \leq {} && 1 - \exp( - \Delta^2 / k ) \\
  \Rightarrow \quad &&& \negphantom{ {} \leq {} } \norm{ 1 - \mathscr{G}( \cdot \mid \vect{m}, k\vect{I}_{D_z} )
  }_{\infty} \\
  & {} = {} && \sup \nolimits_{ \vect{z}\in \mathfrak{Z} } \lvert 1 - \mathscr{G}( \vect{z} \mid \vect{m},
  k\vect{I}_{D_z} ) \rvert \\
  & {} \leq {} && 1 - \exp( - \Delta^2 / k ) \\
  \Rightarrow \quad &&& \negphantom{ {} \leq {} } \lim \nolimits_{ k\to \infty } \norm{ 1 - \mathscr{G}( \cdot \mid
    \vect{m}, k\vect{I}_{D_z} ) }_{\infty} = 0 \,.
\end{alignat*}
Furthermore, for any $\vect{z}, \vect{m}\in \mathfrak{Z}$, there always exists a Q-function $Q \in \cup_{K\in
  \IntegerPP} \mathcal{Q}_K$ which ``separates'' the points $\vect{z}$ and $\vect{m}$, i.e., $Q( \vect{z} ) \neq Q(
\vect{m} )$, and at the same time $Q( \vect{m} ) \neq 0$---choose, for example, $Q \coloneqq \mathscr{G}( \cdot \mid
\vect{m}, \vect{C} )$, for some $\vect{C}\in \PS^{D_z}$.

The preceding discussion indicates that the sufficient
conditions of the Stone-Weierstrass
theorem~\cite[Thm.~4.1]{Pinkus:05} are satisfied---the proof of which does not require the unit element $1$ to be
included in the subalgebra---thereby ensuring that the statement of \cref{thm:universal.approx.C} holds true.

With regards to the proof of \cref{thm:universal.approx.L2.Borel}, consider an arbitrarily fixed $Q \in
L_2(\mathfrak{P})$ and an arbitrarily small $\varepsilon \in \RealPP$. Under \cref{ass:compact.Z} and according to
\cite[Thm.~3.14, pp.~69]{Rudin:RealComplex:87}, there exists $Q^{\prime} \in C(\mathfrak{Z})$ s.t.\ $\norm{ Q -
  Q^{\prime} }_{ L_2 (\mathfrak{P}) } \leq \varepsilon / 2$. Furthermore, by \cref{thm:universal.approx.C} there exists
$Q^{\prime \prime}\in \cup_{K \in \IntegerPP} \mathcal{Q}_K$ s.t.\ $\norm{ Q^{\prime} - Q^{\prime \prime} }_{ \infty }
\leq \varepsilon / 2$. Hence,
\begin{align*}
  \norm{ Q - Q^{\prime \prime} }_{ L_2 (\mathfrak{P}) }
  & \leq \norm{ Q - Q^{\prime} }_{ L_2 (\mathfrak{P}) } + \norm{ Q^{\prime} - Q^{\prime \prime} }_{ L_2 (\mathfrak{P}) }
  \\
  & \leq \tfrac{\varepsilon}{2} + \left( \int_{ \mathfrak{Z} } \lvert Q^{\prime} (\vect{z}) - Q^{\prime \prime}
  (\vect{z}) \rvert^2\, \mathfrak{P}(d\vect{z}) \right)^{1/2} \\
  & \leq \tfrac{\varepsilon}{2} + \left( \norm{ Q^{\prime} - Q^{\prime \prime} }_{ \infty }^2 \int_{ \mathfrak{Z} }
  \mathfrak{P}(d\vect{z}) \right)^{1/2} \\
  & = \tfrac{\varepsilon}{2} + \norm{ Q^{\prime} - Q^{\prime \prime} }_{ \infty } \leq \tfrac{\varepsilon}{2} +
  \tfrac{\varepsilon}{2} \leq \varepsilon\,,
\end{align*}
which establishes the proof of \cref{thm:universal.approx.L2.Borel}.

The following proof of \cref{thm:universal.approx.L2.Lebesgue} extends that of~\cite[Thm.~4]{chen96ebf} to the case
where the covariance matrices of the Gaussians are not necessarily diagonal.

Assume for a contradiction that $\cup_{K\in \IntegerPP} \mathcal{Q}_K$ is not dense in $L_2( \Real^{D_z} )$. First,
recall the well-known fact that $\mathscr{G}(\cdot \given \vect{m}, \vect{C}) \in L_2( \Real^{D_z} )$, $\forall
(\vect{m}, \vect{C}) \in \Real^{D_z} \times \PS^{D_z}$. Hence, $\cup_{K\in \IntegerPP} \mathcal{Q}_K$ is a linear
subspace of $L_2( \Real^{D_z} )$. The hypothesis on the non-denseness and the Hahn-Banach theorem~\cite[Chapter III,
  Cor.~6.14]{Conway:FuncAna:90} suggest that there exists a \textit{nonzero}\/ continuous linear functional $l(\cdot)$
on $L_2( \Real^{D_z} )$ such that $l(Q) = 0$, $\forall Q \in \cup_{K\in \IntegerPP} \mathcal{Q}_K$. Consequently, by the
Riesz representation theorem~\cite[Chapter I, Cor.~3.5]{Conway:FuncAna:90}, there exists a \textit{nonzero}\/ $R(\cdot)
\in L_2( \Real^{D_z} )$ such that $\forall Q \in \cup_{K\in \IntegerPP} \mathcal{Q}_K$,
\begin{equation} \label{eq:hahn-banach-riesz}
  l(Q) = \int Q(\vect{z}) R(\vect{z})\, d\vect{z} = \innerp{Q}{R}_{L_2( \Real^{D_z} )} = 0 \,,
\end{equation}
which holds true also for the special Q-function $\mathscr{G}(\cdot \given \vect{m}, \vect{C}) \in \cup_{K\in
  \IntegerPP} \mathcal{Q}_K$, $\forall (\vect{m}, \vect{C}) \in \Real^{D_z} \times \PS^{D_z}$:
\begin{equation*}
  \int \mathscr{G}(\vect{z} \given \vect{m}, \vect{C}) R(\vect{z})\, d\vect{z} = 0 \,.
\end{equation*}

Now, consider \textit{any}\/ function $W(\cdot)$ from the set of all (Schwartz) rapidly decreasing functions $S(
\Real^{D_z} ) \subset L_2( \Real^{D_z} )$~\cite[Chapter~X, Def.~6.3]{Conway:FuncAna:90} to obtain
\begin{equation}\label{eq:pre-fourier}
  \int W(\vect{m}) \Bigl( \int \mathscr{G}(\vect{z} \given \vect{m}, \vect{C}) R(\vect{z})\, d\vect{z} \Bigr) d\vect{m}
  = 0 \,.
\end{equation}
Define $\vect{x} \coloneqq \vect{C}^{-1/2}(\vect{z} - \vect{m})$. Then, $(\vect{z}-\vect{m})^\intercal \vect{C}^{-1}
(\vect{z} - \vect{m}) = \norm{\vect{x}}^2$. Notice that $\mathscr{G}(\cdot \given \vect{0}, \vect{I})$ is also a rapidly
decreasing function, and that by dividing the nonzero $\lvert \det\vect{C}^{1/2} \rvert$ from both sides, Fubini's
theorem~\cite{Conway:FuncAna:90} and \eqref{eq:pre-fourier} yields
\begin{align}
  & \int W(\vect{m}) \int \mathscr{G}(\vect{x} \given \vect{0}, \vect{I})\, R(\vect{C}^{1/2}
  \vect{x} + \vect{m})\, d\vect{x}\ d\vect{m}\notag \\
  & = \int \mathscr{G}(\vect{x} \given \vect{0}, \vect{I}) \Big( \int W(\vect{m})\, R(\vect{C}^{1/2}\vect{x} +
  \vect{m})\, d\vect{m}\, \Big) d\vect{x} \notag \\
  & = \int \mathscr{G}(\vect{x} \given \vect{0}, \vect{I}) \Big( \int W(\vect{m})\, \check{R}(-\vect{C}^{1/2}\vect{x} -
  \vect{m})\, d\vect{m} \Big) d\vect{x} \notag \\
  & = \int \mathscr{G}(\vect{x} \given \vect{0}, \vect{I})\ (W * \check{R}) (-\vect{C}^{1/2} \vect{x})\, d\vect{x} =
  0 \,, \label{eq:pre-fourier-simplify}
\end{align}
where $\check{R}(\cdot) \coloneqq R(-\cdot)$ and $(W*\check{R})(\cdot)$ denotes the convolution of $W$ and
$\check{R}$~\cite{Conway:FuncAna:90}.

With $\mathcal{F}[\cdot]$ denoting the $D_z$-dimensional Fourier transform, let $\mathcal{G}(\vectgr{\omega}) \coloneqq
\mathcal{F}[ \mathscr{G}(\cdot \given \vect{0}, \vect{I}) ](\vectgr{\omega})$, which is well known to be isotropic,
$\mathcal{W}(\vectgr{\omega}) \coloneqq \mathcal{F}[ W(\cdot) ](\vectgr{\omega})$, $\check{ \mathcal{R} }
(\vectgr{\omega}) \coloneqq \mathcal{F}[\check{R}(\cdot)](\vectgr{\omega})$. Recall also that the Fourier transform is
unitary, i.e., $\mathcal{F}^* = \mathcal{F}^{-1}$, where $\mathcal{F}^*$ is the adjoint operator of
$\mathcal{F}$~\cite[Chapter~X, Thm.~6.17]{Conway:FuncAna:90}. Now, by \eqref{eq:pre-fourier-simplify},
\begin{align}
  0 & = \innerp{ \mathscr{G}( \cdot \given \vect{0}, \vect{I}) }{ (W * \check{R}) (-\vect{C}^{1/2} \cdot ) }_{L_2(
    \Real^{D_z} )} \notag\\
  & = \innerp{ \mathcal{F}^{-1} \mathcal{F} [ \mathscr{G}( \cdot \given \vect{0}, \vect{I}) ] }{ (W * \check{R})
    (-\vect{C}^{1/2} \cdot ) }_{L_2 ( \Real^{D_z} ) } \notag\\
  & = \innerp{ \mathcal{F} [ \mathscr{G}( \cdot \given \vect{0}, \vect{I}) ] }{ \mathcal{F}^{-*} [ (W * \check{R})
    (-\vect{C}^{1/2} \cdot ) ] }_{L_2 ( \Real^{D_z} ) } \notag \\
  & = \innerp{ \mathcal{F} [ \mathscr{G}( \cdot \given \vect{0}, \vect{I}) ] }{ \mathcal{F} [ (W * \check{R})
    (-\vect{C}^{1/2} \cdot ) ] }_{L_2 ( \Real^{D_z} ) } \notag \\
  & = \int \mathcal{G}(\vectgr{\omega})\, \mathcal{W}(-\vect{C}^{-1/2}\vectgr{\omega})\,
  \check{ \mathcal{R} } (-\vect{C}^{-1/2}\vectgr{\omega})\, d\vectgr{\omega} \notag \\
  & = \int \mathcal{G}(-\vect{C}^{1/2} \vectgr{\omega}')\, \mathcal{W}(\vectgr{\omega}')\,
  \check{ \mathcal{R} } (\vectgr{\omega}')\, d\vectgr{\omega}' \quad (\vectgr{\omega}' \coloneqq
  -\vect{C}^{-1/2}\vectgr{\omega}) \notag \\
  & = \innerp{ \mathcal{W}( \cdot) }{ \mathcal{G}(\vect{C}^{1/2} \cdot )\, \check{ \mathcal{R} } ( \cdot) }_{L_2 (
    \Real^{D_z} ) } \notag\\
  & = \innerp{ \mathcal{F}[ W ](\cdot) }{ \mathcal{G}(\vect{C}^{1/2} \cdot )\, \check{ \mathcal{R} } ( \cdot) }_{L_2 (
    \Real^{D_z} ) } \notag \\
  & = \innerp{ W }{ \mathcal{F}^* [ \mathcal{G}(\vect{C}^{1/2} \cdot )\, \check{ \mathcal{R} } ( \cdot) ] }_{L_2 (
    \Real^{D_z} ) } \notag \\
  & = \innerp{ W }{ \mathcal{F}^{-1} [ \mathcal{G}(\vect{C}^{1/2} \cdot )\, \check{ \mathcal{R} } ( \cdot) ] }_{L_2 (
    \Real^{D_z} ) } \,, \label{eq:pos-fourier-L2}
\end{align}
where $\mathcal{G}(-\vect{C}^{1/2} \cdot ) = \mathcal{G}(\vect{C}^{1/2} \cdot )$ because $\mathcal{G}(\cdot)$ is
isotropic, and $\mathcal{G}(\vect{C}^{1/2} \cdot )\, \check{ \mathcal{R} } (\cdot) \in L_2 ( \Real^{D_z} )$ because
$\mathcal{G}(\vect{C}^{1/2} \cdot )$ is bounded and $\check{ \mathcal{R} } (\cdot)\in L_2 ( \Real^{D_z} )$. Note that
\eqref{eq:pos-fourier-L2} holds true for any $W \in S( \Real^{D_z} )$.

Owing to the fact that $S( \Real^{D_z} )$ is dense in $L_2 ( \Real^{D_z} )$~\cite[Chapter~X,
  Prop.~6.5]{Conway:FuncAna:90}, consider a sequence $(W_m)_{m\in \IntegerP} \in S( \Real^{D_z} )$ such that
$\lim\nolimits_{m\to \infty} W_m = \mathcal{F}^{-1} [ \mathcal{G} (\vect{C}^{1/2} \cdot )\, \check{ \mathcal{R}}
  (\cdot)]$ in the strong topology of $L_2 ( \Real^{D_z} )$. It can be therefore deduced by \eqref{eq:pos-fourier-L2}
that
\begin{align*}
  & \norm{ \mathcal{G}(\vect{C}^{1/2} \cdot )\, \check{ \mathcal{R} } (\cdot) }^2_{ L_2 ( \Real^{D_z} ) } \\
  & = \innerp{ \mathcal{G}(\vect{C}^{1/2} \cdot )\, \check{ \mathcal{R} } (\cdot) }{ \mathcal{G}(\vect{C}^{1/2} \cdot
    )\, \check{ \mathcal{R} } (\cdot) }_{L_2 ( \Real^{D_z} ) } \\
  & = \innerp{ \mathcal{F} \mathcal{F}^{-1} [\mathcal{G}(\vect{C}^{1/2} \cdot )\, \check{ \mathcal{R} } (\cdot)] }{
    \mathcal{G}(\vect{C}^{1/2} \cdot )\, \check{ \mathcal{R} } (\cdot) }_{L_2 ( \Real^{D_z} ) } \\
  & = \innerp{ \mathcal{F}^{-1} [\mathcal{G}(\vect{C}^{1/2} \cdot )\, \check{ \mathcal{R} } (\cdot)] }{
    \mathcal{F}^{-1} [ \mathcal{G}(\vect{C}^{1/2} \cdot )\, \check{ \mathcal{R} } (\cdot) ] }_{L_2 ( \Real^{D_z} )} \\
  & = \innerp{ \lim\nolimits_{m\to \infty} W_m }{ \mathcal{F}^{-1} [\mathcal{G}(\vect{C}^{1/2} \cdot
    )\, \check{ \mathcal{R} } (\cdot) ] }_{L_2 ( \Real^{D_z} )} \\
  & = \lim\nolimits_{m\to \infty} \innerp{ W_m }{ \mathcal{F}^{-1} [ \mathcal{G}(\vect{C}^{1/2} \cdot
    )\, \check{ \mathcal{R} } (\cdot) ] }_{L_2 ( \Real^{D_z} ) } = 0 \,.
\end{align*}
Consequently, $\mathcal{G}(\vect{C}^{1/2} \cdot )\, \check{ \mathcal{R} } (\cdot) = 0$, almost everywhere (a.e.) with
respect to the Lebesgue measure $\Rightarrow \check{ \mathcal{R} } = 0$, a.e., because $\mathcal{G}(\vect{C}^{1/2} \cdot
)$ is nonzero everywhere, and thus $\check{R} = \mathcal{F}^{-1}[ \check{ \mathcal{R} } ] = 0$, a.e. Hence, $R (\cdot) =
\check{R} (-\cdot) = 0$, a.e., which contradicts the initial hypothesis that $R$ is nonzero,
asserting~\cref{thm:universal.approx.L2.Lebesgue}.

\section{Proof of \cref{prop:inf.Q_K}}\label{app:prop:inf.Q_K}

For an arbitrarily fixed $\epsilon \in \RealPP$, there exists $Q \in C( \mathfrak{Z} )$ s.t.\
\begin{align*}
  \hat{\mathcal{L}}_{\mu}[T] ( Q ) \leq \inf_{ Q^{\prime} \in C( \mathfrak{Z} ) } \hat{\mathcal{L}}_{\mu}[T] ( Q^{\prime} ) +
  \epsilon \,.
\end{align*}
Define $\forall t\in \overline{1,T}$ the function
\begin{align*}
  \mathscr{l}_t
  & \colon \Real^2 \to \RealP \colon ( q_t^{\prime}, q_t ) \mapsto \mathscr{l}_t ( q_t^{\prime}, q_t ) \coloneqq ( g_t +
  \alpha q_t^{\prime} - q_t )^2 \,,
\end{align*}
which is continuous everywhere in $\Real^2$, and hence continuous at $(\, Q(\vect{z}_t^{\prime}), Q(\vect{z}_t)\,
)$. Thus, there exists a sufficiently small $\delta_t\in \RealPP$ s.t.\
\begin{alignat}{2}
  &&& \negphantom{ \mathscr{l}_t ( q_t^{\prime}, q_t ) } \norm{ ( q_t^{\prime}, q_t ) - (\, Q(\vect{z}_t^{\prime}),
    Q(\vect{z}_t)\, ) }_{ \Real^2 } \leq \delta_t \notag \\
  \Rightarrow {} & \mathscr{l}_t ( q_t^{\prime}, q_t ) && \leq \mathscr{l}_t (\, Q(\vect{z}_t^{\prime}), Q(\vect{z}_t)\,
  ) + \epsilon \notag \\
  &&& = [ g_t + \alpha Q(\vect{z}_t^{\prime}) - Q(\vect{z}_t) ]^2 + \epsilon \,. \label{loss.t.continuous}
\end{alignat}

Define now $\delta \coloneqq \min_{ t\in \overline{1,T} } \delta_t$. By \cref{thm:universal.approx.C}, there exists
$\tilde{Q}\in \cup_{K\in \IntegerPP} \mathcal{Q}_K$ s.t.\ $\norm{ \tilde{Q} - Q }_{\infty} \leq \delta / \sqrt{2}$, and
consequently, $\forall t\in \overline{1,T}$,
\begin{align*}
  \lvert \tilde{Q}(\vect{z}_t^{\prime}) - Q(\vect{z}_t^{\prime}) \rvert & \leq \delta / \sqrt{2} \,, \qquad \lvert
  \tilde{Q}(\vect{z}_t) - Q(\vect{z}_t) \rvert \leq \delta / \sqrt{2} \,.
\end{align*}
Then, the tuple $(q_t^{\prime}, q_t) \coloneqq (\, \tilde{Q}(\vect{z}_t^{\prime}), \tilde{Q}(\vect{z}_t)\, )$ satisfies
the sufficient condition in \eqref{loss.t.continuous}, and hence,
\begin{align*}
  [ g_t + \alpha \tilde{Q}(\vect{z}_t^{\prime}) - \tilde{Q}(\vect{z}_t) ]^2
  & = \mathscr{l}_t (\, \tilde{Q}(\vect{z}_t^{\prime}), \tilde{Q}(\vect{z}_t)\,) \\
  & \leq [ g_t + \alpha Q(\vect{z}_t^{\prime}) - Q(\vect{z}_t) ]^2 + \epsilon \,.
\end{align*}
Summing up all previous terms $\forall t\in \overline{1,T}$, and dividing finally by $T$ yields
\begin{align*}
  & \inf\nolimits_{ \tilde{Q}^{\prime} \in \cup_{K\in \IntegerPP} \mathcal{Q}_K } \hat{\mathcal{L}}_{\mu}[T] (
  \tilde{Q}^{\prime} ) \\
  & \leq \hat{\mathcal{L}}_{\mu}[T] ( \tilde{Q} ) = \tfrac{1}{T} \sum \nolimits_{t=1}^T [ g_t + \alpha
    \tilde{Q}(\vect{z}_t^{\prime}) - \tilde{Q}(\vect{z}_t) ]^2 \\
  & \leq \tfrac{1}{T} \sum \nolimits_{t=1}^T [ g_t + \alpha Q(\vect{z}_t^{\prime}) - Q(\vect{z}_t) ]^2 + \epsilon \\
  & = \hat{\mathcal{L}}_{\mu}[T] ( Q ) + \epsilon \leq \inf \nolimits_{ Q^{\prime} \in C( \mathfrak{Z} ) }
  \hat{\mathcal{L}}_{\mu}[T] ( Q^{\prime} ) + 2 \epsilon \,.
\end{align*}
Because $\epsilon$ was chosen arbitrarily, the previous discussion leads to $\inf\nolimits_{ Q \in \cup_{K\in
    \IntegerPP} \mathcal{Q}_K } \hat{\mathcal{L}}_{\mu}[T] ( Q ) \leq \inf \nolimits_{ Q \in C( \mathfrak{Z} ) }
\hat{\mathcal{L}}_{\mu}[T] ( Q )$. Moreover, because $\cup_{K\in \IntegerPP} \mathcal{Q}_K \subset C( \mathfrak{Z} )$, $\inf
\nolimits_{ Q \in C( \mathfrak{Z} )} \hat{\mathcal{L}}_{\mu}[T] ( Q ) \leq \inf\nolimits_{ Q \in \cup_{K\in \IntegerPP}
  \mathcal{Q}_K } \hat{\mathcal{L}}_{\mu}[T] ( Q )$, which completes the proof.

\section{Derivation of Gradients in \eqref{all.gradients}}\label{dL.deriv}

To simplify notation in the following proofs, $\mathscr{G}(\cdot \mid \vect{m}_k, \vect{C}_k)$ will be written as
$\mathscr{G}_k(\cdot)$ or $\mathscr{G}_{\vect{C}_k} (\cdot)$; the hyperparameter $T$ will be suppressed in the loss
notation, and the superscript $(j)$ will likewise be omitted. The derivation of \eqref{eq:dL.dxi} is straightforward and
therefore skipped.

\subsection{Derivation of \eqref{eq:dL.dm}}\label{dL.dm.deriv}

The standard chain rule of differentiation suggests that
\begin{alignat*}{2}
  &&& \negphantom{ {} = {} } \frac{ \partial \hat{\mathcal{L}}_{\mu} }{ \partial \vect{m}_k } \\
  & {} = {} && \tfrac{1}{T} \sum\nolimits_{t=1}^{T} \frac{\partial}{\partial \vect{m}_k} \Bigl [ g_t +
    \sum\nolimits_{i=1}^{K} \xi_i \left(\, \alpha \mathscr{G}_i (\vect{z}^\prime_t) - \mathscr{G}_i (\vect{z}_t)\,
    \right) \Bigr ]^2 \\
  & = && \tfrac{1}{T}\sum\nolimits_{t=1}^{T} 2\delta_t \xi_k \frac{\partial}{\partial \vect{m}_k} \left(\,
  \alpha \mathscr{G}_k (\vect{z}_t^\prime) - \mathscr{G}_k (\vect{z}_t)\, \right) \\
  & = && \tfrac{1}{T}\sum\nolimits_{t=1}^{T} 2\delta_t \xi_k \big[ 2\alpha \mathscr{G}_k(\vect{z}^\prime_t)
    \cdot \vect{C}_k^{-1} (\vect{z}^\prime_t - \vect{m}_k) \\
    &&& \hphantom{ \tfrac{1}{T}\sum\nolimits_{t=1}^{T} 2\delta_t \xi_k \big[ } - 2 \mathscr{G}_k(\vect{z}_t)
      \cdot \vect{C}_k^{-1} (\vect{z}_t - \vect{m}_k) \big] \\
  & = && \tfrac{1}{T}\sum\nolimits_{t=1}^{T} 4\delta_t \xi_k \vect{C}_k^{-1} \vect{d}_{tk} \,,
\end{alignat*}
which establishes \eqref{eq:dL.dm}.

\subsection{Derivation of \eqref{eq:dL.dC.ai}}\label{dL.dC.F.deriv}

The (partial) derivative of the loss $\hat{\mathcal{L}}_{\mu}$ with respect to $\vect{C}_k$, at the point
$\vectgr{\Omega}$ and along an arbitrarily fixed direction $\vectgr{\Gamma} \in T_{\vect{C}_k} \PS^{D_z} =
\mathbb{S}^{D_z}$, together with the chain rule of differentiation and the fact that the derivative of the inverse
matrix function $\inv\colon \PS^{D_z} \to \PS^{D_z} \colon \vect{C} \mapsto \inv( \vect{C} ) \coloneqq \vect{C}^{-1}$
along $\vectgr{\Gamma}$ is $\Df \inv( \vect{C} ) [ \vectgr{\Gamma} ]= -\vect{C}^{-1} \vectgr{\Gamma}
\vect{C}^{-1}$~\cite[Appendix~A.5]{Absil:OptimManifolds:08}, yields
\begin{alignat*}{2}
  &&& \negphantom{ {} = {} } \Df_{ \vect{C}_k }\! \hat{\mathcal{L}}_{\mu} ( \vectgr{\Omega} ) [ \vectgr{\Gamma} ] \\
  & {} = {} && \tfrac{1}{T} \sum_{t=1}^{T} \Df_{ \vect{C}_k } \Bigl( \Bigl [ g_t +
    \sum\nolimits_{i=1}^{K} \xi_i \left(\, \alpha \mathscr{G}_i (\vect{z}^\prime_t) - \mathscr{G}_i (\vect{z}_t)\,
    \right) \Bigr ]^2 \Bigr) [ \vectgr{\Gamma} ] \\
  & {} = {} && \tfrac{1}{T}\sum_{t=1}^{T} 2\delta_t \xi_k \Bigl [ - \alpha \mathscr{G}_k(\vect{z}^\prime_t)
    (\vect{z}^\prime_t - \vect{m}_k)^\intercal \Df\inv( \vect{C}_k )[ \vectgr{\Gamma} ] (\vect{z}^\prime_t -
    \vect{m}_k) \\
    &&& \hphantom{ \tfrac{1}{T}\sum_{t=1}^{T} 2\delta_t \xi_k \Bigl [ } + \mathscr{G}_k(\vect{z}_t) (\vect{z}_t -
      \vect{m}_k)^\intercal \Df\inv( \vect{C}_k )[ \vectgr{\Gamma} ] (\vect{z}_t - \vect{m}_k) \Bigr ] \\
  & = && \tfrac{1}{T} \sum_{t=1}^{T} 2\delta_t \xi_k \Bigl [ \alpha \mathscr{G}_k(\vect{z}^\prime_t)
    (\vect{z}^\prime_t - \vect{m}_k)^\intercal \vect{C}_k^{-1} \vectgr{\Gamma} \vect{C}_k^{-1} (\vect{z}^\prime_t -
      \vect{m}_k) \\
    &&& \hphantom{ \tfrac{1}{T} \sum_{t=1}^{T} 2\delta_t \xi_k \Bigl [ } - \mathscr{G}_k(\vect{z}_t) (\vect{z}_t -
        \vect{m}_k)^\intercal \vect{C}_k^{-1} \vectgr{\Gamma} \vect{C}_k^{-1} (\vect{z}_t - \vect{m}_k) \Bigr ] \\
  & = && \tfrac{1}{T} \sum_{t=1}^{T} 2\delta_t \xi_k \trace \Bigl( \vect{C}_k^{-1} \Bigl[ \alpha
        \mathscr{G}_k(\vect{z}_t^\prime) (\vect{z}^\prime_t - \vect{m}_k) (\vect{z}_t^\prime - \vect{m}_k)^\intercal \\
    &&& \hphantom{ \tfrac{1}{T} \sum_{t=1}^{T} 2\delta_t \xi_k \Bigl [ } - \mathscr{G}_k(\vect{z}_t) (\vect{z}_t -
          \vect{m}_k)(\vect{z}_t - \vect{m}_k)^\intercal \Bigr] \vect{C}_k^{-1} \vectgr{\Gamma} \Bigr) \\
  & = && \tfrac{1}{T} \sum_{t=1}^{T} 2\delta_t \xi_k \trace \left( \vect{C}_k^{-1} \vect{B}_{tk} \vect{C}_k^{-1}
        \vectgr{\Gamma} \right) \,.
\end{alignat*}

Recall now the following identity connecting the gradient with the derivative~\cite[Appendix~A.5]{Absil:OptimManifolds:08}:
\begin{alignat}{3}
  && \Df_{ \vect{C}_k }\! \hat{\mathcal{L}}_{\mu} ( \vectgr{\Omega} ) [ \vectgr{\Gamma} ] & {} = {} && \innerp{
    \tfrac{\partial \hat{\mathcal{L}}_{\mu} }{ \partial\vect{C}_k } }{ \vectgr{\Gamma} }_{\vect{C}_k} \notag \\
  \Rightarrow {} && \tfrac{1}{T} \sum_{t=1}^{T} 2\delta_t \xi_k \trace \left( \vect{C}_k^{-1} \vect{B}_{tk}
  \vect{C}_k^{-1} \vectgr{\Gamma} \right) & {} = {} && \innerp{ \tfrac{\partial \hat{\mathcal{L}}_{\mu} }{
      \partial\vect{C}_k } }{ \vectgr{\Gamma} }_{\vect{C}_k} \,, \label{gradient.derivative}
\end{alignat}
where $\innerp{\cdot}{\cdot}_{\vect{C}_k}$ stands for the adopted Riemannian metric.

In the case of the AffI metric, \eqref{gradient.derivative} yields
\begin{align*}
  \tfrac{1}{T} \sum_{t=1}^{T} 2\delta_t \xi_k \trace \left( \vect{C}_k^{-1} \vect{B}_{tk} \vect{C}_k^{-1}
  \vectgr{\Gamma} \right) = \trace ( \vect{C}_k^{-1} \tfrac{\partial \hat{\mathcal{L}}_{\mu}}{\partial \vect{C}_k}
  \vect{C}_k^{-1} \vectgr{\Gamma} ) \,,
\end{align*}
and because $\vectgr{\Gamma}$ was chosen arbitrarily,
\begin{align*}
  \frac{\partial \hat{\mathcal{L}}_{\mu}}{\partial \vect{C}_k} = \tfrac{1}{T} \sum\nolimits_{t=1}^{T} 2\delta_t \xi_k
  \vect{B}_{tk} \,,
\end{align*}
which establishes \eqref{eq:dL.dC.ai}.


\section{Proof of~\cref{prop:Bellman.contraction.L2}}\label{prop:Bellman.contraction.L2.proof}

To conserve space, the proof that the classical Bellman mappings take $L_2(\mathfrak{P})$ into itself is omitted, as it
can be obtained straightforwardly by applying elementary inequalities, such as
\begin{align*}
  \norm{ \Bellman^{\diamond}_{\mu} Q }^2_{ L_2(\mathfrak{P}) }
  {} \leq {} & 2 \norm{ g }^2_{ L_2(\mathfrak{P}) } \\
  & + 2 \alpha^2 \mathbb{E}\{\, \lvert\, \mathbb{E}_{\bm{s}^{\prime} \given \bm{z}} \{\, \lvert\, Q(\, \zeta(
  \bm{s}^{\prime}, \mu(\bm{s}^{\prime}))\, ) \, \}\, \rvert^2\, \} \,,
\end{align*}
for example, to establish $\norm{ \Bellman^{\diamond}_{\mu} Q }^2_{ L_2(\mathfrak{P}) } < +\infty$, $\forall Q\in
L_2(\mathfrak{P})$, and by mirroring arguments from the subsequent proof.

First, for any subset $\mathcal{E}$ of $\Lambda$ within the probability space $(\Lambda, \mathfrak{F}, \mathbb{P})$,
define the indicator function $I_{ \mathcal{E} } \colon \Lambda \to \{ 0, 1\}$ with $I_{ \mathcal{E} } (\lambda) = 1$,
if $\lambda \in \mathcal{E}$, while $I_{ \mathcal{E} } (\lambda) = 0$, if $\lambda \notin \mathcal{E}$. Then, owing to
well-known properties of the conditional expectation~\cite[\S9.7]{Williams:ProMargin:91}, such as Jensen's inequality,
it can be verified that $\forall Q_1, Q_2 \in L_2(\mathfrak{P})$ and with $\bm{z}^{\prime} \coloneqq \zeta
(\bm{s}^{\prime}, \bm{a}^{\prime})$,
\begin{align*}
  & \norm{\Bellman^{\diamond}_{\mu} Q_1 - \Bellman^{\diamond}_{\mu} Q_2}^2_{ L_2(\mathfrak{P}) } \\
  & = \alpha^2 \mathbb{E}\{\, \lvert\, \mathbb{E}_{\bm{s}^{\prime} \given \bm{z}} \{\, \lvert\, Q_1(\, \zeta(
  \bm{s}^{\prime}, \mu(\bm{s}^{\prime}))\, ) - Q_2(\, \zeta( \bm{s}^{\prime}, \mu(\bm{s}^{\prime}))\, ) \, \}\,
  \rvert^2\, \} \\
  & \leq \alpha^2 \mathbb{E}\, \mathbb{E}_{\bm{s}^{\prime} \given \bm{z}} \{\, \lvert\, Q_1(\, \zeta( \bm{s}^{\prime},
  \mu(\bm{s}^{\prime}))\, ) - Q_2(\, \zeta( \bm{s}^{\prime}, \mu(\bm{s}^{\prime}))\, ) \, \rvert^2\, \} \\
  & = \alpha^2 \mathbb{E} \{\, \lvert\, Q_1(\, \zeta( \bm{s}^{\prime},
  \mu(\bm{s}^{\prime}))\, ) - Q_2(\, \zeta( \bm{s}^{\prime}, \mu(\bm{s}^{\prime}))\, ) \, \rvert^2\, \} \\
  & = \alpha^2 \mathbb{E} \{\, \lvert\, Q_1(\, \zeta(\bm{s}^{\prime}, \bm{a}^{\prime})\, ) - Q_2(\,
  \zeta(\bm{s}^{\prime}, \bm{a}^{\prime})\, )\, \rvert^2\, I_{ \{ \bm{a}^{\prime} = \mu ( \bm{s}^{\prime} ) \} } \, \} \\
  & \leq \alpha^2 \mathbb{E} \{\, \lvert\, Q_1(\, \zeta(\bm{s}^{\prime}, \bm{a}^{\prime})\, ) - Q_2(\,
  \zeta(\bm{s}^{\prime}, \bm{a}^{\prime})\, ) \, \rvert^2\, \} \\
  & = \alpha^2 \mathbb{E} \{\, \lvert\, Q_1(\bm{z}^{\prime}) - Q_2(\bm{z}^{\prime}) \, \rvert^2\, \} = \alpha^2
  \norm{Q_1 - Q_2}^2_{ L_2(\mathfrak{P}) } \,,
\end{align*}
which establishes~\eqref{eq:Bellman.mu.contract}.

The following fact, taken from~\cite[Lemma~11]{akiyama24nonparametric}, is adapted to the current setting.

\begin{fact}[{\protect\cite[Lemma~11]{akiyama24nonparametric}}]\label{fact:inf.Qznext}
  For any $Q_1, Q_2 \in L_2( \mathfrak{P} )$ and any $\vect{s} \in \mathfrak{S}$, there exists a $\delta(\vect{s}) \neq
  0$ s.t.\ for any $\epsilon(\vect{s}) \in (0, \sqrt{ \lvert \delta (\vect{s}) \rvert })$, an $\vect{a}_{\sharp} (
  \vect{s}) \in \mathfrak{A}$ can be always selected with
  \begin{align}
    & \lvert\, \inf\nolimits_{ \bar{\vect{a}} \in \mathfrak{A}} Q_1(\, \zeta(\vect{s}, \bar{\vect{a}} )\, ) - \inf
    \nolimits_{ \bar{\vect{a}} \in \mathfrak{A}} Q_2(\, \zeta(\vect{s}, \bar{\vect{a}} )\, )\, \rvert^2 \notag \\
    & \leq \lvert\, Q_1(\, \zeta(\vect{s}, \vect{a}_{\sharp} ( \vect{s} ) )\, ) - Q_2(\, \zeta(\vect{s},
    \vect{a}_{\sharp}(\vect{s}) )\, )\, \rvert^2 + \epsilon(\vect{s}) \,.  \label{eq:inf.Qznext}
  \end{align}
\end{fact}

According to \cref{fact:inf.Qznext}, for arbitrarily fixed $Q_1, Q_2 \in L_2( \mathfrak{P} )$, $\varepsilon \in
\RealPP$, and $\vect{s}\in \mathfrak{S}$, choose any $\epsilon(\vect{s}) \in (0, \min \{ \varepsilon / \alpha^2, \sqrt{
  \lvert \delta (\vect{s}) \rvert } \})$ together with its associated $\vect{a}_{\sharp}(\vect{s})$ which satisfies
\eqref{eq:inf.Qznext}. Then, by applying Jensen's inequality again,
\begin{alignat*}{2}
  &&& \negphantom{ {} \leq {} } \norm{\Bellman^{\diamond} Q_1 - \Bellman^{\diamond}Q_2}^2_{ L_2(\mathfrak{P}) } \\
  & {} \leq {} && \alpha^2 \mathbb{E}\, \mathbb{E}_{ \bm{s}^{\prime} \given \bm{z} } \{\, \lvert\, \inf_{
    \bar{\vect{a}} \in \mathfrak{A}} Q_1(\, z( \bm{s}^{\prime}, \bar{\vect{a}} )\, ) - \inf_{ \bar{\vect{a}} \in
    \mathfrak{A}} Q_2(\, z( \bm{s}^{\prime}, \bar{\vect{a}} )\, )\, \rvert^2\, \} \\
  & = && \alpha^2 \mathbb{E} \{\, \lvert\, \inf_{ \bar{\vect{a}} \in \mathfrak{A}} Q_1(\, z( \bm{s}^{\prime},
  \bar{\vect{a}} )\, ) - \inf_{ \bar{\vect{a}} \in \mathfrak{A}} Q_2(\, z( \bm{s}^{\prime}, \bar{\vect{a}} )\, )\,
  \rvert^2\, \} \\
  & \leq && \alpha^2 \mathbb{E} \{\, \lvert\, Q_1(\, z(\bm{s}^{\prime}, \bm{a}_{\sharp} ( \bm{s}^{\prime} ) )\, ) -
  Q_2(\, z(\bm{s}^{\prime}, \bm{a}_{\sharp} (\bm{s}^{\prime}) )\, )\, \rvert^2 \,\} \\
  &&& + \alpha^2 \mathbb{E} \{ \epsilon(\bm{s}^{\prime}) \} \\
  & = && \alpha^2 \mathbb{E} \{\, \lvert\, Q_1(\, z(\bm{s}^{\prime}, \bm{a}^{\prime} )\, ) - Q_2(\,
  z(\bm{s}^{\prime}, \bm{a}^{\prime} )\, )\, \rvert^2 I_{ \{ \bm{a}^{\prime} = \bm{a}_{\sharp} ( \bm{s}^{\prime} ) \} }
  \,\} \\
  &&& + \alpha^2 \mathbb{E} \{ \epsilon(\bm{s}^{\prime}) \} \\
  & \leq && \alpha^2 \mathbb{E} \{\, \lvert\, Q_1( \bm{z}^{\prime} ) - Q_2( \bm{z}^{\prime} )\, \rvert^2 \,\} +
  \alpha^2 \mathbb{E} \{ \tfrac{\varepsilon}{\alpha^2} \} \\
  & = && \alpha^2 \norm{ Q_1 - Q_2 }^2_{ L_2 (\mathfrak{P}) } + \varepsilon \,,
\end{alignat*}
which establishes \eqref{eq:Bellman.contract} because $\varepsilon$ is an arbitrarily chosen positive real number.

\section{Proof of \cref{lemma:ensemble.loss}}\label{proof:lemma:ensemble.loss}

According to basic properties of the conditional expectation~\cite[\S9.7]{Williams:ProMargin:91},
\begin{alignat}{2}
  \mathcal{L}_{\mu} (Q)
  & {} = {} && \alpha^2\, \mathbb{E}\{\, [\,  Q(\bm{z}^{\prime}) - \mathbb{E}_{ \bm{s}^{\prime} \given \bm{z} } \{
    Q(\bm{z}^{\prime}) \}\, ]^2\, \} \notag \\
  &&& + \mathbb{E} \{\, [\, g(\bm{z}) + \alpha \mathbb{E}_{ \bm{s}^{\prime} \given \bm{z} } \{ Q(\bm{z}^{\prime})
  \} - Q(\bm{z})\, ]^2\, \} \notag \\
  &&& + 2 \alpha \mathbb{E} \{\, [\, g(\bm{z}) + \alpha \mathbb{E}_{ \bm{s}^{\prime} \given \bm{z} } \{
    Q(\bm{z}^{\prime}) \} - Q(\bm{z})\, ] \notag \\
  &&& \hphantom{ + 2 \alpha\, \mathbb{E}\, \mathbb{E}_{ \bm{s}^{\prime} \given \bm{z} } \{\, } \cdot [\,
    Q(\bm{z}^{\prime}) - \mathbb{E}_{ \bm{s}^{\prime} \given \bm{z} } \{ Q(\bm{z}^{\prime}) \}\, ]\, \} \notag \\
  & {} = {} && \alpha^2\, \mathbb{E}\, \{\, \overbrace{ \mathbb{E}_{ \bm{s}^{\prime} \given \bm{z} } \{\, [\,
      Q(\bm{z}^{\prime}) - \mathbb{E}_{ \bm{s}^{\prime} \given \bm{z} } \{ Q(\bm{z}^{\prime}) \}\, ]^2\, \} }^{
    \mathbb{V}_{ \bm{s}^{\prime} \given \bm{z} } \{ Q( \bm{z}^{\prime} ) \} }\, \} \notag \\
  &&& + \mathbb{E} \{\, [\, g(\bm{z}) + \alpha \mathbb{E}_{ \bm{s}^{\prime} \given \bm{z} } \{ Q(\bm{z}^{\prime})
  \} - Q(\bm{z})\, ]^2\, \} \notag \\
  &&& + 2 \alpha\, \mathbb{E}\, \mathbb{E}_{ \bm{s}^{\prime} \given \bm{z} } \{\, [\, g(\bm{z}) + \alpha \mathbb{E}_{
      \bm{s}^{\prime} \given \bm{z} } \{ Q(\bm{z}^{\prime}) \} - Q(\bm{z})\, ] \notag \\
  &&& \hphantom{ + 2 \alpha\, \mathbb{E}\, \mathbb{E}_{ \bm{s}^{\prime} \given \bm{z} } \{\, } \cdot [\,
    Q(\bm{z}^{\prime}) - \mathbb{E}_{ \bm{s}^{\prime} \given \bm{z} } \{ Q(\bm{z}^{\prime}) \}\, ]\, \} \notag \\
  & = && \alpha^2\, \mathbb{E} \{\, \mathbb{V}_{ \bm{z}^{\prime} \given \bm{z} } \{ Q( \bm{z}^{\prime} ) \} \, \} +
  \norm{Q - \Bellman^{\diamond}_{\mu} Q}^2_{ L_2(\mathfrak{P}) } \notag \\
  &&& + 2 \alpha\, \mathbb{E} \{\, [\, g(\bm{z}) + \alpha \mathbb{E}_{
      \bm{s}^{\prime} \given \bm{z} } \{ Q(\bm{z}^{\prime}) \} - Q(\bm{z})\, ] \notag \\
  &&& \hphantom{ + 2 \alpha\, \mathbb{E}\, \mathbb{E}_{ \bm{s}^{\prime} \given \bm{z} } \{\, } \cdot \mathbb{E}_{
    \bm{s}^{\prime} \given \bm{z} } \{ Q(\bm{z}^{\prime}) - \mathbb{E}_{ \bm{s}^{\prime} \given \bm{z} } \{
  Q(\bm{z}^{\prime}) \} \} \, \} \notag \\
  & = && \alpha^2\, \mathbb{E} \{\, \mathbb{V}_{ \bm{s}^{\prime} \given \bm{z} } \{ Q( \bm{z}^{\prime} ) \} \, \} +
  \norm{Q - \Bellman^{\diamond}_{\mu} Q}^2_{ L_2(\mathfrak{P}) } \,. \notag
\end{alignat}
This establishes \cref{lemma:ensemble.loss}.

\section{Proof of~\cref{prop:limsups}}\label{proof:prop:limsups}

\subsection{Proof of~\cref{prop:Qn.Qmun}}\label{prop:Qn.Qmun.proof}


Consider any sample point from the nonempty event $\mathcal{E}_{ \text{ev} }$ of
\cref{ass:data.from.event}. Then, by \cref{lemma:ensemble.loss}, with $\bm{z}^{\prime} = \zeta(
\bm{s}^{\prime}, \mu_n( \bm{s}^{\prime} ) )$, and for all sufficient large $n$,
\begin{subequations}
  \begin{alignat}{2}
    &&& \negphantom{ {} \leq {} } \norm{Q_n - \Bellman^{\diamond}_{\mu_n} Q_n}^2_{ L_2(\mathfrak{P}) } \notag \\
    & = && \mathcal{L}_{\mu_n} (Q_n) - \alpha^2 \mathbb{E} \{\, \mathbb{V}_{ \bm{s}^{\prime} \given \bm{z} } \{
    Q_n( \bm{z}^{\prime} ) \}\, \} \leq \mathcal{L}_{\mu_n}(Q_n) \notag \\
    & {} \leq {} && \hat{\mathcal{L}}_{\mu_n} [T] (Q_n) + \delta_{ {\sharp} } \label{aux.1} \\
    & {} \leq {} && \hat{\mathcal{L}}_{\mu_n} [T] (Q^*_{n})
      + \delta_{\text{c}} + \delta_{ {\sharp} } 
      \leq \Delta + \delta_{\text{c}} + \delta_{
          {\sharp} } \,,
      \label{aux.2}
  \end{alignat}
\end{subequations}
where \cref{ass:data.from.event} was used in \eqref{aux.1}
and \cref{ass:critical.bound,ass:value.bounded} were used
in~\eqref{aux.2}. 
Further, by \cref{prop:Bellman.contraction.L2},
\begin{alignat*}{2}
  &&& \negphantom{ {} \leq {} } \norm{Q_n - Q^{\diamond}_{\mu_n}}_{ L_2(\mathfrak{P}) } \\
  & {} \leq {} && \norm{Q_n - \Bellman^{\diamond}_{\mu_n} Q_n}_{ L_2(\mathfrak{P}) } + \norm{\Bellman^{\diamond}_{\mu_n}
    Q_n - \Bellman^{\diamond}_{\mu_n} Q^{\diamond}_{\mu_n} }_{ L_2(\mathfrak{P}) } \\
    & \leq && ( \Delta + \delta_{\text{c}} + \delta_{{\sharp}} )^{1/2} +  \alpha \norm{Q_n - Q^{\diamond}_{\mu_n}}_{ L_2(\mathfrak{P}) } \,,
\end{alignat*}
and thus, 
\begin{align*}
  \norm{Q_n - Q^{\diamond}_{\mu_n}}_{ L_2(\mathfrak{P}) } \leq \tfrac{ \sqrt{
      \Delta + \delta_{\text{c}} + \delta_{{\sharp}}
  } }{ 1 - \alpha }\,,
\end{align*}
for all sufficiently large $n$, which establishes~\cref{prop:Qn.Qmun}.

\subsection{Proof of~\cref{thm:limsup.Qn}}\label{app:limsup.Qn.proof}

The following auxiliary lemma is first in order.

\begin{lemma}\label{lemma:diff.nplus.n}
  For all sufficiently large $n$,
  \begin{alignat*}{2}
    \norm{Q^{\diamond}_{\mu_{n+1}} - Q^{\diamond}}_{L_2 (\mathfrak{P}) }
    & {} \leq {} && \alpha \norm{Q^{\diamond}_{\mu_n} - Q^{\diamond}}_{ L_2(\mathfrak{P}) } \\
    &&& + \alpha (2 \Delta_1 + \Delta_2) \,.
  \end{alignat*}
\end{lemma}

\begin{proof}
  Note that under the greedy policy improvement $\mu_{n+1} (\vect{s}) \coloneqq \arg \min \nolimits_{ \vect{a} \in
    \mathfrak{A}} Q_n(\, \zeta(\vect{s}, \vect{a}) \,) \Rightarrow Q_n(\, \zeta(\vect{s}, \mu_{n+1}(\vect{s})) \,) =
  \min_{ \vect{a} \in \mathfrak{A}} Q_n(\, \zeta(\vect{s}, \vect{a}) \,)$, so that by \eqref{Bellman.maps.standard},
  \begin{equation}
    \Bellman^{\diamond} Q_n = \Bellman^{\diamond}_{\mu_{n+1}} Q_n \,. \label{eq:Bellman.nplus.Qn}
  \end{equation}
  Then, the claim of \cref{lemma:diff.nplus.n} can be established as follows,
  \begin{subequations}
    \begin{alignat}{2}
      &&& \negphantom{ {} \leq {} } \norm{Q^{\diamond}_{\mu_{n+1}} - Q^{\diamond}}_{ L_2(\mathfrak{P}) } \notag \\
      & {} \leq {} && \norm{Q^{\diamond}_{\mu_{n+1}} - \Bellman^{\diamond} Q_n}_{ L_2(\mathfrak{P}) } +
      \norm{\Bellman^{\diamond} Q_n - Q^{\diamond}}_{ L_2(\mathfrak{P}) } \notag \\
      & = && \norm{Q^{\diamond}_{\mu_{n+1}} - \Bellman^{\diamond} Q_n}_{ L_2(\mathfrak{P}) } + \norm{\Bellman^{\diamond}
        Q_n
        - \Bellman^{\diamond} Q^{\diamond}}_{ L_2(\mathfrak{P}) } \notag\\
      & \leq && \norm{Q^{\diamond}_{\mu_{n+1}} - \Bellman^{\diamond} Q_n}_{ L_2(\mathfrak{P}) } + \alpha \norm{Q_n -
        Q^{\diamond}}_{ L_2(\mathfrak{P}) } \label{aux.II.1} \\
      & \leq && \norm{Q^{\diamond}_{\mu_{n+1}} - \Bellman^{\diamond} Q_n}_{ L_2(\mathfrak{P}) } + \alpha \norm{Q_n -
        Q^{\diamond}_{\mu_n}}_{ L_2(\mathfrak{P}) } \notag \\
      &&& + \alpha \norm{Q^{\diamond}_{\mu_n} - Q^{\diamond}}_{ L_2(\mathfrak{P}) } \notag \\
      & \leq && \alpha\norm{Q_{\mu_n}^{\diamond} - Q^{\diamond}}_{ L_2(\mathfrak{P}) } + \alpha \Delta_1 +
      \norm{Q^{\diamond}_{\mu_{n+1}} - \Bellman^{\diamond} Q_n}_{ L_2(\mathfrak{P}) } \label{aux.II.2} \\
      & \leq && \alpha\norm{Q_{\mu_n}^{\diamond} - Q^{\diamond}}_{ L_2(\mathfrak{P}) } + \alpha \Delta_1 \notag \\
      &&& + \norm{Q^{\diamond}_{\mu_{n+1}} - \Bellman^{\diamond}_{\mu_{n+1}} Q_{\mu_n}^{\diamond} }_{ L_2(\mathfrak{P}) }
      \notag \\
      &&& + \norm{\Bellman^{\diamond}_{\mu_{n+1}} Q^{\diamond}_{\mu_n} - \Bellman^{\diamond}_{\mu_{n+1}} Q_n}_{
        L_2(\mathfrak{P}) } \notag \\
      &&& + \norm{\Bellman^{\diamond}_{\mu_{n+1}} Q_n - \Bellman^{\diamond} Q_n}_{ L_2(\mathfrak{P}) } \notag \\
      & = && \alpha\norm{Q_{\mu_n}^{\diamond} - Q^{\diamond}}_{ L_2(\mathfrak{P}) } + \alpha \Delta_1 \notag \\
      &&& + \norm{ \Bellman^{\diamond}_{\mu_{n+1}} Q^{\diamond}_{\mu_{n+1}} - \Bellman^{\diamond}_{\mu_{n+1}}
        Q_{\mu_n}^{\diamond} }_{ L_2(\mathfrak{P}) } \notag \\
      &&& + \norm{\Bellman^{\diamond}_{\mu_{n+1}} Q^{\diamond}_{\mu_n} - \Bellman^{\diamond}_{\mu_{n+1}} Q_n}_{
        L_2(\mathfrak{P}) } \notag \\
      &&& + \norm{\Bellman^{\diamond}_{\mu_{n+1}} Q_n - \Bellman^{\diamond} Q_n}_{ L_2(\mathfrak{P}) } \notag \\
      & \leq && \alpha \norm{Q^{\diamond}_{\mu_n} - Q^{\diamond}}_{ L_2(\mathfrak{P}) } + \alpha \Delta_1 +
      \alpha \norm{ Q^{\diamond}_{\mu_{n+1}} - Q_{\mu_n}^{\diamond} }_{ L_2(\mathfrak{P}) } \notag \\
      &&& + \alpha \norm{Q^{\diamond}_{\mu_n} - Q_n}_{ L_2(\mathfrak{P}) } \label{aux.II.3} \\
      & \leq && \alpha \norm{Q^{\diamond}_{\mu_n} - Q^{\diamond}}_{ L_2(\mathfrak{P}) } + \alpha (2 \Delta_1 +
      \Delta_2) \label{aux.II.4}\,,
    \end{alignat}
  \end{subequations}
  where \cref{prop:Bellman.contraction.L2} was used in \eqref{aux.II.1} and \eqref{aux.II.3}, \cref{prop:Qn.Qmun} in
  \eqref{aux.II.2} and \eqref{aux.II.4}, while \eqref{eq:Bellman.nplus.Qn} and \cref{ass:classic.Bellman.n.n} were
  utilized in \eqref{aux.II.4}.
\end{proof}

For all sufficiently large $n$, application of~\cref{lemma:diff.nplus.n} for an arbitrarily fixed $I \in
\IntegerPP$ number of times yields
\begin{align*}
  & \norm{Q^{\diamond}_{\mu_{n + I}} - Q^{\diamond}}_{ L_2(\mathfrak{P}) } \\
  & \leq \alpha^I \norm{Q^{\diamond}_{\mu_{n}} - Q^{\diamond}}_{ L_2(\mathfrak{P}) } + \sum\nolimits_{i=0}^{I-1}
  \alpha^{i+1} ( 2 \Delta_1 + \Delta_2 ) \\
  & \leq \alpha^I \norm{Q^{\diamond}_{\mu_{n}} - Q^{\diamond}}_{ L_2(\mathfrak{P}) } + \underbrace{ \tfrac{ \alpha }{1 -
      \alpha} ( 2 \Delta_1 + \Delta_2 ) }_{\Delta^{\prime} } \,.
\end{align*}
Because of $\alpha \in [0, 1)$,
  \begin{align*}
    & \limsup_{n \to \infty}\, \norm{Q^{\diamond}_{\mu_{n}} - Q^{\diamond}}_{ L_2(\mathfrak{P}) } \\
    & = \limsup_{I \to \infty}\, \norm{Q^{\diamond}_{\mu_{n + I}} - Q^{\diamond}}_{ L_2(\mathfrak{P}) } \leq
    \Delta^{\prime} \,.
\end{align*}
Now, the triangle inequality and \cref{prop:Qn.Qmun} suggest
\begin{align*}
  & \norm{Q_n - Q^{\diamond}}_{ L_2(\mathfrak{P}) } \\
  & \leq \norm{Q_n - Q^{\diamond}_{\mu_n}}_{ L_2(\mathfrak{P}) } + \norm{Q^{\diamond}_{\mu_{n}} - Q^{\diamond}}_{
    L_2(\mathfrak{P}) } \\
  & \leq \Delta_1 + \norm{Q^{\diamond}_{\mu_{n}} - Q^{\diamond}}_{ L_2(\mathfrak{P}) } \,.
\end{align*}
An application of $\limsup\nolimits_{n \to \infty}$ to the previous inequality establishes~\cref{thm:limsup.Qn}.

\section{Computational Complexity and FLOP Accounting}\label{app:flops}

Computational cost via FLOP accounting is measured as follows.  \Cref{algo:PI} is performed on
$\mathfrak{M}_K \coloneqq \Real^{K} \times \Real^{K \times D_z} \times (\PS^{D_z})^K$.

The dominant computational cost arises from the Riemannian update on $K$ manifolds
$\PS^{D_z}$. Specially, the retraction via exponential maps involves in exponential evaluations. In case
of symmetric matrices $\vect{C} \in \PS^{D_z}$, matrix exponential is computed via eigendecomposition,
which requires approximately $10 D^3_z / 3$ FLOPs.  Additionally, matrix multiplications associated with
the exponential mapping contribute approximately $6 D_z^3$ FLOPs. As a result, every update on the
covariance matrix components requires $28D_z^3/3$ FLOPs per iteration.

The Euclidean gradients computations associated with $\vectgr{\xi} \in \Real^K$ and $\vect{M} \coloneqq
[\vect{m}_1, \dots, \vect{m}_K] \in \Real^{D_z \times K}$ require $2 K^2$ and $2 K^2 D_z$ FLOPs,
respectively.

Moreover, all of the above computations are operated within~\cref{algo:armijo} within $J$ iterations, a
pre-defined $M_\text{a}$ line-search budget, and $T$ samples. Thus, the total FLOPs inflicted per policy
iteration is
\[
\text{FLOPs} = JM_\text{a}T (2 K^2 + 2 K^2 D_z + \frac{28}{3}KD_z^3) \,.
\]

\section{Additional Numerical Tests}
\subsection{Inverted pendulum}\label{subsec:inverted-pendulum}
\begin{figure}[ht]
    \centering
    \includegraphics[width=.35\columnwidth]{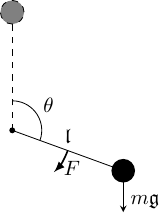}
    \caption[]{Inverted pendulum. The dynamics were given by $\dot{\theta}=\omega$ and
          $m\mathfrak{l}^2\dot{\omega} = -\mu_f\omega + m\mathfrak{g}\mathfrak{l} \sin\theta +
          \vect{a}$. The physical parameters were
          $m=1~\text{kg},~\mathfrak{l}=1~\text{m},~\mathfrak{g}=9.8~\text{m/s}^2,~\mu_f=0.01$ and force
          $\vect{a} \in \mathfrak{A}$. The time interval for simulation is $0.05~\text{s}$.
    }\label{fig:pendulum}
\end{figure}

This environment concerns the task of swinging a pendulum from its lowest position to the upright one
using a limited set of available torques~\cite{doya00rl}. The optimal policy is nontrivial to obtain
because, given the torque constraints and the effect of gravity, the controller must swing the pendulum
back and forth until sufficient kinetic energy is accumulated to reach the desired upright position
(\cref{fig:pendulum}).

The state $\vect{s} \coloneqq [\theta, \dot{\theta}]^{\intercal}$, where $\theta\in [-\pi, \pi]$ is the
angular position ($\theta = 0$ corresponds to the upright position), and $\dot{\theta}\in [-4,
  4]\text{s}^{-1}$ is the angular velocity. The action space is the set of torques $\mathfrak{A}
\coloneqq \Set{-5, -3, 0, 3, 5}\text{N}$. The one-step loss is defined as $g(\zeta(\vect{s}, \vect{a}))
\coloneqq 0$, if $\theta = 0$, and $g(\zeta(\vect{s}, \vect{a})) \coloneqq 1$, if $\theta \neq 0$.

To collect data samples $\mathcal{D}_{\mu_n}[T]$ in~\cref{algo:PI}, the pendulum starts from the
lowest position and explores a number of actions under the current policy $\mu_n$. This exploration is called an
episode, and per iteration $n$ in ~\cref{algo:PI}, data $\mathcal{D}_{\mu_n}[T]$ with $T \coloneqq (
\text{number of episodes}) \times (\text{number of actions}) = 20 \times 70 = 1400$ are collected.
KLSPI~\cite{xu07klspi} and OBR~\cite{onlineBRloss:16} use the Gaussian kernel with bandwidth
$\sigma_\kappa = 2$, while their ALD threshold is $\delta_{\textnormal{ALD}}=0.01$. KLSPI and OBR need $T
= 5000$ to reach their ``optimal'' performance for the task at hand. DQN~\cite{mnih13dqn} and dueling
DDQN~\cite{duelingddqn} use a fully-connected neural network for Q-functions with \num{2} hidden layers
of size \num{128}, with batch size of \num{64}, and a replay buffer (experience data) of size \num{1e5},
while PPO~\cite{PPO} uses an additional neural network with the same configuration for the policy. For
EM-GMMRL~\cite{agostini17gmmrl}, $T = 500$, while its threshold to add new Gaussian functions in its
dictionary is $10^{-4}$. 

\begin{figure}[ht]
    \centering
    \includegraphics[width=.9\columnwidth]{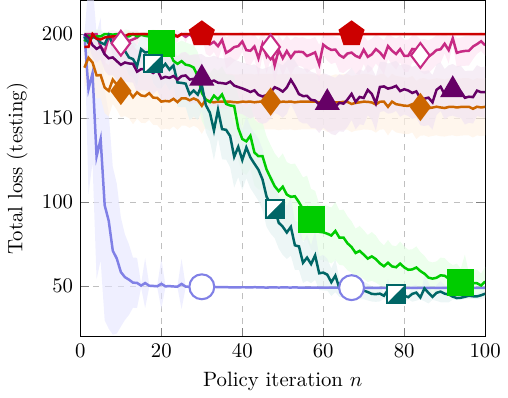}
    \caption[]{ Inverted-pendulum dataset. Curve markers: \cref{algo:PI} (AffInv) with
      $K=5$:\quicksymbol{proposed!50!white}{mark=*, mark options={line width=1pt, fill=white}},
      KLSPI~\cite{xu07klspi}:\klspi, OBR~\cite{onlineBRloss:16}:\obr, DQN~\cite{mnih13dqn}:\dqn, Dueling
      DDQN~\cite{duelingddqn}:\quicksymbol{teal!80!black}{mark=halfsquare*, mark options={rotate=45, line
          width=1pt}}, PPO~\cite{PPO}:\ppo, EM-GMMRL~\cite{agostini17gmmrl}:\emgmm.  }
    \label{fig:pendulum-exp}
\end{figure}

As shown in~\cref{fig:pendulum-exp}, the proposed~\cref{algo:PI} achieves the best performance without
using a replay buffer (experience data), in contrast to DQN~\cite{mnih13dqn} and dueling
DDQN~\cite{duelingddqn}, which require large replay buffers and exhibit slower learning and higher
variance than GMM-QFs. PPO~\cite{PPO} underperforms within the first \num{100} iterations. It is well
known that PPO, as well as other policy-based methods, converges more slowly and requires more data than
value-based approaches. KLSPI~\cite{xu07klspi} attains suboptimal performance, while
OBR~\cite{onlineBRloss:16} and EM-GMMRL~\cite{agostini17gmmrl} fail to achieve satisfactory
results. Notably, KLSPI and OBR are provided with more exploration data than~\cref{algo:PI}. Finally, EM
algorithms are known to be sensitive to initialization~\cite{Figueiredo:mixtures:02}, a fact that was
also verified in the present tests through the use of several initialization strategies.

\subsection{Mountain car}\label{subsec:mountain-car}
\begin{figure}[ht]
    \centering
    \includegraphics[width=.6\columnwidth]{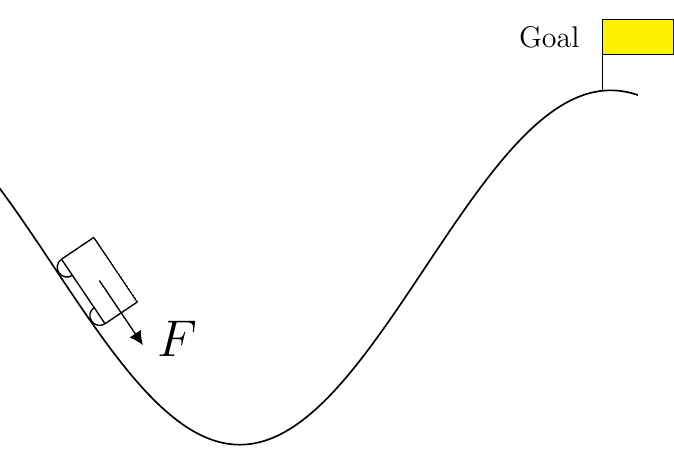}
    \caption[]{Mountain car. The dynamics were given by $v^\prime = v + \vect{a}F -
          \mathfrak{g}_\textnormal{s}\cos(3x)$. The physical parameters in simulation were given as
          $\mathfrak{g}_{\textnormal{s}}=0.0025~\text{m/s}^2$ and the force $F = 0.005~\text{N}$, action
          $\vect{a} \in \Set{-1, 0, 1}$.
    }\label{fig:mountain-car}
\end{figure}

This setting describes a car placed at the bottom of a sinusoidal valley, where the slope is given by $y
= \sin(3x)$ in the $xy$-plane, with $x \in [-1.2, 0.6]$~\cite{moore90}. The only available actions are
accelerations applied to the car in either direction. The objective is to strategically accelerate the
car, using both directions, to accumulate sufficient kinetic energy to reach the terminal (goal) state at
the top of the hill (\cref{fig:mountain-car}).

The state $\vect{s}\coloneqq [x, v]^{\intercal}$, where the velocity of the car $v\in [-0.07, 0.07]$. The
goal is achieved when the car gets beyond $x_g \coloneqq 0.5$ with velocity larger than or equal to $v_g
\coloneqq 0$, that is, whenever the car reaches a state in $\mathfrak{S}_g \coloneqq \Set{ [x,
    v]^{\intercal} \given x \geq x_g, v \geq v_g }$. The one-step loss is defined as $g(\vect{s}, a)
\coloneqq 1$, if $\vect{s} \notin \mathfrak{S}_g$, while $g(\zeta(\vect{s}, \vect{a})) \coloneqq 0$, if
$\vect{s} \in \mathfrak{S}_g$. The state-action $\zeta(\vect{s}, \vect{a})$ is defined similarly
to~\cref{subsec:inverted-pendulum}.

With regards to the data $\mathcal{D}_{\mu_n}[T]$ in~\cref{algo:PI}, a strategy similar to that of the
inverted pendulum is used. More specifically, $T = 1000$ for the proposed GMM-QFs, while $T = 20000$ for
KLSPI~\cite{xu07klspi} and $T = 1000$ for OBR~\cite{onlineBRloss:16}. A Gaussian kernel with width of
$\sigma_\kappa = 0.1$ is used for KLSPI~\cite{xu07klspi} and OBR~\cite{onlineBRloss:16}. The
implementation of DQN~\cite{mnih13dqn}, dueling DDQN~\cite{duelingddqn}, PPO~\cite{PPO} is identical to
ones for the inverted-pendulum case, while exploration length is set as $T = 100$ for
EM-GMMRL~\cite{agostini17gmmrl}.

Due to the fluctuation in performance of all algorithms, moving average values on the curves, with
window size of \num{10}, are employed hereafter.

\begin{figure}[ht]
    \centering
    \includegraphics[width=.9\columnwidth]{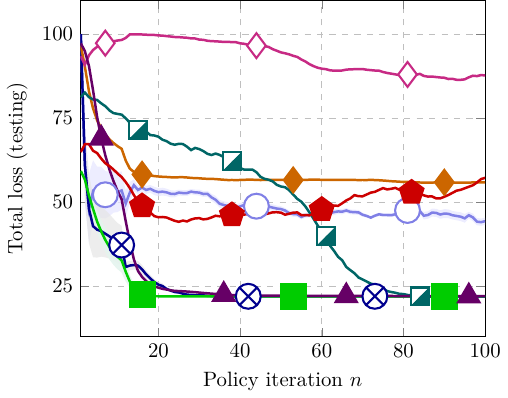}
    \caption[]{ Mountain-car dataset. Curve markers: GMM-QFs (AffI) with
      $K=500$:\quicksymbol{proposed!70!black}{mark=otimes*, mark options={line width=1pt, fill=white}},
      others follow \cref{fig:pendulum-exp}.  }
    \label{fig:mountain-car-exp}
\end{figure}

In~\cref{fig:mountain-car-exp}, DQN~\cite{mnih13dqn} achieves the best performance, followed by
PPO~\cite{PPO} and GMM-QFs with $K = 500$. Note that DQN relies on a large amount of experience data
(with a replay buffer of size \num{1e5}), whereas the proposed GMM-QFs require \textit{no}\/ experience
data to attain the performance shown in \cref{fig:mountain-car-exp}. It can also be observed that
increasing the number $K$ of Gaussian components in~\eqref{eq:gmm-q} allows GMM-QFs to reach the optimal
performance of DQN in \cref{fig:mountain-car-exp}, albeit at the cost of increased computational
complexity; see also \cref{fig:mountaincar-compare-K}. OBR~\cite{onlineBRloss:16} and
EM-GMMRL~\cite{agostini10gmmrl} perform better in this setting than in the inverted pendulum task, with
EM-GMMRL exhibiting stronger learning behavior than OBR.

\begin{figure}[ht]
    \centering
    \subfloat[Affine invariant]{ \includegraphics[width =
        0.95\columnwidth]{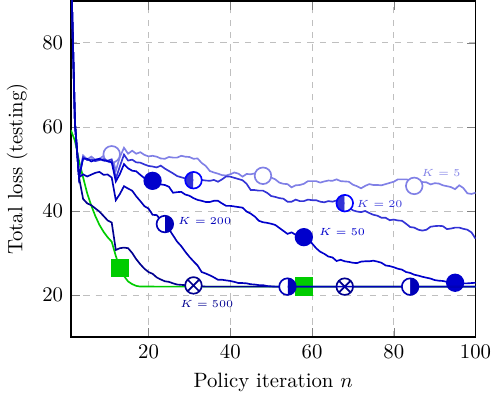} \label{mountaincar-compare-K-ai}
    } \\
    \subfloat[Bures-Wasserstein]{ \includegraphics[width =
        .95\columnwidth]{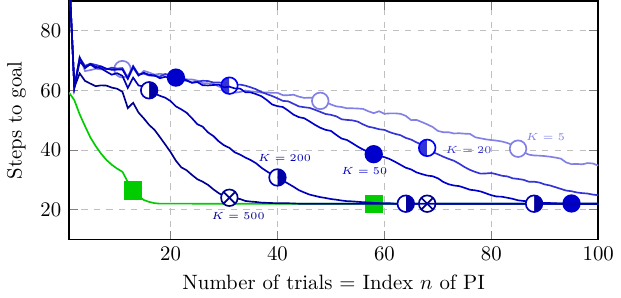} \label{mountaincar-compare-K-bw}
    }
    \caption[]{Effect of different $K$ in \cref{algo:PI} for the setting of mountain car. Curve markers
      (for both AffI and BW) $K=20$:\quicksymbol{proposed!80!white}{mark=halfcircle*, mark
        options={rotate=90, line width=1pt, fill=proposed!80!white}}, $K=50$:\proposed,
      $K=200$:\quicksymbol{proposed!90!black}{mark=halfcircle*, mark options={rotate=-90, line width=1pt,
          fill=proposed!90!black}}. DQN~\cite{mnih13dqn}:\dqn as a baseline.  The larger the $K$, the
      richer the hyperparameter space $\mathfrak{M}_K$ becomes and the faster the agent learns through
      the feedback from the environment, at the expense of increased computational
      complexity.} \label{fig:mountaincar-compare-K}
\end{figure}

\printbibliography
\end{document}